%% file: main.tex
\documentclass{article}

\usepackage{microtype}
\usepackage{graphicx}
\usepackage{booktabs} % for professional tables

\usepackage{hyperref}

\usepackage[accepted]{icml2024}

\usepackage{amsmath}
\usepackage{amssymb}
\usepackage{mathtools}
\usepackage{amsthm}

\DeclareMathOperator*{\argmax}{arg\,max}

\usepackage[capitalize,noabbrev]{cleveref}

\usepackage{url}
\usepackage{pgfplots}
\usepackage{tikz}
\usepackage{subcaption}
\usepackage{booktabs}
\usepackage{multirow}
\usepackage{pgfplotstable}
\usepgfplotslibrary{fillbetween}
\usepackage{bbm}
\usepackage{anyfontsize}
\usepackage{tabularray}
\usepackage{pifont} 

\usepackage{color, colortbl}
\usepackage{float}

\definecolor{cBlue}{HTML}{1852CC}
\definecolor{cBlue2}{HTML}{3fb8b8}
\definecolor{cBlue3}{HTML}{02ffff}
\definecolor{cBlue4}{HTML}{00b5ff}
\definecolor{cRed}{HTML}{D62728}
\definecolor{cRed2}{HTML}{ED521F}
\definecolor{cRed3}{HTML}{F69C40}
\definecolor{cRed4}{HTML}{fa4d3a}
\definecolor{cGreen}{HTML}{2CA02C}
\definecolor{cGreen2}{HTML}{3fdf3f}
\definecolor{cPink}{HTML}{ED1FD2}
\definecolor{cWhite}{HTML}{ffffff}
\definecolor{Violet}{HTML}{b05cff}
\definecolor{Gray}{gray}{0.9}

\newcommand{\cmark}{\ding{51}}
\newcommand{\xmark}{\ding{55}}

\newcommand{\sftype}[1]{{\textsf{\small #1}}}

\theoremstyle{plain}
\newtheorem{theorem}{Theorem}[section]

\theoremstyle{definition}

\theoremstyle{remark}

\usepackage[textsize=tiny]{todonotes}

\icmltitlerunning{Active Label Correction for Semantic Segmentation with Foundation Models}

\begin{document}

\twocolumn[
\icmltitle{
Active Label Correction for Semantic Segmentation with Foundation Models
}

% Active Learning for Semantic Segmentation with Foundation Models, 
% Active Data Collection and Learning for Semantic Segmentation with Foundation Models
% Active Label Correction for Semantic Segmentation Dataset Construction



\icmlsetsymbol{equal}{*}

\begin{icmlauthorlist}
\icmlauthor{Hoyoung Kim}{gsai}
\icmlauthor{Sehyun Hwang}{cse}
\icmlauthor{Suha Kwak}{gsai,cse}
\icmlauthor{Jungseul Ok}{gsai,cse}
\end{icmlauthorlist}

\icmlaffiliation{gsai}{Graduate School of AI, POSTECH, Pohang, Republic of Korea}
\icmlaffiliation{cse}{Department of CSE, POSTECH, Pohang, Republic of Korea}

\icmlcorrespondingauthor{Jungseul Ok}{jungseul@postech.ac.kr}

\icmlkeywords{Machine Learning, ICML}

\vskip 0.3in
]



\printAffiliationsAndNotice{}  % leave blank if no need to mention equal contribution

\input{Sections/0_abstract}
\input{Sections/1_introduction}
\input{Sections/2_related_work}

\input{Sections/3_method}
\input{Sections/4_experiment}

\input{Sections/5_analyses}

\input{Sections/6_conclusion}

% \clearpage
% \section*{Acknowledegments}
\smallskip\noindent\textbf{Acknowledgements.}
% {\small \section*{Acknowledgements}}
% AIGS - RS-2019-II191906 
% AIHUB - RS-2021-II212068
% BRL - RS-2023-00217286
% 곽수하 교수님 - RS-2022-II220926
This work was partly supported by the IITP grants and the NRF grants funded by Ministry of Science and ICT, Korea 
(No.RS-2019-II191906, Artificial Intelligence Graduate School Program (POSTECH);
No.RS-2021-II212068, Artificial Intelligence Innovation Hub;
No.RS-2023-00217286;
No.RS-2022-II220926).

% Authors are required to include a statement of the potential broader impact of their work, including its ethical aspects and future societal consequences. This statement should be in a separate section at the end of the paper (co-located with Acknowledgements, before References), and does not count toward the paper page limit. In many cases, where the ethical impacts and expected societal implications are those that are well established when advancing the field of Machine Learning, substantial discussion is not required, and a simple statement such as: 

% “This paper presents work whose goal is to advance the field of Machine Learning. There are many potential societal consequences of our work, none which we feel must be specifically highlighted here.”

% The above statement can be used verbatim in such cases, but we encourage authors to think about whether there is content which does warrant further discussion, as this statement will be apparent if the paper is later flagged for ethics review.


% \clearpage
% \noindent\textbf{Impact Statement.}
% {\small \section*{Impact Statement}}
\section*{Impact Statement}
% Broader impact, ethical aspects, future societal consequences
This paper presents work whose goal is to advance the field of Machine Learning. There are many potential societal consequences of our work, none which we feel must be specifically highlighted here.

\bibliography{icml2024}
\bibliographystyle{icml2024}

\input{Sections/7_appendix}

\end{document}

%% file: Sections/0_abstract.tex
\begin{abstract}
Training and validating models for semantic segmentation require datasets with pixel-wise annotations, which are notoriously labor-intensive. Although useful priors such as foundation models or crowdsourced datasets are available, they are error-prone. We hence propose an effective framework of $\text{{\it active label correction}}$ (ALC) based on a design of correction query to rectify pseudo labels of pixels, which in turn is more annotator-friendly than the standard one inquiring to classify a pixel directly according to our theoretical analysis and user study. Specifically, leveraging foundation models providing useful zero-shot predictions on pseudo labels and superpixels, our method comprises two key techniques: (i) an annotator-friendly design of correction query with the pseudo labels, and (ii) an acquisition function looking ahead label expansions based on the superpixels. Experimental results on PASCAL, Cityscapes, and Kvasir-SEG datasets demonstrate the effectiveness of our ALC framework, outperforming prior methods for active semantic segmentation and label correction. Notably, utilizing our method, we obtained a revised dataset of PASCAL by rectifying errors in 2.6 million pixels in PASCAL dataset\footnotemark.

\end{abstract}

%% file: Sections/1_introduction.tex
\iffalse
Semantic segmentation 중요하고 발전 많이 됐는데, 여전히 dataset construction 어려움, 최근에 foundation model 잘하고 dataset construction 활용 가능, dense annotation 때문에 noisy label 발생, human intervention 활용하는 active label correction 제안, noisy label detect and relabel, 이 과정에서 효율적인 novel query type 제안 (mixed one-bit query and conventional query), User study, To fully enjoy, Initial pseudo label 때문에 one-bit query 가능 + warm-start 이점, Comparison with AL (uncertainty, unreliability, unlabeled, labeled, SEEDS, SAM), Applicaion to PASCAL, PASCAL+

Semantic image segmentation is important, large advancement, customized dataset is difficult by dense annotations, Active label correction query design, Foundation model, initial pseudo label, 
\fi

\section{Introduction}
% Image 
Semantic segmentation has seen remarkable advancements powered by deep neural networks capable of learning from huge datasets with 
dense annotations for all pixels.
However, such pixel-wise annotations 
are labor-intensive and error-prone. 
To address or bypass these challenges, 
various approaches have been studied,
including crowdsourcing systems to collect large-scale human annotations~\cite{crowston2012amazon}, 
weakly supervised learning methods to train models with image-wise annotations~\cite{Ru_2023_CVPR},
and foundation models capable of useful 
zero-shot prediction on 
superpixels~\cite{Kirillov_2023_ICCV} or even semantic segmentation~\cite{liu2023grounding}.
However, those are unreliable to train and more importantly validate models for exquisite or domain-specific
%, or user-dependent
prediction.
For instance, 
despite recent advances,
the zero-shot prediction with foundation models 
\cite{Kirillov_2023_ICCV, liu2023grounding}
is considerably erroneous
as demonstrated in Table~\ref{tab:grounded-threshold}.
This can be more problematic when the semantic segmentation 
requiring expertise such as medical knowledge~\cite{ma2024segment}.

% and even in general domains illustrated in Figure~\ref{fig:Grounded SAM-images}.
% , the unreliability is frequent, leading to a significant decrease in performance, as shown in Table~\ref{tab:grounded-threshold}. 

% , as well as in general domains illustrated in Figure~\ref{fig:Grounded SAM-images}, the unreliability is frequent, leading to a significant decrease in performance, as shown in Table~\ref{tab:grounded-threshold}. 
% SAM is not good for medical or other domain speicifc. GSA performance...

Hence, we consider the problem of active label correction (ALC) to construct a reliable pixel-wise dataset from an unreliable or unlabeled dataset with a minimum cost of user intervention.
To this end, we propose an ALC framework 
which leverages foundation models
and correction queries.
% To quickly construct a clean segmentation dataset, 
% To enhance the quality of segmentation datasets, 
% we introduce a framework called Active Label Correction (ALC) that works with foundation models.
% This framework includes a new 
Our correction query is designed to rectify the pseudo labels of pixels, only if these pseudo labels are incorrect.
Unlike the standard classification query that directly requests a specific class~\cite{cai2021revisiting,Kim_2023_ICCV},
% , which always asks to classify a pixel directly
our correction query allows annotators to skip labeling if the pseudo labels are correct, making it more annotator-friendly.
Borrowing the information-theoretic annotation cost~\cite{hu2020one}, we prove that our correction query is less costly than the classification query.
Moreover, our user study in Section \ref{sec:user-study} reveals that the correction query is faster to complete than the classification query in practice.

\begin{figure*}[!t]
    \centering
    % width: 127.095
    \begin{subfigure}[h!]{.245\linewidth}
        \centering
        \includegraphics[scale=0.185]{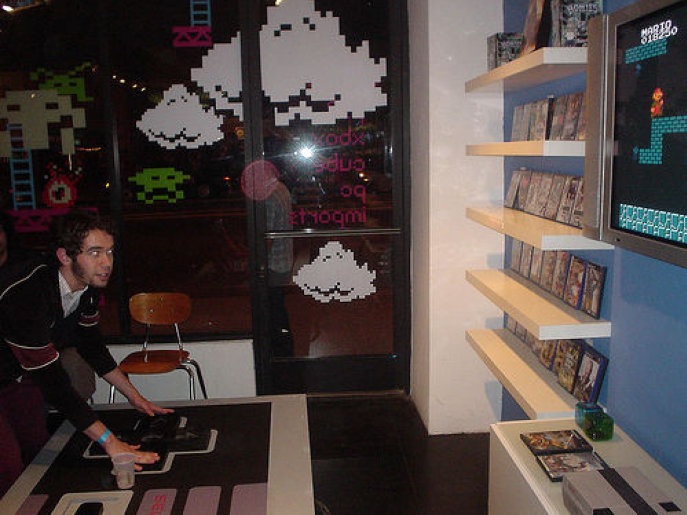}
    \end{subfigure}
    \begin{subfigure}[h!]{.245\linewidth}
        \centering
        \includegraphics[scale=0.402]{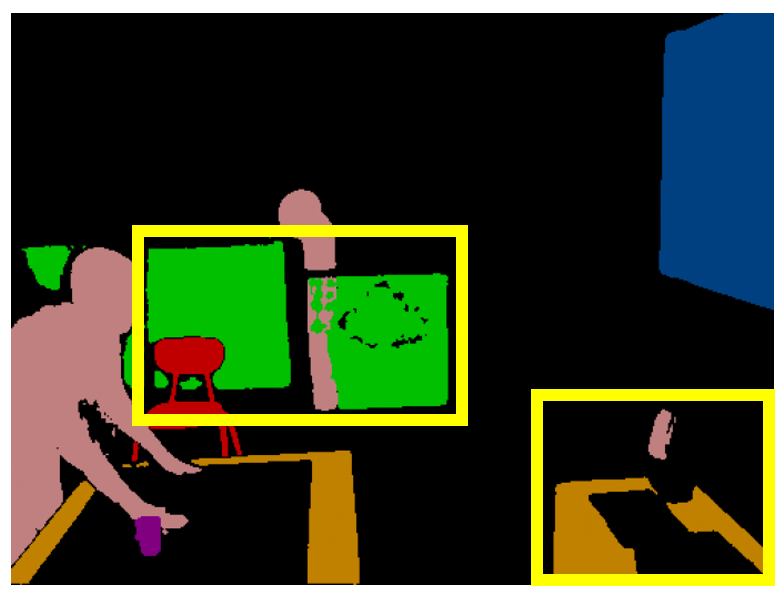}
    \end{subfigure}
    \begin{subfigure}[h!]{.245\linewidth}
        \centering
        \includegraphics[scale=0.402]{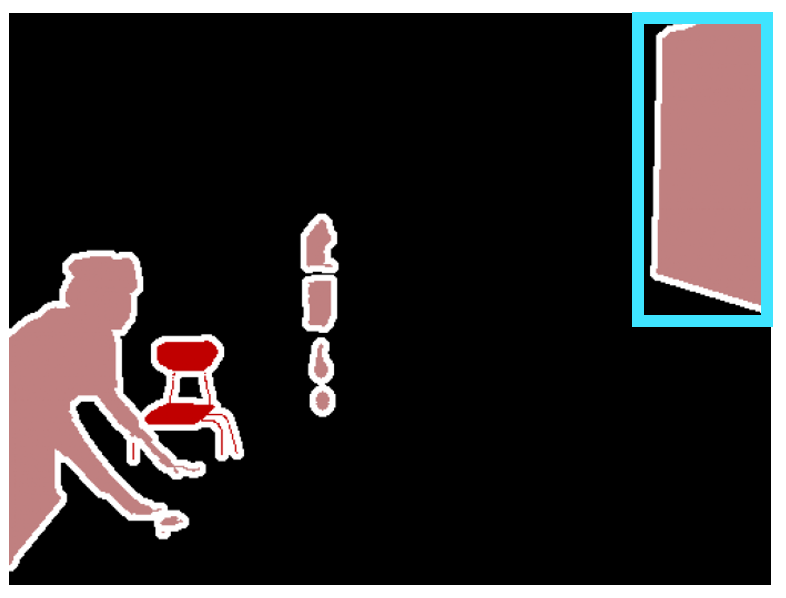}
    \end{subfigure}
    \begin{subfigure}[h!]{.245\linewidth}
        \centering
        \includegraphics[scale=0.402]{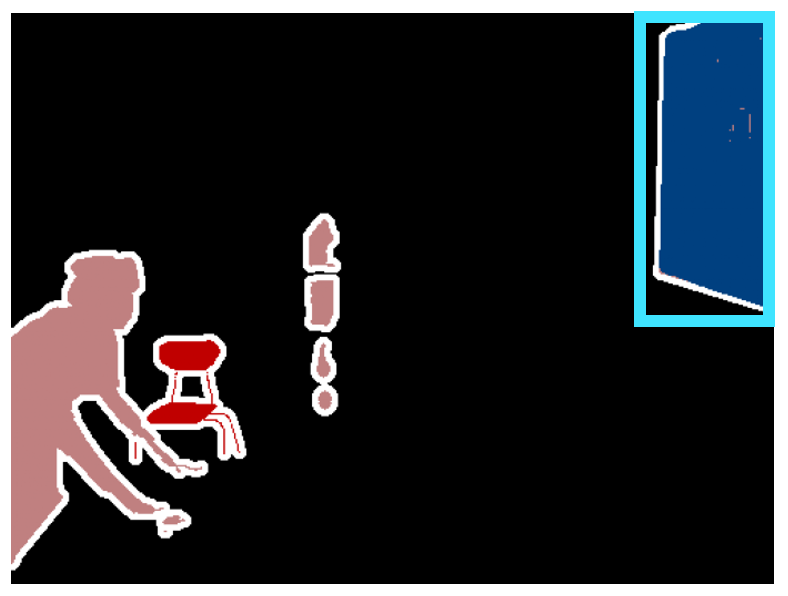}
    \end{subfigure}
    
    \begin{subfigure}[h!]{.245\linewidth}
        \centering
        \includegraphics[scale=0.231]{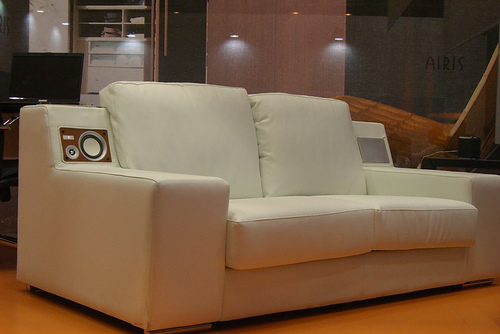}
        \caption{Unlabeled image}
        \label{fig:org-images}
    \end{subfigure}
    \begin{subfigure}[h!]{.245\linewidth}
        \centering
        \includegraphics[scale=0.4067]{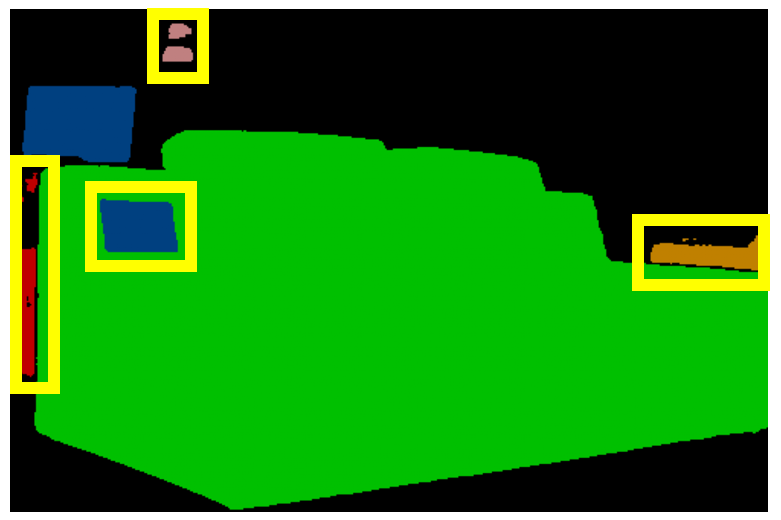}
        \caption{Grounded-SAM}
        \label{fig:Grounded SAM-images}
    \end{subfigure}
    \begin{subfigure}[h!]{.245\linewidth}
        \centering
        \includegraphics[scale=0.4035]{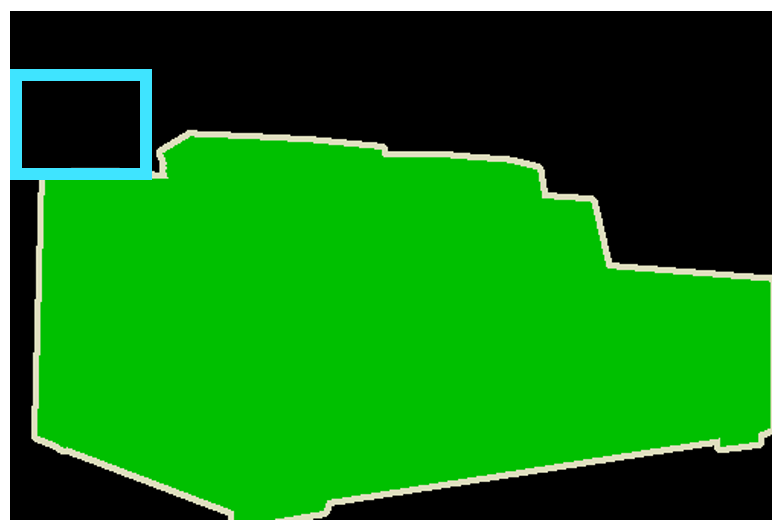}
        \caption{PASCAL}
        \label{fig:org-labels}
    \end{subfigure}
    \begin{subfigure}[h!]{.245\linewidth}
        \centering
        \includegraphics[scale=0.403]{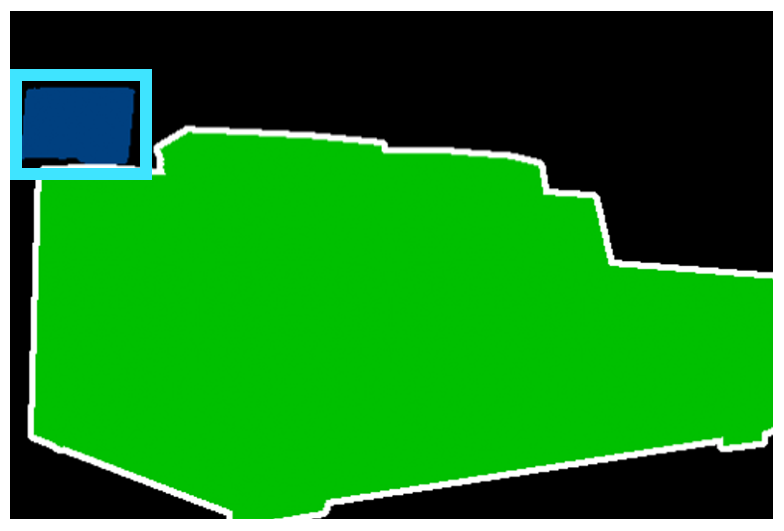}
        \caption{PASCAL+ (ours)}
        \label{fig:pascal+}
    \end{subfigure}
    \caption{{\em Examples of noisy and corrected labels in PASCAL.} (a, b) Initial pseudo labels are generated by applying Grounded-SAM (G-SAM) to unlabeled images. As depicted by the yellow boxes, noisy pseudo labels result in a decline in performance, as shown in Table~\ref{tab:grounded-threshold}. (c) PASCAL also contains noisy labels in cyan boxes. (d) By employing the superpixels from G-SAM, we construct a corrected version of PASCAL, called PASCAL+. For instance, in the first row, we correct the object labeled as person to tvmonitor, and in the second row, the object labeled as background to tvmonitor. Here, the colors black, blue, red, green, and pink represent the background, tvmonitor, chair, sofa, and person classes, respectively.}
    \label{fig:pascal-pascal-plus}
\end{figure*}

Specifically, we leverage useful zero-shot predictions on pseudo labels and superpixels from foundation models.
These pseudo labels are employed in our correction query to designate pixel labels.
They also allow us to warm-start, avoiding the typical cold-start problem that comes from the absence of a reliable way to evaluate data at the beginning of active learning~\cite{mahmood2021low,chen2023making}.
% to begin with a warm-start, avoiding the typical cold-start issue from the absence of a reliable embedding space in the beginning of active learning~\cite{mahmood2021low,chen2023making}.
Furthermore, we fully enjoy the decent superpixels to solve the challenges of pixel-wise queries.
Although pixel-wise queries can generate a flawless dataset, they require substantial time and memory to examine each pixel and lead to redundancy in the pixels chosen~\cite{shin2021all}.

\footnotetext{\href{https://github.com/ml-postech/active-label-correction}{https://github.com/ml-postech/active-label-correction}}

To address the problems, we devise superpixel-aware strategies across our entire framework.
% the selection and correction processes, respectively.
Initially, we build a diversified pixel pool consisting of partial key pixels representing each image.
As superpixels cluster pixels with 
% By incorporating the concept of a superpixel, which clusters pixels with
similar features~\cite{van2012seeds}, we choose one representative pixel per superpixel and add it to our pixel pool.
To solve the inefficiency of correcting each pixel individually per query, we extend the corrections from individual pixels to the entire superpixels they belong to.
% the corrected pixel labels to the superpixels they belong to. 
% In the correction stage, we expand the corrected pixel labels to the entire superpixel they belong to, addressing the inefficiency of individual pixel corrections.
Accordingly, we propose a look-ahead acquisition function, which anticipates the benefits of label expansion beforehand.
% takes into account the potential of label expansion in advance.
% We choose pixels with our look-ahead acquisition, considering the potential of label expansion in advance,
% Our look-ahead acquisition, which takes into account the potential of label expansion in advance, can select pix
% We finally link the selection and correction with a look-ahead acquisition which takes into account the potential of label expansion in advance.

\iffalse
Selecting an effective form of superpixel is critical, as it affects the entire framework consisting of selection and correction.
Existing superpixel generation methods~\cite{van2012seeds,achanta2012slic} can be facilitated, however, these algorithms mostly group pixels based on natural features like color.
Recent research has shown that semantic considered superpixels are beneficial for active learning in semantic segmentation~\cite{Kim_2023_ICCV}.
In line with foundation models that aim to maintain semantics, we treat the objects identified by these models as superpixels.
\fi

% utilize superpixels corresponding to objects detected by these models.
% After constructing a diversified pixel pool from superpixels, we select and refine unreliable pixels.
% To address the inefficiency of single-pixel corrections per query, we additionally introduce a label expansion technique and a look-ahead acquisition function that takes into account the potential of label expansion in advance.

The proposed framework is notably cost-efficient in constructing clean segmentation datasets.
We evaluate it by constructing new segmentation datasets from the initial pseudo labels given by foundation models in different fields, including the medical domain.
Our ALC framework outperforms prior methods for active semantic segmentation and label correction over a range of budgets.
% In addition, we correct existing segmention dat
% We evaluate our ALC framework in two practical scenarios: first, by constructing new segmentation datasets in various fields such as the medical domain, and second, by correcting existing segmentation datasets.
% In both situations, our framework outperforms prior methods for active semantic segmentation and label correction across diverse budget scenarios.
% proves to be budget-efficient in both situations compared to baselines.
% We demonstrate that our method is budget-effective in both situations compared to baselines.
In particular, we highlight its practical application by enhancing the popular PASCAL dataset~\cite{pascal-voc-2012}.
We call our corrected dataset PASCSAL+, which can be widely used in the literature of semantic segmentation.
% our framework by conducting two practical scenarios: firstly, in the creation of new segmentation datasets across multiple domains, and secondly, in the refinement of existing segmentation datasets.
% construct a corrected version of PASCAL, called PASCAL+, and validate its enhanced performance through experiments.
% applicability of the proposed framework in the real world by improving the widely used PASCAL dataset~\cite{pascal-voc-2012}.
% We construct a corrected version of PASCAL, called PASCAL+, and validate its enhanced performance through experiments.
% A pivotal part of demonstrating our framework's applicability in real world is the enhancement of the extensively utilized PASCAL dataset~\cite{pascal-voc-2012}.
% This enhancement involves the production of a corrected version, refer as PASCAL+, which is achieved through sophisticated refinement applied to the original PASCAL.

Our main contributions are summarized as follows:
\vspace{-3mm}
\begin{itemize}
    % \vspace{-0.5mm}
    \item We provide theoretical and empirical justifications 
    on the efficacy of the correction query, compared to the
    classification query (Section~\ref{sec:correction-query} and~\ref{sec:user-study}). 
    % \vspace{-1.5mm}
    \item We propose an active label correction framework, 
     leveraging the correction query and foundation models,
     where the look-ahead acquisition function 
     enables selecting informative and diverse pixels to be corrected (Section~\ref{sec:diversified-pixel-pool} and \ref{sec:look-ahead}).
     %while diversifying  them. 
     %considering the impact of label expansion 
    % \item We propose an active label correction framework with correction queries, whose efficiency is validated both theoretically and empirically.
    % designed to iteratively rectify the pseudo labels of pixels. 
    % to construct a clean semantic segmentation dataset.
    % \vspace{-1mm}
    % \item To maximize the benefits of corrections, we introduce a look-ahead acquisition function considering the impact of label expansion.
    % % diversified pixel pool combined with a look-ahead acquisition function taking into account label expansion effects.
    % to overcome the limitations of pixel-wise queries.
    % \vspace{-1.5mm}
    \item To achieve comparable performance with SOTA active semantic segmentation methods, we only use 33\% to 50\% of budgets on various datasets (Section~\ref{sec:exp-main}). 
    % Compared to the current leading superpixel-based active learning methods, we only use 33\% to 50\% of budgets on various datasets.
    % We achieve comparable performance to current leading superpixel-based active learning methods using only 33\% to 50\% of budget on various datasets.
    % \vspace{-1.5mm}
    \item Using the proposed framework, 
    we correct 2.6 million pixel labels in PASCAL
    and provide a revised version, called PASCAL+ (Section~\ref{sec:exp-pp}).
    % \vspace{-1.5mm}
\end{itemize}

%% file: Sections/2_related_work.tex
\section{Related Work}

\noindent\textbf{Active Learning for Segmentation.}
Active Learning (AL)~\cite{kim2023saal,saran2023streaming,yang2023towards} aims at increasing labeling efficiency by selectively annotating informative subsets of data.
In semantic segmentation, previous work focuses on two aspects: the design of labeling units and acquisition functions.
% In terms of labeling unit design, classical approaches explore image-based~\cite{sinha2019variational, yang2017suggestive}, pixel-based~\cite{shin2021all}, and patch-based~\cite{casanova2019reinforced, colling2020metabox+, golestaneh2020importance, mackowiak2018cereals} data selection and labeling.
In terms of labeling unit design, classical approaches explore image-based~\cite{yang2017suggestive, sinha2019variational} and patch-based~\cite{mackowiak2018cereals, casanova2019reinforced} selection.
Recently, superpixel-based approaches~\cite{siddiqui2020viewal, cai2021revisiting, hwang2023active, Kim_2023_ICCV}, are gaining attention as they only require one click for labeling each region.
% which acquire annotations at superpixel levels, are proposed
% They argue that superpixel-based approaches are efficient as they only require few clicks per each region.
% In terms of acquisition functions, they generally focus on selecting regions where the trained model is uncertain, measured with entropy~\cite{mackowiak2018cereals, kasarla2019region}, the gap between top-1 and top-2 prediction~\cite{joshi2009multi, wang2016cost, cai2021revisiting, Kim_2023_ICCV, hwang2023active}, and prediction inconsistency for augmented data~\cite{siddiqui2020viewal}.
In terms of acquisition functions, they generally focus on selecting uncertain regions, measured with entropy~\cite{mackowiak2018cereals, kasarla2019region}, the gap between the top-1 and the top-2 predictions~\cite{joshi2009multi, wang2016cost, cai2021revisiting,  hwang2023active, Kim_2023_ICCV}.
% and prediction inconsistency for augmented data~\cite{siddiqui2020viewal}.
% \cite{colling2020metabox+, golestaneh2020importance, mackowiak2018cereals, kasarla2019region, joshi2009multi, wang2016cost, cai2021revisiting, Kim_2023_ICCV, hwang2023active, siddiqui2020viewal}.
% Previous approaches utilize entropy~\cite{colling2020metabox+, golestaneh2020importance, mackowiak2018cereals, kasarla2019region}, the gap between top-1 and top-2 prediction~\cite{joshi2009multi, wang2016cost, cai2021revisiting, Kim_2023_ICCV, hwang2023active}, and prediction inconsistency between data augmentation~\cite{siddiqui2020viewal} as uncertainty measurement.
% These methods are based on the assumption that learning from uncertain areas leads to more significant improvements in model accuracy.
% Recent approaches additionally combine uncertainty measures with factors like data density or class balancing terms~\cite{cai2021revisiting, Kim_2023_ICCV, hwang2023active}.
% While AL focuses on collecting labels from scratch, the proposed method focuses on correcting erroneous labels within an existing dataset.
% While these approaches deal with acquiring labels from unlabeled data, our proposed method focuses on correcting erroneous labels within an existing dataset, which can improve dataset quality and labeling efficiency by correcting errors in auto-labeled data
While conventional AL methods collect labels from scratch, the proposed method starts from the initial pseudo labels from foundation models, correcting erroneous labels.

\noindent\textbf{Noisy Label Detection.}
The studies in noisy label detection (NLD) aim to identify incorrect labels efficiently by selecting error-like samples.
% Annotation Error Detection (AED) aims to identify incorrect labels efficiently by selecting error-like samples.
% Prior research~\cite{northcutt2021labelerrors} has already discovered label errors in ten types of computer vision, natural language, and audio datasets.
% AED selects error-like sample candidates for examination and then has these candidates reviewed by an oracle to efficiently check for errors.
In computer vision, methods for robust training toward label noise often include NLD components~\cite{natarajan2013learning, xiao2015learning, patrini2017making, han2018co, ren2018learning, song2022learning}, and recently, there is an increase in studies focusing solely on NLD~\cite{muller2019identifying, northcutt2021labelerrors}.
% AED methods for semantic segmentation focus on aggregating pixel-wise calculated error scores into labeling units such as an image or superpixel.
NLD methods for semantic segmentation aggregate pixel-wise error scores into labeling units, like an image or superpixel.
% One of them~
\citet{lad2023segmentation} aggregate per-pixel error scores from Confident Learning~\cite{northcutt2021confident} into per-image scores, % using softmin, 
while \citet{rottmann2023automated} average error scores from locally connected components sharing the same pseudo label.
% While most existing methods do not consider the information content of corrected samples, the Active Label Correction (ALC) methods~\cite{kremer2018robust, kim2022active, bernhardt2022active} select informative samples among error-like candidates beneficial for learning.
% In addition to error detection
Recently, the Active Label Correction (ALC)~\cite{bernhardt2022active, kim2022active} methods identify noisy labels and correct them in a classification task.
Our work is the first ALC method for semantic segmentation, correcting pixel labels and expanding them to their corresponding superpixels.
% the superpixels they are associated with.
% which selects
% where the proposed method selects 
% informative error-like regions by considering the locality of segmentation masks.
% selects informative samples among error-like candidates 
% label correction in semantic segmentation: \cite{lad2023segmentation} (image-wise), \cite{rottmann2023automated} (Superpixel-wise)
% active label correction in classification: \cite{kim2022active}
% disparity between label error detection and labeling in classification and segmentation
% limitations in segmentation (image-wise and size)

% Interestingly, repeating queries~\cite{shah2016no}

\noindent\textbf{Efficient Query Design.}
Designing a practical and cost-effective annotation query is crucial, as it directly impacts annotation budgets. 
In semantic segmentation, various approaches have been explored, including classification queries asking for a specific class~\cite{cai2021revisiting,Kim_2023_ICCV}, one-bit queries requesting yes or no responses~\cite{hu2020one}, and multi-class queries obtaining all classes in a superpixel~\cite{hwang2023active}.
Recently, there have been studies on efficiently constructing datasets using foundation models. 
For instance,~\citet{wang2023samrs} leverages these models for automated labeling in remote sensing imagery, and~\citet{qu2023abdomenatlas} focuses on building large medical datasets with them.
However, its query form is stagnant in previous query types.
By employing the initial pseudo labels from foundation models, we suggest correction queries that only request the correct label when the given pseudo label is incorrect.

%% file: Sections/3_method.tex
\section{Active Label Correction Framework}
Given an initial noisy dataset $\mathcal{D}_0$, we consider an active label correction (ALC) scenario operating with pixel-wise labeling.
Each query to an oracle annotator requests the accurate label $y \in \mathcal{C} := \{1,2,..., C\}$ for an associated pixel $x$. 
In contrast to active learning (AL), which commences with an unlabeled image set, ALC focuses on progressively refining a labeled dataset $\mathcal{D}_0$ which may include noisy labels.
For each round $t$, we issue a batch $\mathcal{B}_t$ of $B$ queries from a pixel pool $\mathcal{X}_t$ and train a model $\theta_t$ with the corrected annotations obtained so far.

In the following, we first prepare an initial dataset for correction (Section~\ref{sec:initial-dataset-preparation}).
After that, we present a correction query that requests for rectifying pseudo labels of pixels (Section~\ref{sec:correction-query}). 
% We build a diversified pixel pool from the initial dataset (Section~\ref{sec:diversified-pixel-pool}) 
To fully enjoy the corrections, we introduce a look-ahead acquisition function, which selects from a diversified pixel pool (Section~\ref{sec:diversified-pixel-pool}), considering the effect of label expansion (Section~\ref{sec:look-ahead}).
The overall procedure is summarized in Algorithm~\ref{algorithm1}.

% https://github.com/IDEA-Research/Grounded-Segment-Anything
\subsection{Initial Dataset Preparation}
\label{sec:initial-dataset-preparation}
For ALC, an initial segmentation dataset is essential, and we can start with well-known datasets like Cityscapes~\cite{Cordts2016Cityscapes} or PASCAL VOC (PASCAL)~\cite{pascal-voc-2012}. 
However, the presence of labeled datasets may be impractical in many domains.
Employing AL is one method for preparing labeled datasets.
However, AL typically builds datasets through random pixel~\cite{shin2021all} or superpixel labeling~\cite{cai2021revisiting} leading to lots of budgets and rounds, as it starts from unlabeled images, commonly known as the cold-start problem~\cite{mahmood2021low}.
Away from conventional AL methods, we utilize recent foundation models to construct segmentation datasets.

% Recently, there has been a surge in research focused on referring image segmentation and the development of foundation models in segmentation tasks.
% We adopt the latest techniques to generate initial datasets. 
Recently, foundation models for zero-shot segmentation have been emerged.
For example, Grounded-SAM, a fusion of Grounding DINO~\cite{liu2023grounding}
% (for detecting objects with text prompts)
and Segment Anything Model~\cite{Kirillov_2023_ICCV}
% (for segmenting from the detected objects),
is capable of detecting and segmenting objects based on text prompts.
Each class is identified with its own text prompt, and we can obtain the initial pseudo labels by using a series of $|\mathcal{C}|$ text prompts, one for each class.
% Each class can be represented by a unique text prompt, and the initial labeling for the unlabeled image set is achieved by using a series of $|\mathcal{C}|$ text prompts.
We solve the problem of multi-classes in object detection by giving each object the most likely class as a pseudo-label.
% We address the multi-class issue in object detection by assigning the most likely single class as a pseudo-label.
% using the argmax function.
Figures~\ref{fig:org-images} and \ref{fig:Grounded SAM-images} display examples of the unlabeled images in PASCAL and corresponding initial pseudo labels generated by Grounded-SAM.
% Compared with the ground-truth in Figure.~\ref{fig:org-labels}, the initial labels appear satisfactory, yet Table~\ref{tab:grounded-threshold} represents that the model's performance trained with the initial pseudo-labels is considerably affected by noisy labels.
% Therefore, active label correction is essential for rectifying these noisy labels.
% , thereby enhancing the dataset's fidelity for subsequent machine learning applications.

\begin{algorithm}[t!]
\caption{Proposed Framework}
\begin{algorithmic}[1]
\REQUIRE 
Batch size $B$, and final round~$T$. 
\STATE Prepare initial dataset $\mathcal{D}_0$ requiring label correction
\STATE Obtain model $\theta_0$ training with $\mathcal{D}_0$ via~\eqref{eq:ce}
\FOR{$t = 1, 2, \dots, T$}
    \STATE Construct diversified pixel pool $\mathcal{X}^d_t$ via~\eqref{eq:cos-sim-pixel}
    \STATE Correct labels of selected $B$ pixels $\mathcal{B}_t \subset \mathcal{X}^d_t$ via~\eqref{eq:sim}
    \STATE Expand corrected labels to corresponding superpixels
    \STATE Obtain model $\theta_{t}$ training with corrected $\mathcal{D}_{t}$ via~\eqref{eq:ce2}
\ENDFOR
% \STATE Update dataset $\mathcal{D}_{T}$ with model $\theta_{T}$ \\
\STATE \textbf{return} $\mathcal{D}_T$ and $\theta_T$
\end{algorithmic}
\label{algorithm1}
\end{algorithm}

\noindent\textbf{Warm-start.}
\label{sec:warm-start}
In contrast to the cold-start problem in AL, our ALC benefits from warm-start thanks to the initial labels provided by foundation models.
In Appendix~\ref{app:text-prompts}, detailed descriptions of text prompts for warm-start are provided.
To obtain $\theta_0$, we initialize $\theta$ to a model pre-trained on ImageNet~\cite{deng2009imagenet}.
We then train it to reduce the following cross-entropy (CE) loss:
\begin{equation}
\hat{\mathbb{E}}_{(x, y) \sim {\mathcal{D}}_0} [ \text{CE}(y, f_\theta (x))] \;,
\label{eq:ce}
\end{equation}
where $f_\theta( x ) \in \mathbb{R}^{|\mathcal{C}|}$ represents the estimated class probability for pixel $x$ by the model $\theta$.
% where $f_\theta( x ) \in \mathbb{R}^{|\mathcal{C}|}$ represents $\theta$'s estimate of class probability on pixel $x$.
Here, the difference lies in $\mathcal{D}_0$: AL uses only partial $y$, while ALC can access all $y$ for each pixel $x$.
However, compared to ground-truth in Figure~\ref{fig:org-labels}, the initial pseudo-labels in Figure~\ref{fig:Grounded SAM-images} contain noisy labels.
This results in negative impacts on the model's performance, as shown in Table~\ref{tab:grounded-threshold}.
Therefore, active label correction is essential for rectifying these noisy labels. 
% Table~\ref{tab:grounded-threshold} represents that the model's performance trained with the initial pseudo-labels is considerably affected by noisy labels.
% Therefore, active label correction is essential for rectifying these noisy labels.

% to first ask the annotator whether the pixel label is correct or not.
% If the pseudo label is already correct, annotator  
% we then ask the annotator to fix it to the clean label with a click.
% refine the dataset in earnest by relabeling a pixel $x$ selected from the entire pool of pixels $\mathcal{X}$.

\begin{figure}[t!]
    \centering
    \includegraphics[width=0.99\linewidth]{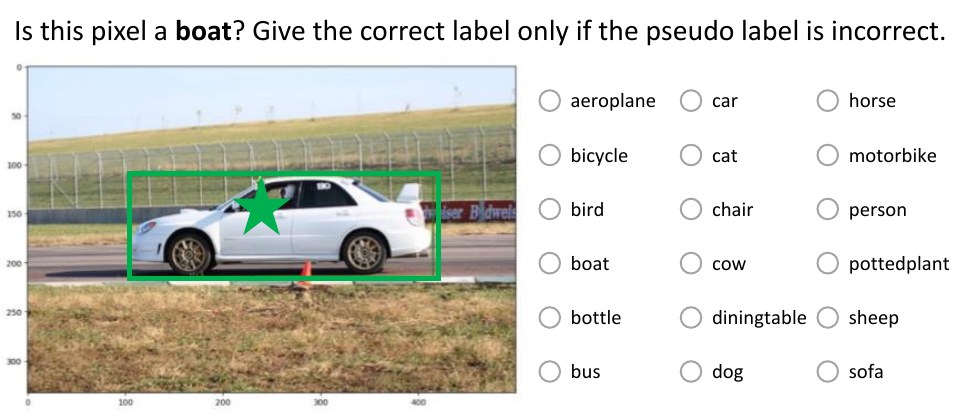}
    \caption{\textit{An example of correction query}. 
    % Correction queries are cost-efficient by bypassing labeling when the pseudo label is correct. 
    Correction query presents an instruction requesting a label for a representative pixel (green star), an image displaying an object within a bounding box (green rectangle), and possible class options.
    }
    \label{fig:cor-query}
\end{figure}

\subsection{Correction Query}
\label{sec:correction-query}

Once we prepare the initial dataset for correction, we use our correction query to rectify the pseudo labels of pixels.
As the number of classes increases, the classification query asking for the precise label of a pixel can become more time-consuming~\cite{zhang2022one}.
In contrast, our correction query lowers the overall cost by reducing the number of classification queries needed, allowing annotators to bypass labeling when the pseudo label is already correct.
% , borrowing the concept of a one-bit query.
Specifically, we use the instruction with a pseudo label on a pixel, written as follows:
\begin{center}
\it Give the correct label only if the pseudo label is incorrect.
\end{center}

Figure~\ref{fig:cor-query} and Appendix~\ref{app:user-study} provide detailed descriptions of our correction query.
In the following, we information-theoretically compare 
the expected costs of classification and correction queries, denoted by 
$C_{\textnormal{cls}}$ and $C_{\textnormal{cor}}$, respectively.
\begin{theorem}
\label{the:queries}
Assume the information-theoretic annotation cost \citep{hu2020one}
of selecting one out of $L$ possible options to be $\log_2 L$. Let $L \ge 2$ be the number of classes, and $p$ be the probability that
the pseudo label is correct. Then, 
$C_{\textnormal{cls}}(L) = \log_2 L$ and $C_{\textnormal{cor}} (L, p)= p+ (1 - p) \log_2 L$.
Thus, for any $p \in [0,1]$ and $L \ge 2$, 
%the gap is given as:
% \begin{align} \label{eq:gap}
% C_{\textnormal{cls}}(L) - C_{\textnormal{cor}}(L,p) = p \log_2 \frac{L}{2} \ge 0\;.
% \end{align}
\begin{align} \label{eq:ratio}
1 - \frac{C_{\textnormal{cor}}(L,p)}{C_{\textnormal{cls}}(L)} = \left( 1- \frac{1}{\log_2 L} \right) p \ge 0\;.
\end{align}
%  $C_{\textnormal{cls}}(L) - C_{\textnormal{cor}}(L,p) = p \log_2 \frac{L}{2}$ is non-negative for any $p$ and $L$ and increasing in $L$.
% %the costs are given by:
% Assume the labeling cost for choosing one out of $L$ options 
% is $\log_2 L$ followed by the information-theoretic analysis in \citet{hu2020one}.
% Let $L$ be the number of classes and $p$ be the probability of correctly identifying noisy pixel labels. Then,
% \begin{align*}
% C_{\textnormal{cls}} &= \log_2 L, & C_{\textnormal{cor}} &= p + (1 - p) \log_2 L.
% \end{align*}
% % as $p \leq 1$, 
% In addition, 
% $C_{\textnormal{cls}}(L) \geq C_{\textnormal{cor}}(L,p)$
% for any $L \ge 2$ and $p \in [0,1]$
% and the gap $(C_{\textnormal{cls}}(L) - C_{\textnormal{cor}}(L,p))$ is increasing in $L$.
% If $p \ge \frac{1}{L}$, then
% $C_{\textnormal{cls}}(L) \geq C_{\textnormal{cor}}(L,p)$
% and the gap $(C_{\textnormal{cls}}(L) - C_{\textnormal{cor}}(L,p))$ is increasing in $L$.
%for any $L \ge 2$ and $p \in [0,1]$. 
%In addition, 
%the difference 
%$C_{\textnormal{cls}}(L) - C_{\textnormal{cor}}(L,p)$ is increasing with respect to $p$ and $L$.
% where 
% $C_{\textnormal{cls}} > C_{\textnormal{cor}}$ for $L \geq 3$. 
\end{theorem}
% Hence, as $p \leq 1$, $C_{\textnormal{cls}} > C_{\textnormal{cor}}$ for $L \geq 3$ and  $C_{\textnormal{cls}} \geq C_{\textnormal{cor}}$ for $L=2$. 
\begin{proof}
The correction query can be interpreted as a binary question if the pseudo label is correct, and a $L$-ary one otherwise.
Recalling the definition of $p$ and $C_{\textnormal{cls}}(L) = \log_2 L$, we have $C_{\textnormal{cor}} (L, p)= p \log_2 2 + (1 - p) \log_2 L$.
\end{proof}
\iffalse
The correction query first asks a binary query on the correctness of the given pseudo label. Only when the pseudo label is incorrect, a classification query asks for the true label which cannot be the pseudo label, i.e., it is an $(L-1)$-ary question.
Thus, $C_{\textnormal{cor}} (L, p) =  \log_2 2 + (1 - p) \log_2 (L-1)$. (however, this is not true... two sequential questions are not independent!, e.g., L = 2, we know $C_cls$ =$ C_cor$)
\fi
% Then, the second statement is straightforward since
% %the gap $(C_{\textnormal{cls}}(L) -  C_{\textnormal{cor}} (L, p))$ 
% \begin{align}
% C_{\textnormal{cls}}(L) -  C_{\textnormal{cor}} (L, p)
% & = \log_2 L - p - (1 - p) \log_2 L \\
% & =  p \log_2 \frac{L}{2}   \;.
% \end{align}
% \begin{align}
% C_{\textnormal{cls}}(L) -  C_{\textnormal{cor}} (L, p)
% & = \log_2 L -  1 - (1 - p) \log_2 (L-1)  \\
% & =  p \log_2 \frac{L}{2}   \;.
% \end{align}
% \end{proof}
The costs of both correction and classification queries are the same if $L=2$. Indeed, those are logically identical when $L=2$.
In~\eqref{eq:ratio}, 
the cost-saving rate using 
the correction query instead on the classification one
is computed as $\left( 1- \frac{1}{\log_2 L} \right) p$,
which is increasing in $p$ and $L$.
Hence,
using the correction query is particularly beneficial 
when the number of classes is large
or the pseudo labels can be obtained accurately.
In addition, a user study on correction queries experimentally confirms their practical effectiveness in Section~\ref{sec:user-study}.

\iffalse
advantage $(C_{\textnormal{cls}}(L) -  C_{\textnormal{cor}} (L, p))$ of using the correction query becomes clearer as the pseudo label accuracy $p$ is improved.
In addition, if $p >0$, 
the advantage is strictly increasing in $L$.
As we have 
%This shows an advantage of the correction 
In general, segmentation datasets consist of many classes, such as $L=20$ in PASCAL and $L=19$ in Cityscapes, the cost-effectiveness of correction queries is definite in real-world scenarios.
\fi
% If $p > 0.5$ and $L > 2$, the cost of our correction query is lower than classification queries.
% However, in general, segmentation datasets consist of many classes, such as $L=20$ in PASCAL and $L=19$ in Cityscapes, the cost-effectiveness of correction queries can still be shown even when $p$ is smaller than 0.5.

\subsection{Diversified Pixel Pool}
\label{sec:diversified-pixel-pool}
% \noindent\textbf{Diversified Pixel Pool.}
Employing pixel-wise queries is instrumental in constructing error-free segmentation datasets.
However, examining each pixel with an acquisition function requires substantial time and memory.
Furthermore, as adjacent pixels often share similar acquisition values, there exists a risk of lacking diversity in the selected pixels, i.e., pixels in a certain area of the image with high acquisition values may be picked simultaneously.
To tackle these challenges at once, we propose a diversified pixel pool $\mathcal{X}^d$, which is a subset of the total pixel set $\mathcal{X}$, as follows:
% denoted as $\mathcal{X}^d \subset \mathcal{X}$:
\begin{equation}
\mathcal{X}^d := \{x_1, x_2, \dots, x_{|\mathcal{S}|} \}\;,
\label{eq:div-pixel-pool}
\end{equation}
where each $x_i$ represents a key pixel from the superpixel $s_i$ within the set of superpixels $\mathcal{S}$.

Specifically, starting with a model $\theta_{t-1}$ trained on the dataset $\mathcal{D}_{t-1}$ from the previous round, we construct a diversified pixel pool $\mathcal{X}_t^d := \{ x_{t1},x_{t2}, \dots, x_{t|\mathcal{S}|} \}$ for the current round $t$.
For ease of explanation, we refer to $\theta_{t-1}$ simply as $\theta$, $x_{ti}$ as $x_i$ and $\mathcal{X}^d_t$ as $\mathcal{X}^d$.
% In this process, 
We select a representative pixel $x_i$ from each superpixel $s_i$ based on the highest cosine similarity as:
% between the pixel and its corresponding superpixel:
\begin{equation}
x_i := \argmax_{x \in s_i} \frac{f_\theta(x) \cdot f_\theta(s'_i)}{\|f_\theta(x)\| \| f_\theta(s'_i) \| }\;, 
% f_\theta(s') := \frac{\sum_{x \in s'} f_\theta(x)}{|\{x : x\in s'\}|}
% \text{cos} \Big( f_\theta(x), \frac{1}{|s'_i|} \sum_{x' \in s'_i} f_\theta(x') \Big)\;.
\label{eq:cos-sim-pixel}
\end{equation}
where $f_\theta(s) := \frac{\sum_{x \in s} f_\theta(x)}{|\{x : x\in s|\}}$ represents the averaged class prediction for superpixel $s$. 
% a subset $s'_i$ of the superpixel $s_i$.
% where $f_\theta( x ) \in \mathbb{R}^{|\mathcal{C}|}$ represents the estimated class probability for pixel $x$ by the model $\theta$.
To address the flaws in superpixels and ensure more uniformity of pixel labels within them, we employ a subset $s'$ rather than the complete set $s$.
We start by defining the pseudo dominant label $\text{D}_\theta(s)$, which serves as the representative label for superpixel $s$ according to model $\theta$, as follows:
\begin{equation}
\text{D}_\theta(s) :=\argmax_{c \in \mathcal{C}} | \{x \in s : y_\theta(x) = c \} |\;,
\end{equation}
where $y_{\theta}(x) := \argmax_{c \in \mathcal{C}} f_\theta(c;x)$ is the estimated label for pixel $x$ using model $\theta$.
Subsequently, we form the subset $s'$, consisting of pixels that align with the pseudo dominant label $\text{D}_\theta(s)$, as follows:
\begin{equation}
s' := \{ x \in s : y_{\theta}(x) = \text{D}_\theta(s) \} \;.
\end{equation}
After that, we select the pixel that best represents $s'$ for each superpixel based on~\eqref{eq:cos-sim-pixel}, 
% the pixel most representative of $s'$ is selected per superpixel according to \eqref{eq:cos-sim-pixel}, 
contributing to the formation of a diverse pixel pool in \eqref{eq:div-pixel-pool}.
% It is important to 
We highlight that the proposed diversified pixel pool reduces time and memory usage and lessens the redundancy issue in the chosen pixels.
% computational resources and redundancy issues in the selected pixels.
% actual budget expenditures.

\noindent\textbf{Remarks.}
While various superpixel generation algorithms~\cite{achanta2012slic,van2012seeds} can be used for $\mathcal{S}$ in~\eqref{eq:div-pixel-pool}, these standard algorithms typically group neighboring pixels based on similar inherent properties like color and maintain nearly uniform sizes.
% However, 
Recent research indicates that semantically considered superpixels from a model are effective for AL in segmentation~\cite{Kim_2023_ICCV}.
Therefore, we opt to organize superpixels based on the objects identified by Grounded-SAM. 
% A comparative analysis of different superpixel algorithms is presented in Table~\ref{tab:superpixel-comparison}.

\subsection{Look-Ahead Acquisition Function}
\label{sec:look-ahead}
Once the set of pixels $\mathcal{X}^d_t$ for examination through an acquisition function is established, we select a pixel batch $\mathcal{B}_t \subset \mathcal{X}^d_t$ of size $B$ to be corrected.
In each round $t$, we iteratively select the most informative pixel, guided by the acquisition $a(x;\theta_{t-1})$:
% we proceed to identify the pixel with the highest unreliability with the acquisition function $a(x;\theta)$:
\begin{equation}
x^* := \argmax_{x \in \mathcal{X}^d_t} a(x;\theta_{t-1})\;.
\end{equation}
For simplicity, we refer to $\theta_{t-1}$ as $\theta$.
Recently, \citet{lad2023segmentation} propose a confidence in label (CIL), which evaluates the confidence of a given label $y$ for a pixel $x$, using the predictions of the model $\theta$ as follows:
\begin{equation}
a_{\text{CIL}}(x;\theta) := 1 - f_\theta(y;x)\;.
\label{eq:cil}
\end{equation}
The underlying assumption is that a pixel is likely mislabeled if the model demonstrates insufficient learning about that pixel's label.
However, correcting only a single pixel with each query is not only inefficient but also has minimal impact on the learning process.
To enhance the efficiency of pixel-wise query, we introduce a label expansion technique, which involves extending the corrected label of a pixel $x$ into pixels in the same superpixel $s$.

Accordingly, we suggest a look-ahead acquisition function that not only assesses the unreliability of a pixel $x$ as described in~\eqref{eq:cil}, but also takes into account the effect of label expansion into the superpixel $s$.
Here, we rename $x$ to $x_r$ as it serves as a representative pixel for $s$.
For a representative pixel $x_r$ of $s$, our acquisition function is defined as follows:
\begin{equation}
a_{\text{SIM}}(x_r; s, \theta):= \sum_{x \in s} \frac{f_\theta(x_r) \cdot f_\theta(x)}{\|f_\theta(x_r)\| \| f_\theta(x) \| } a_\text{CIL}(x;\theta) \;,
\label{eq:sim}
\end{equation}
where the cosine similarity between two feature vectors is related to the likelihood of correctly expanding the correct label of pixel $x_r$ to another pixel $x$.

% We propose an acquisition function that goes beyond merely evaluating the unreliability of a pixel as described in \eqref{eq:cil} and also considers the effect of label expansion when pixel label is corrected as:
% \text{cos} \big( f_\theta(x_i), f_\theta(x) \big)
% \js{Just use the definition of cosine similarity.}

We note that previous acquisitions including CIL in~\eqref{eq:cil} can be transformed easily to its look-ahead counterparts.
For instance, the look-ahead CIL (LCIL) acquisition can be defined by adjusting the weight of each pixel from the cosine similarity to the inverse of the superpixel size as:
\begin{equation}
a_{\text{LCIL}}(x_r; s, \theta):= \sum_{x \in s} \frac{1}{|s|} a_\text{CIL}(x;\theta) \;.
\label{eq:lcil}
\end{equation}
% The influence of various weights is experimented in X.
Finally, in round $t$, we select the $B$ most informative pixels from the diversified pixel pool $\mathcal{X}^d_t$ in order of SIM acquisition to form query batch $\mathcal{B}_t$.
% $B$ pixels in order of $a_\text{SIM}(x;s,\theta_{t-1})$
% from the diversified pixel pool $\mathcal{X}^d_t$ for query batch $\mathcal{B}_t$ at round $t$.
% \noindent\textbf{Learning with Label Expansion.}

After obtaining the clean labels of selected $B$ pixels, we expand them to the associated superpixels.
We finally construct the dataset $\mathcal{D}_t$ for round $t$ by combining the previous dataset $\mathcal{D}_{t-1}$ with the updated annotations.
Analogously to the warm-start, we initialize $\theta_t$ to a model pre-trained on ImageNet, minimizing the following CE loss:
\begin{equation}
\hat{\mathbb{E}}_{(x, y) \sim {\mathcal{D}}_t} [ \text{CE}(y, f_\theta (x))] \;.
\label{eq:ce2}
\end{equation}

%% file: Sections/4_experiment.tex
\section{Experiments}
\label{sec:exp}
\subsection{Experimental Setup}

\noindent\textbf{Datasets.}
We use three semantic segmentation datasets: Cityscapes~\cite{Cordts2016Cityscapes}, PASCAL VOC 2012 (PASCAL)~\cite{pascal-voc-2012}, and Kvasir-SEG~\cite{jha2020kvasir}.
Cityscapes comprises 2,975 training and 500 validation images with 19 classes, while PASCAL consists of 1,464 training and 1,449 validation images with 20 classes.
Kvasir-SEG is a medical dataset for polyp segmentation consists of 880 training and 120 validation images with 2 classes.

\noindent\textbf{Implementation Details.}
We adopt DeepLab-v3+ architecture~\cite{chen2018encoder} with Resnet101 pre-trained on ImageNet~\cite{deng2009imagenet} as our segmentation model.
During training, we use the SGD optimizer with a momentum of 0.9 and set a base learning rate of 0.1.
We decay the learning rate by polynomial decay with a power of 0.9.
% During training, we use   the SGD optimizer with a momentum of 0.9 and decay the learning rate by polynomial decay with a power of 0.9.
For Cityscapes, we resize training images to 768 $\times$ 768 and train a model for 30K iterations with a mini-batch size 16.
For PASCAL, we resize training images to 513 $\times$ 513 and train a model for 30K iterations with a mini-batch size 16.
For Kvasir-SEG, we resize training images to 352 $\times$ 352 and train a model for 6.3K iterations with a mini-batch size 32.
For the initial dataset generated with Grounded-SAM, we use the box threshold of 0.2 for Cityscapes and PASCAL, and 0.05 for Kvasir-SEG.

% More details are in the Appendix~\ref{sec:implementation-details}.

\subsection{Main Experiments}
\label{sec:exp-main}

\noindent\textbf{Baselines.}
Our Active Label Correction (\sftype{ALC}) method is compared 
% with active learning (AL) methods.
% and noisy label detection (NLD) in segmentation tasks.
% We compare \textit{ALC} 
with the state-of-the-art (SOTA) superpixel-based active learning (AL) methods: 
% for segmentation
\sftype{Spx}~\cite{cai2021revisiting}, \sftype{MerSpx}~\cite{Kim_2023_ICCV}, and \sftype{MulSpx}~\cite{hwang2023active}. 
% They are comparable to our pixel-wise labeling in that they perform active learning by acquiring the majority label for each superpixel.
They are chosen for two reasons: (1) Their measure of labeling cost is the same as ours, i.e., the number of label clicks. (2) They are SOTA methods in AL for segmentation.
Following conventional AL methods~\cite{cai2021revisiting}, we highlight the amount of annotation used to achieve 95\% performance of the fully supervised baseline, where \textit{95\%.} denotes performance.
% Following conventional Al methods~\cite{cai2021revisiting}, we highlight the amount of annotation used to achieve 95\% accuracy of the fully supervised baseline, where \textit{95\%.} denotes the 95\% of fully supervised performance.

\noindent\textbf{Evaluation Protocol.}
Given a limited budget, we identify and fix noisy pixel labels, and expand them to the related superpixels to construct the corrected dataset.
Then, we develop a model using the dataset and evaluate its effectiveness with mean Intersection over Union (mIoU).
% For every experiment, we report average performance of three trials, where the standard deviation is illustrated with 
In all experiments, we report the average results from three trials, with graph shading indicating the standard deviation.
% the standard deviation represented by the shading in the graphs.
We access the model not only on the test dataset but also on the training dataset to calculate the quality of the dataset itself.
% We not only evaluate the model with the mIoU on the test dataset but also by measuring the mIoU on the training dataset to calculate the quality of the constructed dataset itself.
% and set $\epsilon$ to 0 for expanding corrected pixel labels into superpixels.
% \todo{Need update: kvasir, epsilon}

\begin{figure}[t!]
    \captionsetup[subfigure]{font=footnotesize,labelfont=footnotesize,aboveskip=0.05cm,belowskip=-0.15cm}
    \centering
    \hspace{-3mm}
    \begin{subfigure}{.47\linewidth}
        \centering
        \begin{tikzpicture}
            \begin{axis}[
                legend style={
                    nodes={scale=0.6}, 
                    at={(2.23, 1.28)}, 
                    /tikz/every even column/.append style={column sep=2mm},
                    legend columns=3,
                },
                xlabel={The number of clicks},
                ylabel={\% of fully supervised mIoU (\%)},
                width=1.2\linewidth,
                height=1.25\linewidth,
                ymin=76,
                ymax=97,
                xtick={3, 6, 10, 15, 20, 25},
                ytick={80, 85, 90, 95},
                xlabel style={yshift=0.15cm},
                ylabel style={yshift=-0.6cm},
                xmin=0,
                xmax=27,
                label style={font=\scriptsize},
                tick label style={font=\scriptsize},
                xticklabel={$\pgfmathprintnumber{\tick}$K}
            ]
            % Full
            \addplot[cPink, very thick, mark size=2pt, mark options={solid}, dashed] coordinates {(0,-1)};
            \draw [cPink, very thick, dashed] (axis cs:-2,95) -- (axis cs:30,95);
            % CVPR, 22
            \addplot[cRed, very thick, mark=triangle*, mark size=2pt, mark options={solid}] table[col sep=comma, x=x, y=Spx]{Data/alc_pascal.csv};
            % ICCV, 23
            \addplot[orange, very thick, mark=diamond*, mark size=2pt, mark options={solid}] table[col sep=comma, x=x, y=MerSpx]{Data/alc_pascal.csv};
            % NeurIPS, 23
            \addplot[cGreen, very thick, mark=square*, mark size=2pt, mark options={solid}] table[col sep=comma, x=x, y=MulSpx]{Data/alc_pascal.csv};
            % Ours
            \addplot[cBlue, very thick, mark=pentagon*, mark size=2pt, mark options={solid}] table[col sep=comma, x=x, y=ALC]{Data/alc_pascal.csv};
            % Ours - norm
            \addplot[cBlue2, very thick, dashed, mark=pentagon*, mark size=2pt, mark options={solid}] table[col sep=comma, x=x, y=BALC]{Data/alc_pascal.csv};

            % CVPR, 22
            \addplot[name path=Spx-U, draw=none, fill=none] table[col sep=comma, x=x, y=Spx-U]{Data/alc_pascal.csv};
            \addplot[name path=Spx-D, draw=none, fill=none] table[col sep=comma, x=x, y=Spx-D]{Data/alc_pascal.csv};
            \addplot[cRed, fill opacity=0.3] fill between[of=Spx-U and Spx-D];
            % ICCV, 23
            \addplot[name path=MerSpx-U, draw=none, fill=none] table[col sep=comma, x=x, y=MerSpx-U]{Data/alc_pascal.csv};
            \addplot[name path=MerSpx-D, draw=none, fill=none] table[col sep=comma, x=x, y=MerSpx-D]{Data/alc_pascal.csv};
            \addplot[orange, fill opacity=0.3] fill between[of=MerSpx-U and MerSpx-D];
            % NeurIPS, 23
            \addplot[name path=MulSpx-U, draw=none, fill=none] table[col sep=comma, x=x, y=MulSpx-U]{Data/alc_pascal.csv};
            \addplot[name path=MulSpx-D, draw=none, fill=none] table[col sep=comma, x=x, y=MulSpx-D]{Data/alc_pascal.csv};
            \addplot[cGreen, fill opacity=0.3] fill between[of=MulSpx-U and MulSpx-D];
            % Ours
            \addplot[name path=ALC-U, draw=none, fill=none] table[col sep=comma, x=x, y=ALC-U]{Data/alc_pascal.csv};
            \addplot[name path=ALC-D, draw=none, fill=none] table[col sep=comma, x=x, y=ALC-D]{Data/alc_pascal.csv};
            \addplot[cBlue, fill opacity=0.3] fill between[of=ALC-U and ALC-D];
            % Ours
            \addplot[name path=BALC-U, draw=none, fill=none] table[col sep=comma, x=x, y=BALC-U]{Data/alc_pascal.csv};
            \addplot[name path=BALC-D, draw=none, fill=none] table[col sep=comma, x=x, y=BALC-D]{Data/alc_pascal.csv};
            \addplot[cBlue2, fill opacity=0.3] fill between[of=BALC-U and BALC-D];
            
            % \legend{95\%., ALC (ours), ALC (ours; normalized), MulSpx, MerSpx, Spx}
            \legend{95\%., Spx, MerSpx, MulSpx, ALC (ours), ALC (ours; normalized)}
            \end{axis}
        \end{tikzpicture}
        \caption{PASCAL}
    \label{fig:alc-vs-al-pascal}
    \end{subfigure}
    \hspace{1mm}    
    \begin{subfigure}{.47\linewidth}
        \centering
        \begin{tikzpicture}
            \begin{axis}[
                xlabel={The number of clicks},
                ylabel={\% of fully supervised mIoU (\%)},
                width=1.2\linewidth,
                height=1.25\linewidth,
                ymin=82,
                ymax=96.4,
                xtick={50, 100, 150, 200, 250},
                ytick={80, 85, 90, 95},
                xlabel style={yshift=0.15cm},
                ylabel style={yshift=-0.6cm},
                xmin=18,
                xmax=260,
                label style={font=\scriptsize},
                tick label style={font=\scriptsize},
                xticklabel={$\pgfmathprintnumber{\tick}$K}
            ]
            \draw [cPink, very thick, dashed] (axis cs:-2,95) -- (axis cs:300,95);
            % CVPR, 22
            \addplot[cRed, very thick, mark=triangle*, mark size=2pt, mark options={solid}] table[col sep=comma, x=x, y=Spx]{Data/alc_city.csv};
            % ICCV, 23
            \addplot[orange, very thick, mark=diamond*, mark size=2pt, mark options={solid}] table[col sep=comma, x=x, y=MerSpx]{Data/alc_city.csv};
            % NeurIPS, 23
            \addplot[cGreen, very thick, mark=square*, mark size=2pt, mark options={solid}] table[col sep=comma, x=x, y=MulSpx]{Data/alc_city.csv};
            % Ours
            \addplot[cBlue, very thick, mark=pentagon*, mark size=2pt, mark options={solid}] table[col sep=comma, x=x, y=ALC]{Data/alc_city.csv};
            % Ours
            \addplot[cBlue2, very thick, dashed, mark=pentagon*, mark size=2pt, mark options={solid}] table[col sep=comma, x=x, y=BALC]{Data/alc_city.csv};

            % CVPR, 22
            \addplot[name path=Spx-U, draw=none, fill=none] table[col sep=comma, x=x, y=Spx-U]{Data/alc_city.csv};
            \addplot[name path=Spx-D, draw=none, fill=none] table[col sep=comma, x=x, y=Spx-D]{Data/alc_city.csv};
            \addplot[cRed, fill opacity=0.3] fill between[of=Spx-U and Spx-D];
            % ICCV, 23
            \addplot[name path=MerSpx-U, draw=none, fill=none] table[col sep=comma, x=x, y=MerSpx-U]{Data/alc_city.csv};
            \addplot[name path=MerSpx-D, draw=none, fill=none] table[col sep=comma, x=x, y=MerSpx-D]{Data/alc_city.csv};
            \addplot[orange, fill opacity=0.3] fill between[of=MerSpx-U and MerSpx-D];
            % NeurIPS, 23
            \addplot[name path=MulSpx-U, draw=none, fill=none] table[col sep=comma, x=x, y=MulSpx-U]{Data/alc_city.csv};
            \addplot[name path=MulSpx-D, draw=none, fill=none] table[col sep=comma, x=x, y=MulSpx-D]{Data/alc_city.csv};
            \addplot[cGreen, fill opacity=0.3] fill between[of=MulSpx-U and MulSpx-D];
            % Ours
            \addplot[name path=ALC-U, draw=none, fill=none] table[col sep=comma, x=x, y=ALC-U]{Data/alc_city.csv};
            \addplot[name path=ALC-D, draw=none, fill=none] table[col sep=comma, x=x, y=ALC-D]{Data/alc_city.csv};
            \addplot[cBlue, fill opacity=0.3] fill between[of=ALC-U and ALC-D];
            % Ours
            \addplot[name path=BALC-U, draw=none, fill=none] table[col sep=comma, x=x, y=BALC-U]{Data/alc_city.csv};
            \addplot[name path=BALC-D, draw=none, fill=none] table[col sep=comma, x=x, y=BALC-D]{Data/alc_city.csv};
            \addplot[cBlue2, fill opacity=0.3] fill between[of=BALC-U and BALC-D];

            \end{axis}
        \end{tikzpicture}
        \caption{Cityscapes}
    \end{subfigure}
    \caption{{\em Effect of active label correction.} \sftype{ALC} shows comparable results on both datasets with much fewer clicks.
    \sftype{ALC (normalized)} reflects the reduced budget of correction queries with normalization by Theorem~\ref{the:queries}.} 
    % \js{change the legend name for cyan ours, perhaps ACL(ours; normalized)}}
    \label{fig:alc-vs-al}
    \vspace{1mm}
\end{figure}
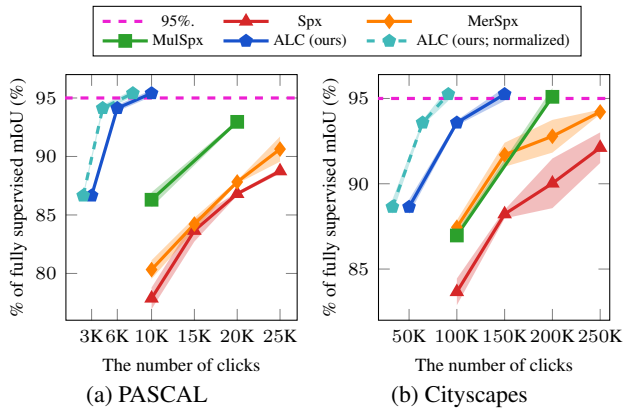

\begin{table}[t!]
\caption{{\em User study for different queries.} Our correction query $C_\textnormal{cor}$ proves to be more cost-effective compared to classification query $C_\textnormal{cls}$.}
\vspace{1mm}
\begin{center}
\begin{small}
\setlength\tabcolsep{3pt}
\centering
\begin{tabular}{l|ccc}
\toprule
Query & Total time (s) & Time per query (s) & Accuracy (\%) \\ \midrule
$C_\textnormal{cls}$ & 126.1$_{\pm 19.8}$ & 6.31$_{\pm 0.99}$ & 95.0$_{\pm 3.3}$ \\
$C_\textnormal{cor}$ & \textbf{95.1$_{\pm 9.0}$} & \textbf{4.76$_{\pm 0.45}$} & 95.0$_{\pm 4.0}$ \\
\bottomrule
\end{tabular}
\end{small}
\end{center}
% \vspace{-3.5mm}
\label{tab:user-study-results}
\end{table}

\noindent\textbf{Active Label Correction vs. Active Learning.}
% In Figure.~\ref{fig:alc-vs-al}, we compare the effectiveness of our framework, named \textit{ALC}, with current AL methods over various budget levels, represented by the number of clicks, for both PASCAL and Cityscapes datasets.
In Figure~\ref{fig:alc-vs-al}, we show the effectiveness of our framework, named \sftype{ALC}, compared with current AL methods over various budget levels, represented by the number of clicks, for both PASCAL and Cityscapes datasets.
% the effectiveness of our framework called \textit{ALC} is compared with current AL methods across diverse the number of clicks, which represents the budget constraints in AL, for both the PASCAL and Cityscapes datasets.
Due to variations in models and hyperparameters used in previous methods, we ensure a fair comparison by evaluating the percentage of fully supervised mIoU, where additional comparisons with absolute mIoU is reported in Appendix~\ref{sec:abs-alc-al}.
% We note that there only exists minimal variances in the performance of the fully supervised models across different models.
% The results illustrate that our solid-line \textit{ALC}, which utilizes initial pseudo-labels and \textit{SIM} acquisition, substantially reduces the necessary budgets to achieve 95\% target performance.
The results illustrate that our \sftype{ALC} substantially reduces the necessary budgets to achieve 95\% target performance.
% In Figure.~\ref{fig:alc-vs-al}, we only use 30\% budgets in PASCAL (20K vs. 6K), and 75\% budgets in Cityscapes (200K vs. 150K) to attain a comparable mIoU.
% In particular, we only use 30\% budgets in PASCAL (20K vs. 6K), and 75\% budgets in Cityscapes (200K vs. 150K) to attain a comparable mIoU.
% In particular, \textit{ALC} only requires 6K and 150K clicks to achieve 95\% performance of the fully supervised baseline for PASCAL and Cityscapes, respectively, which is 30\% and 75\% budget of that of the previous SOTA.
Specifically, \sftype{ALC} achieves 95\% of the fully supervised baseline performance with just 6K clicks for PASCAL and 150K clicks for Cityscapes.
This is only 30\% and 75\% of the budget required by the previous SOTA methods, respectively.
Even when considering the efficient labeling cost of correction queries in Theorem~\ref{the:queries}, the cost of our proposed method reduces to 68\% of its original version, where $p$ in (\ref{eq:ratio}) is 0.27 and 0.5 in PASCAL and Cityscapes, respectively.
This result is denoted as \sftype{ALC (normalized)} in Figure~\ref{fig:alc-vs-al}.
% In addition, by applying the labeling cost of correction queries in Theorem~\ref{the:queries}, the solid-line \textit{ALC} turns into dashed-line \textit{ALC} (normalized).
% This leads to an average budget reduction of 46\%, with $p$ being 0.73 in PASCAL and 0.50 in Cityscapes.
% As shown in Figure.~\ref{fig:alc-vs-al-pascal}, we only use 30\% budgets in PASCAL (20K vs. 6K), and 75\% budgets in Cityscapes (200K vs. 150K) to achieve a comparable mIoU.
% to achieve a comparable mIoU 50\% times in PASCAL, and 
% previous methods require two to three times more budget to achieve a comparable mIoU in PASCAL, and 
%, i.e., 3K vs. 10K and 10K vs. 20K in PASCAL.

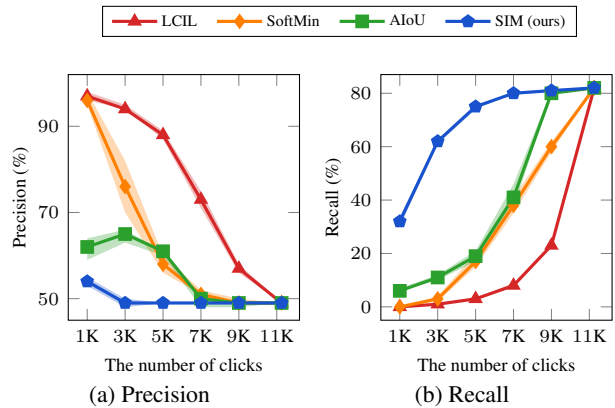
\begin{figure}[t!]
    \captionsetup[subfigure]{font=footnotesize,labelfont=footnotesize,aboveskip=0.05cm,belowskip=-0.15cm}
    \centering
    \hspace{-3mm}
    \begin{subfigure}{.47\linewidth}
        \centering
        \begin{tikzpicture}
            \begin{axis}[
                legend style={
                    nodes={scale=0.6}, 
                    at={(2.2, 1.28)}, 
                    /tikz/every even column/.append style={column sep=2mm},
                    legend columns=-1,
                },
                label style={font=\scriptsize},
                tick label style={font=\scriptsize},
                width=1.2\linewidth,
                height=1.25\linewidth,
                xlabel=The number of clicks,
                ylabel=Precision (\%),
                xmin=0, xmax=12.2,
                ymin=45, ymax=102,
                xtick={1, 3, 5, 7, 9, 11},
                xticklabel={$\pgfmathprintnumber{\tick}$K},
                ytick={50, 70, 90},
                xlabel style={yshift=0.15cm},
                ylabel style={yshift=-0.6cm},
            ]
            % CIL
            \addplot[cRed, very thick, mark=triangle*, mark size=2pt, mark options={solid}] table[col sep=comma, x=x, y=CIL]{Data/precision_pascal.csv};
            % SoftMin
            \addplot[orange, very thick, mark=diamond*, mark size=2pt, mark options={solid}] table[col sep=comma, x=x, y=SoftMin]{Data/precision_pascal.csv};
            % AIoU
            \addplot[cGreen, very thick, mark=square*, mark size=2pt, mark options={solid}] table[col sep=comma, x=x, y=AIoU]{Data/precision_pascal.csv};
            % SIM
            \addplot[cBlue, very thick, mark=pentagon*, mark size=2pt, mark options={solid}] table[col sep=comma, x=x, y=SIM]{Data/precision_pascal.csv};

            % CIL
            \addplot[name path=CIL-U, draw=none, fill=none] table[col sep=comma, x=x, y=CIL-U]{Data/precision_pascal.csv};
            \addplot[name path=CIL-D, draw=none, fill=none] table[col sep=comma, x=x, y=CIL-D]{Data/precision_pascal.csv};
            \addplot[cRed, fill opacity=0.3] fill between[of=CIL-U and CIL-D];
            % SoftMin
            \addplot[name path=SoftMin-U, draw=none, fill=none] table[col sep=comma, x=x, y=SoftMin-U]{Data/precision_pascal.csv};
            \addplot[name path=SoftMin-D, draw=none, fill=none] table[col sep=comma, x=x, y=SoftMin-D]{Data/precision_pascal.csv};
            \addplot[orange, fill opacity=0.3] fill between[of=SoftMin-U and SoftMin-D];
            % AIoU
            \addplot[name path=AIoU-U, draw=none, fill=none] table[col sep=comma, x=x, y=AIoU-U]{Data/precision_pascal.csv};
            \addplot[name path=AIoU-D, draw=none, fill=none] table[col sep=comma, x=x, y=AIoU-D]{Data/precision_pascal.csv};
            \addplot[cGreen, fill opacity=0.3] fill between[of=AIoU-U and AIoU-D];
            % SIM
            \addplot[name path=SIM-U, draw=none, fill=none] table[col sep=comma, x=x, y=SIM-U]{Data/precision_pascal.csv};
            \addplot[name path=SIM-D, draw=none, fill=none] table[col sep=comma, x=x, y=SIM-D]{Data/precision_pascal.csv};
            \addplot[cBlue, fill opacity=0.3] fill between[of=SIM-U and SIM-D];

            \legend{LCIL, SoftMin, AIoU, SIM (ours)}
            \end{axis}
        \end{tikzpicture}
        \caption{Precision}
        \label{fig:pascal-precision}
    \end{subfigure}
    \hspace{1mm}
    \begin{subfigure}{.47\linewidth}
        \centering
        \begin{tikzpicture}
            \begin{axis}[
                label style={font=\scriptsize},
                tick label style={font=\scriptsize},
                width=1.2\linewidth,
                height=1.25\linewidth,
                xlabel=The number of clicks,
                ylabel=Recall (\%),
                xmin=0, xmax=12.2,
                ymin=-5, ymax=87,
                xtick={1, 3, 5, 7, 9, 11},
                xticklabel={$\pgfmathprintnumber{\tick}$K},
                ytick={0, 20, 40, 60, 80},
                xlabel style={yshift=0.15cm},
                ylabel style={yshift=-0.6cm},
            ]
            % CIL
            \addplot[cRed, very thick, mark=triangle*, mark size=2pt, mark options={solid}] table[col sep=comma, x=x, y=CIL]{Data/recall_pascal.csv};
            % SoftMin
            \addplot[orange, very thick, mark=diamond*, mark size=2pt, mark options={solid}] table[col sep=comma, x=x, y=SoftMin]{Data/recall_pascal.csv};
            % AIoU
            \addplot[cGreen, very thick, mark=square*, mark size=2pt, mark options={solid}] table[col sep=comma, x=x, y=AIoU]{Data/recall_pascal.csv};
            % SIM
            \addplot[cBlue, very thick, mark=pentagon*, mark size=2pt, mark options={solid}] table[col sep=comma, x=x, y=SIM]{Data/recall_pascal.csv};

            % CIL
            \addplot[name path=CIL-U, draw=none, fill=none] table[col sep=comma, x=x, y=CIL-U]{Data/recall_pascal.csv};
            \addplot[name path=CIL-D, draw=none, fill=none] table[col sep=comma, x=x, y=CIL-D]{Data/recall_pascal.csv};
            \addplot[cRed, fill opacity=0.3] fill between[of=CIL-U and CIL-D];
            % SoftMin
            \addplot[name path=SoftMin-U, draw=none, fill=none] table[col sep=comma, x=x, y=SoftMin-U]{Data/recall_pascal.csv};
            \addplot[name path=SoftMin-D, draw=none, fill=none] table[col sep=comma, x=x, y=SoftMin-D]{Data/recall_pascal.csv};
            \addplot[orange, fill opacity=0.3] fill between[of=SoftMin-U and SoftMin-D];
            % AIoU
            \addplot[name path=AIoU-U, draw=none, fill=none] table[col sep=comma, x=x, y=AIoU-U]{Data/recall_pascal.csv};
            \addplot[name path=AIoU-D, draw=none, fill=none] table[col sep=comma, x=x, y=AIoU-D]{Data/recall_pascal.csv};
            \addplot[cGreen, fill opacity=0.3] fill between[of=AIoU-U and AIoU-D];
            % SIM
            \addplot[name path=SIM-U, draw=none, fill=none] table[col sep=comma, x=x, y=SIM-U]{Data/recall_pascal.csv};
            \addplot[name path=SIM-D, draw=none, fill=none] table[col sep=comma, x=x, y=SIM-D]{Data/recall_pascal.csv};
            \addplot[cBlue, fill opacity=0.3] fill between[of=SIM-U and SIM-D];
            
            \end{axis}
        \end{tikzpicture}
        \caption{Recall}
        \label{fig:pascal-recall}
    \end{subfigure}
    % \hspace{-2mm}
    % \begin{subfigure}{.33\linewidth}
    %     \centering
    %     \begin{tikzpicture}
    %         \begin{axis}[
    %             label style={font=\scriptsize},
    %             tick label style={font=\scriptsize},
    %             width=1.05\linewidth,
    %             height=0.8\linewidth,
    %             xlabel=The number of clicks,
    %             ylabel=F1 score,
    %             xmin=0, xmax=12,
    %             ymin=-0.05, ymax=0.65,
    %             xtick={1, 3, 5, 7, 9, 11},
    %             xticklabel={$\pgfmathprintnumber{\tick}$k},
    %             ytick={0.1, 0.3, 0.5, 0.7},
    %             xlabel style={yshift=0.15cm},
    %             ylabel style={yshift=-0.4cm},
    %         ]
    %         % CIL
    %         \addplot[cRed, very thick, mark=triangle*, mark size=2pt, mark options={solid}] table[col sep=comma, x=x, y=CIL]{Data/sam_to_p_f1.csv};
    %         % SoftMin
    %         \addplot[orange, very thick, mark=diamond*, mark size=2pt, mark options={solid}] table[col sep=comma, x=x, y=SoftMin]{Data/sam_to_p_f1.csv};
    %         % AIoU
    %         \addplot[cGreen, very thick, mark=square*, mark size=2pt, mark options={solid}] table[col sep=comma, x=x, y=AIoU]{Data/sam_to_p_f1.csv};
    %         % Sim
    %         \addplot[cBlue, very thick, mark=pentagon*, mark size=2pt, mark options={solid}] table[col sep=comma, x=x, y=Sim]{Data/sam_to_p_f1.csv};
    %         \end{axis}
    %     \end{tikzpicture}
    %     \caption{F1 score}
    % \end{subfigure}
    \caption{{\em Precision and recall comparisons.} Our \sftype{SIM} acquisition shows a high recall, indicating it corrects many noisy pixels with limited budgets.}
    % 11K denotes 11,257 representing the total number of superpixels.
    \label{fig:pascal-precision-recall}
    % \vspace{-2mm}
\end{figure}

\begin{table}[t!]
% \caption{{\em Quality of corrected datasets.} For correction, we select 5K pixels from the initial labels with different acquisitions. We remark that acquisition functions are compared on an iteration of \textit{ALC}.} 
\caption{{\em Quality of corrected datasets.} The labels of 5K pixels from the initial datasets are corrected using different acquisition functions in the ALC framework.}
% \textit{ALC} scenario.
% We remark that acquisition functions are compared on an iteration of \textit{ALC}.} 
% We train a model with a label expansion technique.
% \js{remark that a set of acquisition functions are compared on an iteration of ALC, and ACL(SIM) is the proposed... }}
\begin{center}
\begin{small}
\setlength\tabcolsep{6pt}
\centering
% \vspace{-2.5mm}
\begin{tabular}{c|cc}
\toprule
Acquisition function & Data mIoU (\%) & Model mIoU (\%) \\ \midrule
LCIL & 56.59$_{\pm 0.07}$ & 56.82$_{\pm 0.05}$ \\
SoftMin & 59.28$_{\pm 0.59}$ & 58.66$_{\pm 0.89}$ \\ 
AIoU & 59.95$_{\pm 0.57}$ & 59.04$_{\pm 0.27}$ \\
SIM (ours) & \textbf{83.04$_{\pm 0.62}$} & \textbf{68.72$_{\pm 0.10}$} \\
\bottomrule
\end{tabular}
\end{small}
\end{center}
% \vspace{-2mm}
\label{tab:acq-pascal}
\end{table}

\noindent\textbf{Verification of Labeling Costs with User Study.}
\label{sec:user-study}
In Theorem~\ref{the:queries}, we prove that the labeling cost of the correction query $C_{\textnormal{cor}}$ is lower than the classification query $C_{\textnormal{cls}}$.
% Here, we verify its effectiveness in practice with user study of different queries.
In Table~\ref{tab:user-study-results}, we empirically show its effectiveness with a user study conducted by 20 annotators, where they are given 20 queries with $p = 0.5$ scenarios.
% Here, we empirically show its effectiveness with user study.
% We take 20 queries with p = 0.5 under 20 annotators, where halves involve noisy pixel labels requiring correction.
% i.e., 10 queries are related to noisy pixel labels, which require label correction by annotators. 
Theoretically, as $L = 20$ in PASCAL, the cost ratio between the two queries is about 0.62.
In Table~\ref{tab:user-study-results}, we observe that $C_{\textnormal{cor}}$ requires 0.75 times the cost of $C_{\textnormal{cls}}$, in practice.
More details about user study are in the Appendix~\ref{app:user-study}.
% 20 problems, precision 0.5, cls, 86.44, cor, 52.48, ratio, 0.61

% \subsection{Design Justification of Proposed Acquisition Function}
\subsection{Effectiveness of Proposed Acquisition Function}
\noindent\textbf{Baselines.}
In our ALC framework, we compare our \sftype{SIM} acquisition with previous ones for detecting noisy labels in segmentation datasets, such as 
\sftype{LCIL}, \sftype{SoftMin}~\cite{lad2023segmentation}, and \sftype{AIoU}~\cite{rottmann2023automated}.
% which compute acquisitions image-wisely and superpixel-wisely, respectively.
% For a fair comparison, we only change the acquisition function while fixing other techniques including a diversified pixel pool, the concept of look-ahead, and label expansion across all acquisitions.
For a fair comparison, we keep all other methodologies constant, including a diversified pixel pool, lookahead strategy, and label expansion, varying only the acquisition function.

% which conducts image-wise selection, and \textit{AIoU}~\cite{rottmann2023automated}, which introduce superpixel-wise selection.
% To ensure a fair evaluation, we apply other techniques including a diversified pixel pool, the concept of look-ahead, and label expansion across all acquisitions.

% Such as, 
% In the NLD setting, we compare our acquisition function called \textit{SIM} with 
% the acquisitions of previous state-of-the-art noisy label detection methods.
% To check if the proposed methods achieve a certain level of performance,
% In addition, \textit{95\%.} denotes the 95\% of fully supervised performance.

\noindent\textbf{Evaluation Protocol.}
Given a limited budget, we select unreliable pixels, correct their labels, and expand them to the corresponding superpixels.
We first evaluate the efficiency of the acquisition functions in terms of precision and recall at the pixel level.
Specifically, precision refers to the proportion of pixels correctly identified as mislabeled out of the selected pixels, while recall represents the fraction of pixels chosen correctly from the total number of mislabeled pixels.
Then, we access models trained with each corrected dataset from different acquisitions with mIoU.
The ablation experiments on acquisition are conducted in PASCAL.
% Assuming PASCAL consists of clean pixel labels, the following experiments are conducted in PASCAL.

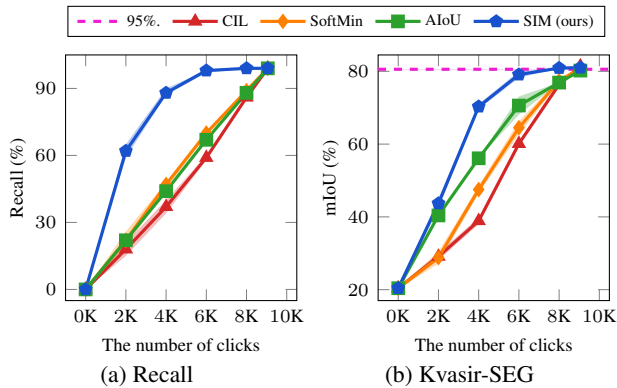
\begin{figure}[t!]
    \captionsetup[subfigure]{font=footnotesize,labelfont=footnotesize,aboveskip=0.05cm,belowskip=-0.15cm}
    \centering
    \hspace{-3mm}
    \begin{subfigure}{.47\linewidth}
        \centering
        \begin{tikzpicture}
            \begin{axis}[
                legend style={
                    nodes={scale=0.6}, 
                    at={(2.334, 1.2)}, 
                    /tikz/every even column/.append style={column sep=1mm},
                    legend columns=-1,
                },
                label style={font=\scriptsize},
                tick label style={font=\scriptsize},
                width=1.2\linewidth,
                height=1.25\linewidth,
                xlabel=The number of clicks,
                ylabel=Recall (\%),
                xmin=-1, xmax=10.5,
                ymin=-5, ymax=105,
                xtick={0, 2, 4, 6, 8, 10},
                xticklabel={$\pgfmathprintnumber{\tick}$K},
                ytick={0, 30, 60, 90},
                xlabel style={yshift=0.15cm},
                ylabel style={yshift=-0.6cm},
            ]
            % Full
            \addplot[cPink, very thick, mark size=2pt, mark options={solid}, dashed] coordinates {(0,-1)};
            % CIL
            \addplot[cRed, very thick, mark=triangle*, mark size=2pt, mark options={solid}] table[col sep=comma, x=x, y=CIL]{Data/recall_kvasir.csv};
            % SoftMin
            \addplot[orange, very thick, mark=diamond*, mark size=2pt, mark options={solid}] table[col sep=comma, x=x, y=SoftMin]{Data/recall_kvasir.csv};
            % AIoU
            \addplot[cGreen, very thick, mark=square*, mark size=2pt, mark options={solid}] table[col sep=comma, x=x, y=AIoU]{Data/recall_kvasir.csv};
            % SIM
            \addplot[cBlue, very thick, mark=pentagon*, mark size=2pt, mark options={solid}] table[col sep=comma, x=x, y=SIM]{Data/recall_kvasir.csv};

            % CIL
            \addplot[name path=CIL-U, draw=none, fill=none] table[col sep=comma, x=x, y=CIL-U]{Data/recall_kvasir.csv};
            \addplot[name path=CIL-D, draw=none, fill=none] table[col sep=comma, x=x, y=CIL-D]{Data/recall_kvasir.csv};
            \addplot[cRed, fill opacity=0.3] fill between[of=CIL-U and CIL-D];
            % SoftMin
            \addplot[name path=SoftMin-U, draw=none, fill=none] table[col sep=comma, x=x, y=SoftMin-U]{Data/recall_kvasir.csv};
            \addplot[name path=SoftMin-D, draw=none, fill=none] table[col sep=comma, x=x, y=SoftMin-D]{Data/recall_kvasir.csv};
            \addplot[orange, fill opacity=0.3] fill between[of=SoftMin-U and SoftMin-D];
            % AIoU
            \addplot[name path=AIoU-U, draw=none, fill=none] table[col sep=comma, x=x, y=AIoU-U]{Data/recall_kvasir.csv};
            \addplot[name path=AIoU-D, draw=none, fill=none] table[col sep=comma, x=x, y=AIoU-D]{Data/recall_kvasir.csv};
            \addplot[cGreen, fill opacity=0.3] fill between[of=AIoU-U and AIoU-D];
            % SIM
            \addplot[name path=SIM-U, draw=none, fill=none] table[col sep=comma, x=x, y=SIM-U]{Data/recall_kvasir.csv};
            \addplot[name path=SIM-D, draw=none, fill=none] table[col sep=comma, x=x, y=SIM-D]{Data/recall_kvasir.csv};
            \addplot[cBlue, fill opacity=0.3] fill between[of=SIM-U and SIM-D];
            
            \legend{95\%., CIL, SoftMin, AIoU, SIM (ours)}
            \end{axis} 
        \end{tikzpicture}
        \caption{Recall}
        \label{fig:kvasir-(a)}
    \end{subfigure}
    \hspace{1mm}
    \begin{subfigure}{.47\linewidth}
        \centering
        \begin{tikzpicture}
            \begin{axis}[
                xlabel={The number of clicks},
                ylabel={mIoU (\%)},
                width=1.2\linewidth,
                height=1.25\linewidth,
                ymin=17,
                ymax=84.5,
                xtick={0, 2, 4, 6, 8, 10},
                ytick={20, 40, 60, 80},
                xlabel style={yshift=0.15cm},
                ylabel style={yshift=-0.6cm},
                xmin=-1,
                xmax=10.5,
                label style={font=\scriptsize},
                tick label style={font=\scriptsize},
                xticklabel={$\pgfmathprintnumber{\tick}$K}
            ]
            % Full 95%
            % \addplot[cPink, very thick, mark size=2pt, mark options={solid}] table[col sep=comma, x=x, y=Full]{Data/acq_kvasir.csv};
            \draw [cPink, very thick, dashed] (axis cs:-1,80.56) -- (axis cs:11,80.56);
            % CIL
            \addplot[cRed, very thick, mark=triangle*, mark size=2pt, mark options={solid}] table[col sep=comma, x=x, y=CIL]{Data/acq_kvasir.csv};
            % SoftMin
            \addplot[orange, very thick, mark=diamond*, mark size=2pt, mark options={solid}] table[col sep=comma, x=x, y=SoftMin]{Data/acq_kvasir.csv};
            % AIoU
            \addplot[cGreen, very thick, mark=square*, mark size=2pt, mark options={solid}] table[col sep=comma, x=x, y=AIoU]{Data/acq_kvasir.csv};
            % SIM
            \addplot[cBlue, very thick, mark=pentagon*, mark size=2pt, mark options={solid}] table[col sep=comma, x=x, y=SIM]{Data/acq_kvasir.csv};

            % Full 95%
            \addplot[name path=Full-U, draw=none, fill=none] table[col sep=comma, x=x, y=Full-U]{Data/acq_kvasir.csv};
            \addplot[name path=Full-D, draw=none, fill=none] table[col sep=comma, x=x, y=Full-D]{Data/acq_kvasir.csv};
            \addplot[cPink, fill opacity=0.3] fill between[of=Full-U and Full-D];
            % CIL
            \addplot[name path=CIL-U, draw=none, fill=none] table[col sep=comma, x=x, y=CIL-U]{Data/acq_kvasir.csv};
            \addplot[name path=CIL-D, draw=none, fill=none] table[col sep=comma, x=x, y=CIL-D]{Data/acq_kvasir.csv};
            \addplot[cRed, fill opacity=0.3] fill between[of=CIL-U and CIL-D];
            % SoftMin
            \addplot[name path=SoftMin-U, draw=none, fill=none] table[col sep=comma, x=x, y=SoftMin-U]{Data/acq_kvasir.csv};
            \addplot[name path=SoftMin-D, draw=none, fill=none] table[col sep=comma, x=x, y=SoftMin-D]{Data/acq_kvasir.csv};
            \addplot[orange, fill opacity=0.3] fill between[of=SoftMin-U and SoftMin-D];
            % AIoU
            \addplot[name path=AIoU-U, draw=none, fill=none] table[col sep=comma, x=x, y=AIoU-U]{Data/acq_kvasir.csv};
            \addplot[name path=AIoU-D, draw=none, fill=none] table[col sep=comma, x=x, y=AIoU-D]{Data/acq_kvasir.csv};
            \addplot[cGreen, fill opacity=0.3] fill between[of=AIoU-U and AIoU-D];
            % SIM
            \addplot[name path=SIM-U, draw=none, fill=none] table[col sep=comma, x=x, y=SIM-U]{Data/acq_kvasir.csv};
            \addplot[name path=SIM-D, draw=none, fill=none] table[col sep=comma, x=x, y=SIM-D]{Data/acq_kvasir.csv};
            \addplot[cBlue, fill opacity=0.3] fill between[of=SIM-U and SIM-D];

            \end{axis}
        \end{tikzpicture}
        \caption{Kvasir-SEG}
        \label{fig:kvasir-(b)}
    \end{subfigure}
    \caption{{\em Kvasir-SEG experiments.} The proposed \sftype{SIM} acquisition operate robustly on medical dataset across different budgets.}
    \label{fig:kvasir}
    \vspace{-3mm}
\end{figure}

\noindent\textbf{Precision and Recall of Acquisition Functions.}
\iffalse
We first obtain a model by warm-starting with the initial pseudo-labels from Grounded-SAM~\cite{liu2023grounding}, as depicted in Section~\ref{sec:warm-start}.
Then, we choose pixels in order of acquisition values from the model, within our budget limits.
\fi
In Figure~\ref{fig:pascal-precision-recall}, we compare \sftype{SIM} to baseline acquisitions by calculating the precision and recall for detecting incorrect pixels.
% the effectiveness of various acquisitions is evaluated by measuring the precision and recall of the corrected pixels, including expansion. 
% Specifically, precision refers to the proportion of pixels correctly identified as mislabeled out of the selected pixels, while recall represents the fraction of pixels chosen correctly from the total number of mislabeled pixels.
\sftype{SIM} outperforms the baseline acquisition functions in terms of recall while showing a comparably low precision rate.
% This is due to the fact that SIM takes into account the effect of label expansion as in (\ref{eq:sim}), favoring large superpixels. 
This is attributed to \sftype{SIM} considering the effect of label expansion as in (\ref{eq:sim}), which favors large superpixels.
We consider this design choice to be effective for ALC for two reasons.
% First, the labeling cost of false positive selection, selecting correct labels, is considerably low as shown in the Theorem~\ref{the:queries}.
First, as demonstrated in Theorem~\ref{the:queries}, the labeling cost of reconfirming false positives, i.e., correct pseudo labels, is significantly low.
Second, correcting the labels of as many pixels as possible, which is related to high recall, leads to greater improvements in data and model performance as shown in Table~\ref{tab:acq-pascal}.
% ---achieving high recall---l
% leads to greater improvements in model performance as shown in Table~\ref{tab:acq-pascal}.
% Note that the upper bound of the recall is 82\% since the style of the segmentation 
% takes into account the effect of label expansion
% Compared to \textit{LCIL} in~\eqref{eq:lcil} calculating the average pixel value of \textit{CIL} in~\eqref{eq:cil}, our \textit{SIM} demonstrates lower precision but higher recall.
\iffalse
When we select all 11.2K pixels from a diversified pixel pool, every acquisition function achieves the same precision, 49\%, and recall, 82\%.
The reason for low recall is related to the style of segmentation employed.
For the bicycle class, labeling guidelines may differ, allowing either the entire inside or only the frame to be labeled, as illustrated in Figure X.
\fi

\noindent\textbf{Quality of Corrected Datasets.}
In Table~\ref{tab:acq-pascal}, we compare \sftype{SIM} to baseline acquisition functions in terms of the quality of the corrected dataset when using 5K clicks.
% The quality of the corrected dataset is evaluated by mIoU between the corrected label and Ground Truth label (Data mIoU), and the performance of a segmentation model trained with the corrected labels (Model mIoU).
The quality of the corrected dataset is evaluated by the accuracy of corrected labels (Data mIoU), and the performance of a model trained with them (Model mIoU).
For both metrics, the dataset corrected by \sftype{SIM} shows the best quality.
This shows that the performance of the model is more correlated to the recall in Figure~\ref{fig:pascal-recall}, as high recall indicates fewer incorrect labels in the dataset.
% 
\iffalse
After selecting pixels with acquisitions, we correct their labels and expand them to their associated superpixels.
The results show that \textit{SIM} improves the quality of the training dataset the most.
This improvement is indicated by a high data mIoU, which leads to better performance on the test dataset, reflected in a higher model mIoU.
Recalling that \textit{SIM} gets a high recall in Figure.~\ref{fig:pascal-recall}, it is worth noting that high recall is more crucial than high precision, as high recall indicates fewer incorrect labels in the dataset.
\fi

\subsection{Further Analyses}
\noindent\textbf{Applicability to Medical Dataset.}
% \noindent\textbf{Generalization to Medical Dataset.}
% To verify the applicability of our proposed framework in different domains,
% our proposed framework in different domains, 
% we experiment with the Kvasir-SEG (Kvasir) medical dataset, which includes images of gastrointestinal polyps along with corresponding segmentation masks.
In Figure~\ref{fig:kvasir}, we apply the ALC framework to the Kvasir-SEG dataset to verify the generalization ability of our framework to challenging medical domain.
Here, the initial dataset shows 20\% mIoU, as shown in 0K of Figure~\ref{fig:kvasir-(b)}.
Even under such challenging initial conditions, the ALC combined with \sftype{SIM} reaches 93\% performance of the fully supervised model only using 6K clicks.
This performance can be attributed to \sftype{SIM} acquisition function, which consistently achieves the highest recall among baselines over various numbers of clicks, as shown in Figure~\ref{fig:kvasir-(a)}.
% "Even under such challenging initial conditions, the ALC combined with \textit{SIM} reaches 93\% of the fully supervised model's performance /with just 6K clicks.
% This performance can be attributed to the \textit{SIM} acquisition function, which consistently achieves the highest recall among baselines across various click counts, as demonstrated in Figure~\ref{fig:kvasir-(a)}."
% Notably, with 4K clicks, \textit{SIM}'s performance significantly stands out: its recall rate is 93\% higher, and its mIoU is 25\% higher than the second-best acquisition.
% Moreover, in Figure~\ref{fig:kvasir-(b)}, we attain a comparable mIoU with 6K clicks, in contrast to about 9K clicks required by other methods.
% 
% 
% In the medical domain, foundation models often fail, as seen in Figure~\ref{fig:kvasir-(b)} with 0K, i.e., the mIoU of only 20\%. 
% 
% The experiment in Figure~\ref{fig:kvasir} only varies in acquisition functions while keeping other techniques constant.
% Note that we are the first approach for applying superpixel-wise sampling approaches in Kvair-SEG, departed from conventional image-wise sampling approaches~\cite{smailagic2018medal,wu2021hal}.
We note that our approach introduces superpixel-wise sampling to Kvasir-SEG for the first time, diverging from the traditional image-wise sampling methods~\cite{smailagic2018medal,wu2021hal}.
% since we are the first approach to apply superpixel-based active learning to Kav
% 
% Since previous methods~\cite{smailagic2018medal,wu2021hal} in the field of active learning for medical domains employ image-wise selection techniques, it is challenging to directly compare with our pixel-wise acquisition.
% Therefore, our comparison mainly focuses on the acquisitions used for detecting noisy labels in segmentation tasks.
% As shown in Figure~\ref{fig:kvasir-(a)}, our \textit{SIM} achieves the highest recall among different acquisitions over various numbers of clicks.
% Notably, with 4K clicks, \textit{SIM}'s performance significantly stands out: its recall rate is 93\% higher, and its mIoU is 25\% higher than the second-best acquisition.
% Moreover, in Figure~\ref{fig:kvasir-(b)}, we attain a comparable mIoU with 6K clicks, in contrast to about 9K clicks required by other methods.
%, representing a budget saving of 33\%.

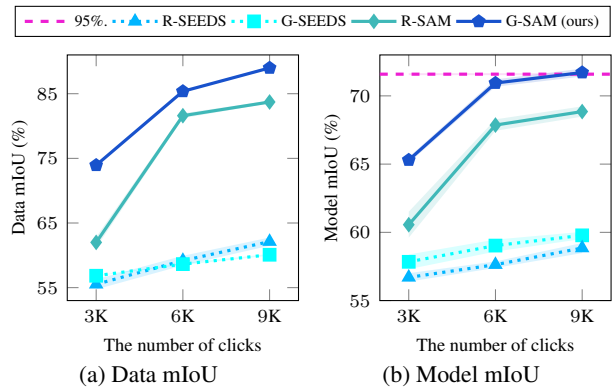
\begin{figure}[t!]
    \captionsetup[subfigure]{font=footnotesize,labelfont=footnotesize,aboveskip=0.05cm,belowskip=-0.15cm}
    \centering
    \hspace{-3mm}
    \begin{subfigure}{.47\linewidth}
        \centering
        \begin{tikzpicture}
            \begin{axis}[
                legend style={nodes={scale=0.6}, at={(2.35, 1.2)}},
                legend columns=-1,
                label style={font=\scriptsize},
                tick label style={font=\scriptsize},
                width=1.2\linewidth,
                height=1.25\linewidth,
                xlabel=The number of clicks,
                ylabel=Data mIoU (\%),
                xmin=2, xmax=10,
                ymin=53, ymax=91,
                xtick={3, 6, 9},
                xticklabel={$\pgfmathprintnumber{\tick}$K},
                ytick={55, 65, 75, 85},
                xlabel style={yshift=0.15cm},
                ylabel style={yshift=-0.6cm},
            ]
            % Full
            \addplot[cPink, very thick, mark size=2pt, dashed, mark options={solid}] coordinates {(0,-1)};
            % Random-SEEDS
            \addplot[cBlue4, very thick, dotted, mark=triangle*, mark size=2pt, mark options={solid}] table[col sep=comma, x=x, y=RSEEDS]{Data/foundation_data_mIoU.csv};
            % Grounded-SEEDS
            \addplot[cBlue3, very thick, dotted, mark=square*, mark size=2pt, mark options={solid}] table[col sep=comma, x=x, y=GSEEDS]{Data/foundation_data_mIoU.csv};
            % Random-SAM
            \addplot[cBlue2, very thick, mark=diamond*, mark size=2pt, mark options={solid}] table[col sep=comma, x=x, y=RSAM]{Data/foundation_data_mIoU.csv};
            % Grounded SAM (ours)
            \addplot[cBlue, very thick, mark=pentagon*, mark size=2pt, mark options={solid}] table[col sep=comma, x=x, y=GSAM]{Data/foundation_data_mIoU.csv};
            
            % Grounded SAM (ours)
            \addplot[name path=GSAM-U, draw=none, fill=none] table[col sep=comma, x=x, y=GSAM-U]{Data/foundation_data_mIoU.csv};
            \addplot[name path=GSAM-D, draw=none, fill=none] table[col sep=comma, x=x, y=GSAM-D]{Data/foundation_data_mIoU.csv};
            \addplot[cBlue, fill opacity=0.15] fill between[of=GSAM-U and GSAM-D];
            % Random-SAM
            \addplot[name path=RSAM-U, draw=none, fill=none] table[col sep=comma, x=x, y=RSAM-U]{Data/foundation_data_mIoU.csv};
            \addplot[name path=RSAM-D, draw=none, fill=none] table[col sep=comma, x=x, y=RSAM-D]{Data/foundation_data_mIoU.csv};
            \addplot[cBlue2, fill opacity=0.15] fill between[of=RSAM-U and RSAM-D];
            % Grounded-SEEDS
            \addplot[name path=GSEEDS-U, draw=none, fill=none] table[col sep=comma, x=x, y=GSEEDS-U]{Data/foundation_data_mIoU.csv};
            \addplot[name path=GSEEDS-D, draw=none, fill=none] table[col sep=comma, x=x, y=GSEEDS-D]{Data/foundation_data_mIoU.csv};
            \addplot[cBlue3, fill opacity=0.15] fill between[of=GSEEDS-U and GSEEDS-D];
            % Random-SEEDS
            \addplot[name path=RSEEDS-U, draw=none, fill=none] table[col sep=comma, x=x, y=RSEEDS-U]{Data/foundation_data_mIoU.csv};
            \addplot[name path=RSEEDS-D, draw=none, fill=none] table[col sep=comma, x=x, y=RSEEDS-D]{Data/foundation_data_mIoU.csv};
            \addplot[cBlue4, fill opacity=0.15] fill between[of=RSEEDS-U and RSEEDS-D];

            \legend{95\%., R-SEEDS, G-SEEDS, R-SAM, G-SAM (ours)}
            \end{axis}
        \end{tikzpicture}
        \caption{Data mIoU}
    \end{subfigure}
    \hspace{1mm}
    \begin{subfigure}{.47\linewidth}
        \centering
        \begin{tikzpicture}
            \begin{axis}[
                legend style={nodes={scale=0.6}, at={(2.05, 1.2)}},
                legend columns=4,
                xlabel={The number of clicks},
                ylabel={Model mIoU (\%)},
                width=1.2\linewidth,
                height=1.25\linewidth,
                ymin=55,
                ymax=73,
                xtick={3, 6, 9},
                ytick={55, 60, 65, 70},
                xlabel style={yshift=0.15cm},
                ylabel style={yshift=-0.6cm},
                xmin=2,
                xmax=10,
                label style={font=\scriptsize},
                tick label style={font=\scriptsize},
                xticklabel={$\pgfmathprintnumber{\tick}$K}
            ]
            % Full
            % \addplot[cPink, very thick, mark size=2pt, mark options={solid}] table[col sep=comma, x=x, y=Full]{Data/foundation_model_mIoU.csv};
            \draw [cPink, very thick, dashed] (axis cs:2,71.59) -- (axis cs:10,71.59);
            % Grounded SAM (ours)
            \addplot[cBlue, very thick, mark=pentagon*, mark size=2pt, mark options={solid}] table[col sep=comma, x=x, y=GSAM]{Data/foundation_model_mIoU.csv};
            % Random-SAM
            \addplot[cBlue2, very thick, mark=diamond*, mark size=2pt, mark options={solid}] table[col sep=comma, x=x, y=RSAM]{Data/foundation_model_mIoU.csv};
            % Grounded-SEEDS
            \addplot[cBlue3, very thick, dotted, mark=square*, mark size=2pt, mark options={solid}] table[col sep=comma, x=x, y=GSEEDS]{Data/foundation_model_mIoU.csv};
            % Random-SEEDS
            \addplot[cBlue4, very thick, dotted, mark=triangle*, mark size=2pt, mark options={solid}] table[col sep=comma, x=x, y=RSEEDS]{Data/foundation_model_mIoU.csv};

            % Full
            \addplot[name path=Full-U, draw=none, dashed, fill=none] table[col sep=comma, x=x, y=Full-U]{Data/foundation_model_mIoU.csv};
            \addplot[name path=Full-D, draw=none, fill=none] table[col sep=comma, x=x, y=Full-D]{Data/foundation_model_mIoU.csv};
            \addplot[cPink, fill opacity=0.15] fill between[of=Full-U and Full-D];
            % Grounded SAM (ours)
            \addplot[name path=GSAM-U, draw=none, fill=none] table[col sep=comma, x=x, y=GSAM-U]{Data/foundation_model_mIoU.csv};
            \addplot[name path=GSAM-D, draw=none, fill=none] table[col sep=comma, x=x, y=GSAM-D]{Data/foundation_model_mIoU.csv};
            \addplot[cBlue, fill opacity=0.15] fill between[of=GSAM-U and GSAM-D];
            % Random-SAM
            \addplot[name path=RSAM-U, draw=none, fill=none] table[col sep=comma, x=x, y=RSAM-U]{Data/foundation_model_mIoU.csv};
            \addplot[name path=RSAM-D, draw=none, fill=none] table[col sep=comma, x=x, y=RSAM-D]{Data/foundation_model_mIoU.csv};
            \addplot[cBlue2, fill opacity=0.15] fill between[of=RSAM-U and RSAM-D];
            % Grounded-SEEDS
            \addplot[name path=GSEEDS-U, draw=none, fill=none] table[col sep=comma, x=x, y=GSEEDS-U]{Data/foundation_model_mIoU.csv};
            \addplot[name path=GSEEDS-D, draw=none, fill=none] table[col sep=comma, x=x, y=GSEEDS-D]{Data/foundation_model_mIoU.csv};
            \addplot[cBlue3, fill opacity=0.15] fill between[of=GSEEDS-U and GSEEDS-D];
            % Random-SEEDS
            \addplot[name path=RSEEDS-U, draw=none, fill=none] table[col sep=comma, x=x, y=RSEEDS-U]{Data/foundation_model_mIoU.csv};
            \addplot[name path=RSEEDS-D, draw=none, fill=none] table[col sep=comma, x=x, y=RSEEDS-D]{Data/foundation_model_mIoU.csv};
            \addplot[cBlue4, fill opacity=0.15] fill between[of=RSEEDS-U and RSEEDS-D];
            
            \end{axis}
        \end{tikzpicture}
        \caption{Model mIoU}
    \end{subfigure}
    \caption{{\em Advantages of foundation models.} Our \sftype{ALC} is called \sftype{G-SAM}, as it depends on Grounded-SAM. The effect of superpixels is larger than that of initial pseudo-labels.}
    \label{fig:ablation-foundation}
    \vspace{-3mm}
\end{figure}

\noindent\textbf{Decomposing the Advantages of Foundation Models.}
In Figure~\ref{fig:ablation-foundation}, we analyze the effect of the foundation model for \sftype{ALC} in two aspects: initial pseudo-labels and superpixels, in PASCAL.
We denote the proposed method using both aspects as \sftype{G-SAM}.
% As a baseline, we warm-up train a model with randomly sampled 3K
% As a baseline, we warm-up train a model with 3K budgets using random sampling, and utilize this model for the initial pseudo-label predictor.
% As a baseline, we train a model with 3K budgets using random sampling in the initial round and utilize this model for the pseudo-label generator in the consecutive rounds.
For the baseline, we initially train a model with a 3K budget through random sampling and then employ this model as the pseudo-label generator in subsequent rounds, which is denoted as \sftype{R-SAM}.
We note that the distinction between \sftype{G-SAM} and \sftype{R-SAM} lies in the method of obtaining initial pseudo-labels, rather than in the acquisition itself.
In subsequent rounds, namely for budgets of 6K and 9K, both \sftype{R-SAM} and \sftype{G-SAM} adhere to the same experimental settings, including the same \sftype{SIM} acquisition function.
Another baseline is to use superpixels from SEEDS~\cite{van2012seeds} instead of the ones from SAM, which is denoted as \sftype{G-SEEDS}.
% We denote \textit{R-SEEDS} as a baseline combining both initial pseudo labels from the model trained with random sampling and superpixels from SEEDS.
We denote \sftype{R-SEEDS} as a baseline combining both random sampling in the initial round and superpixels from SEEDS.
% Given different pseudo-labels and superpixels, we proceed with an active label correction using different budgets.
As shown in Figure~\ref{fig:ablation-foundation}, both aspects improve both Data mIoU and Model mIoU.
In particular, utilizing the superpixels from SAM shows significant performance improvement.
% The results demonstrate that both components are effective, but in particular, the effect of the accurate superpixels is greater.
% In the view of superpixels, we adopt semantic considered superpixels from foundation models, however, there exists lots of conventional superpixel algorithms including SEEDS~\cite{van2012seeds}, which generates superpixels via energy-driven sampling.
% We obtain the initial pseudo-label from foundation models, however, the output of the model trained with labels of randomly selected pixels in the initial phase, similar to conventional active learning methods, can also be used as pseudo-labels.
% For a fair comparison, pixels are selected from the proposed diversified pixel pool, and the model is trained with label expansion.
% In the view of superpixels, we adopt semantic considered superpixels from foundation models, however, there exists lots of conventional superpixel algorithms including SEEDS~\cite{van2012seeds}, which generates superpixels via energy-driven sampling.
% In Figure.~\ref{fig:ablation-foundation}, we compare our \textit{ALC} referred to as \textit{P-SAM}, where \textit{P} denotes initial pseudo-labels and \textit{SAM} implies superpixels, with \textit{R-SAM} and \textit{R-SEEDS}, where \textit{R} means starting with random labels and \textit{SEEDS} superpixels for active label correction.
% The results demonstrate that both components are effective, but in particular, the effect of the correct superpixel is greater.

\noindent\textbf{Synergy of Proposed Components.}
% \noindent\textbf{Contribution of Each Component.}
Table~\ref{tab:each-component} quantifies the contribution of each component in our method: (1) the diversified pixel pool (Diversity) in Section~\ref{sec:diversified-pixel-pool}, (2) the look-ahead acquisition (Look-ahead), and (3) the label expansion technique (Expansion) in Section.~\ref{sec:look-ahead}.
The ablation study is conducted by correcting the initial dataset using 5K budgets in PASCAL, and evaluated with both the accuracy of corrected labels (Data mIoU) and the performance of a model trained with them (Model mIoU).
The results show that all components improve both Data mIoU and Model mIoU.
In particular, the synergy of proposed components is pronounced.
Since correcting numerous pixels across various regions simultaneously is significant, omitting even one component results in significant performance degradation.
% While each element alone offers a modest enhancement in performance, 
% 
% In particular, the performance improvement is significant when both Look-ahead and Expansion are combined, supporting their design rationale.
% In particular, the combined effect of a look-ahead acquisition and label expansion is especially significant.
% 
% 
% To fully enjoy pixel-wise correction, we introduce three key components: (1) a diversified pixel pool, (2) a look-ahead acquisition, and (3) a label expansion technique.
% To evaluate the impact of each component, we gradually incorporate them into our \textit{ALC} framework and observe the outcomes.
% Table~\ref{tab:each-component} represents that while each element alone offers a modest enhancement in performance, the combined effect of a look-ahead acquisition and label expansion is especially significant.
% We propose a diversified pixel pool and look-ahead concept for acquisition and a label expansion technique for learning. 
% To evaluate the impact of each component, we incrementally add them to our framework and report the changes in performance.

\begin{table}[t!]
\caption{{\em Synergy of proposed components.}
% Contribution of each component
We conduct an ablation study, when correcting the initial dataset using 5K budgets in PASCAL.}
% With 5K budgets, we correct the initial labels to PASCAL by masking each.}
% \caption{{\em Contribution of each component.} With 5K budgets, we correct the initial labels to PASCAL by masking each.}
\label{tab:each-component}
\begin{center}
\begin{small}
\setlength\tabcolsep{3pt}
\centering
\begin{tabular}{ccc|cc}
\toprule
\multicolumn{2}{c}{Acquisition} & \multirow{2}{*}{Expansion} & \multirow{2}{*}{Data mIoU} & \multirow{2}{*}{Model mIoU} \\
\multicolumn{1}{c}{Diversity} & Look-ahead & & & \\ \midrule
\xmark & \xmark & \xmark & 55.03$_{\pm 0.25}$ & 56.30$_{\pm 0.56}$ \\ 
% \cmark & \xmark & \xmark & 55.60$_{\pm 0.00}$ & 56.61$_{\pm 0.32}$ \\
\xmark & \cmark & \cmark & 55.38$_{\pm 0.08}$ & 56.01$_{\pm 0.58}$ \\ 
\cmark & \xmark & \cmark & 56.59$_{\pm 0.07}$ & 56.82$_{\pm 0.05}$ \\
\cmark & \cmark & \xmark & 55.61$_{\pm 0.00}$ & 56.69$_{\pm 0.35}$ \\ 
\cmark & \cmark & \cmark & \textbf{83.04$_{\pm 0.62}$} & \textbf{68.72$_{\pm 0.10}$} \\ 
\bottomrule
\end{tabular}
\end{small}
\end{center}
\vspace{-5mm}
\end{table}

\noindent\paragraph{Fair Comparison with Baselines.}
We provide additional experiments and discussions to clarify the advantages of our method called \sftype{ALC}, compared to adopting Grounded-SAM (G-SAM) to \sftype{Spx} baseline. 
In turn, only our method fully leverages G-SAM mainly thanks to our acquisition function, SIM.
Table~\ref{tab:fair-spx} presents an ablation study on the advantages of G-SAM, which are two-fold: warm-start with initial pseudo-labels and SAM superpixels. 
The gap between the first and second rows quantifies the advantage of warm-start with G-SAM when using \sftype{Spx}. 
This is not substantial since the pseudo labels from G-SAM contain considerable noises, as shown in Figure~\ref{fig:Grounded SAM-images}, i.e., Data mIoU 55.32\% in PASCAL. 
In addition, comparing the second and third rows, the advantage of using SAM superpixels for \sftype{Spx} is negligible. 
The gain of our method in the fourth row is clear. 
This is mainly thanks to the proposed acquisition function, SIM, with the look-ahead ability. 
We note that \sftype{MerSpx}~\cite{Kim_2023_ICCV} based on ClassBal of \sftype{Spx} has no such look-ahead.
\sftype{MulSpx}~\cite{hwang2023active} proposes a multi-class query, which requests labeling all classes within a superpixel, making it difficult to conduct a fair comparison. 
% Lastly, the comparison in terms of mIoU is available in Appendix B.5 and Link 1 although it is hard to be fair due to their differences in model architectures and hyperparameters.

\iffalse
\textbf{Experimental setup of Spx in Table~\ref{tab:fair-spx}.}
To proceed with cold-start, we train a model with the dominant labels of randomly selected 1.5K superpixels in the initial round. Here, dominant labels indicate the majority pixel labels within superpixels. In the subsequent round, we train a model with additional dominant labels of 1.5K superpixels chosen by ClassBal acquisition in \sftype{Spx}. 
For the warm-start, we initiate training with initial pseudo-labels obtained from G-SAM. Then, we train a model with dominant labels of 3K superpixels selected by ClassBal at once.
\fi

\begin{figure*}[t!]
% average line
\captionsetup[subfigure]{font=footnotesize,labelfont=footnotesize,aboveskip=0.05cm,belowskip=-0.15cm}
\centering
\begin{subfigure}{0.49\textwidth}
\begin{tikzpicture}
    \begin{axis}[
        label style={font=\scriptsize},
        tick label style={font=\scriptsize},
        xticklabel style={anchor=center, yshift=-4.5mm, font=\scriptsize},
        symbolic x coords={
            \rotatebox{60}{pottedplant},
            \rotatebox{60}{sofa},
            \rotatebox{60}{chair},
            \rotatebox{60}{diningtable},
            \rotatebox{60}{sheep},
            \rotatebox{60}{bicycle},
            \rotatebox{60}{bird},
            \rotatebox{60}{car},
            \rotatebox{60}{horse},
            \rotatebox{60}{tvmonitor},
            \rotatebox{60}{background},
            \rotatebox{60}{boat},
            \rotatebox{60}{cat},
            \rotatebox{60}{bus},
            \rotatebox{60}{bottle},
            \rotatebox{60}{motorbike},
            \rotatebox{60}{dog},
            \rotatebox{60}{aeroplane},
            \rotatebox{60}{train},
            \rotatebox{60}{person},
            \rotatebox{60}{cow},
        },
        axis y line*=left,
        axis x line=bottom,
        width=1.05\textwidth,
        height=0.5\textwidth,
        major x tick style = transparent,
        ybar=3*\pgflinewidth,
        bar width=4pt,
        ymajorgrids=true,
        ylabel={IoU gain (\%)},
        xtick=data,
        scaled y ticks=false,
        enlarge x limits=0.03,
        axis line style={-},
        ymin=-1,
        ymax=4.0,
        xlabel style={yshift=0.4cm},
        ylabel style={yshift=-0.6cm},
    ]
        \addplot[style={cBlue,fill=cBlue,mark=none}] coordinates {
            (\rotatebox{60}{pottedplant}, 3.72306)
            (\rotatebox{60}{sofa}, 2.36594)
            (\rotatebox{60}{chair}, 1.53988)
            (\rotatebox{60}{diningtable}, 0.85645)
            (\rotatebox{60}{sheep}, 0.770895)
            (\rotatebox{60}{bicycle}, 0.733825)
            (\rotatebox{60}{bird}, 0.486811)
            (\rotatebox{60}{car}, 0.459622)
            (\rotatebox{60}{horse}, 0.43175)
            (\rotatebox{60}{tvmonitor}, 0.391015)
            (\rotatebox{60}{background}, 0.042401)
            (\rotatebox{60}{boat}, 0.017323)
            (\rotatebox{60}{cat}, -0.01506)
            (\rotatebox{60}{bus}, -0.0704)
            (\rotatebox{60}{bottle}, -0.13237)
            (\rotatebox{60}{motorbike}, -0.24271)
            (\rotatebox{60}{dog}, -0.24545)
            (\rotatebox{60}{aeroplane}, -0.29385)
            (\rotatebox{60}{train}, -0.31687)
            (\rotatebox{60}{person}, -0.38381)
            (\rotatebox{60}{cow}, -0.79704)
        };
        \draw [orange, dashed, very thick] (rel axis cs:0,0.288) -- (rel axis cs:1,0.288);
    \end{axis}
\end{tikzpicture}
\caption{Class-wise IoU gain}
\label{fig:stats-(a)}
\end{subfigure}
\begin{subfigure}{0.49\textwidth}
\begin{tikzpicture}
    \begin{axis}[
        label style={font=\scriptsize},
        tick label style={font=\scriptsize},
        xticklabel style={anchor=center, yshift=-4.5mm, font=\scriptsize},
        symbolic x coords={
            \rotatebox{60}{background},
            \rotatebox{60}{person},
            \rotatebox{60}{sofa},
            \rotatebox{60}{chair},
            \rotatebox{60}{pottedplant},
            \rotatebox{60}{diningtable},
            \rotatebox{60}{tvmonitor},
            \rotatebox{60}{car},
            \rotatebox{60}{bottle},
            \rotatebox{60}{sheep},
            \rotatebox{60}{train},
            \rotatebox{60}{motorbike},
            \rotatebox{60}{horse},
            \rotatebox{60}{aeroplane},
            \rotatebox{60}{cow},
            \rotatebox{60}{bird},
            \rotatebox{60}{bus},
            \rotatebox{60}{boat},
            \rotatebox{60}{cat},
            \rotatebox{60}{bicycle},
            \rotatebox{60}{dog},
        },
        axis y line*=left,
        axis x line=bottom,
        width=1.05\textwidth,
        height=0.5\textwidth,
        major x tick style = transparent,
        ybar=3*\pgflinewidth,
        bar width=4pt,
        ymajorgrids=true,
        ylabel={\# of relabeled class},
        xtick=data,
        scaled y ticks=false,
        enlarge x limits=0.03,
        axis line style={-},
        ymin=-1,
        ymax=220,
        xlabel style={yshift=0.4cm},
        ylabel style={yshift=-0.5cm},
    ]
        \addplot[style={cBlue,fill=cBlue,mark=none}] coordinates {
            (\rotatebox{60}{background}, 214)
            (\rotatebox{60}{person}, 131)
            (\rotatebox{60}{sofa}, 75)
            (\rotatebox{60}{chair}, 62)
            (\rotatebox{60}{pottedplant}, 57)
            (\rotatebox{60}{diningtable}, 42)
            (\rotatebox{60}{tvmonitor}, 39)
            (\rotatebox{60}{car}, 38)
            (\rotatebox{60}{bottle}, 26)
            (\rotatebox{60}{sheep}, 20)
            (\rotatebox{60}{train}, 19)
            (\rotatebox{60}{motorbike}, 7)
            (\rotatebox{60}{horse}, 6)
            (\rotatebox{60}{aeroplane}, 5)
            (\rotatebox{60}{cow}, 5)
            (\rotatebox{60}{bird}, 2)
            (\rotatebox{60}{bus}, 2)
            (\rotatebox{60}{boat}, 1)
            (\rotatebox{60}{cat}, 1)
            (\rotatebox{60}{bicycle}, 0)
            (\rotatebox{60}{dog}, 0)
        };
    \end{axis}
\end{tikzpicture}
\caption{Distribution of corrected class}
\label{fig:stats-(b)}
\end{subfigure}
\caption{{\em PASCAL+ statistics.} (a) The IoU gain is calculated by averaging the improvements in the train and valid datasets in Table~\ref{tab:effect-of-pp}. The orange line ({\textcolor{orange}{-\;\!-\;\!-}}) denotes the average gain. (b) 
% Similar to a long-tail distribution, 
Certain classes are corrected a lot. Notably, the pottedplant, sofa, chair, and diningtable classes get many corrections, leading to a noticeable increase in IoU gain.}
\label{fig:class-iou}
\end{figure*}
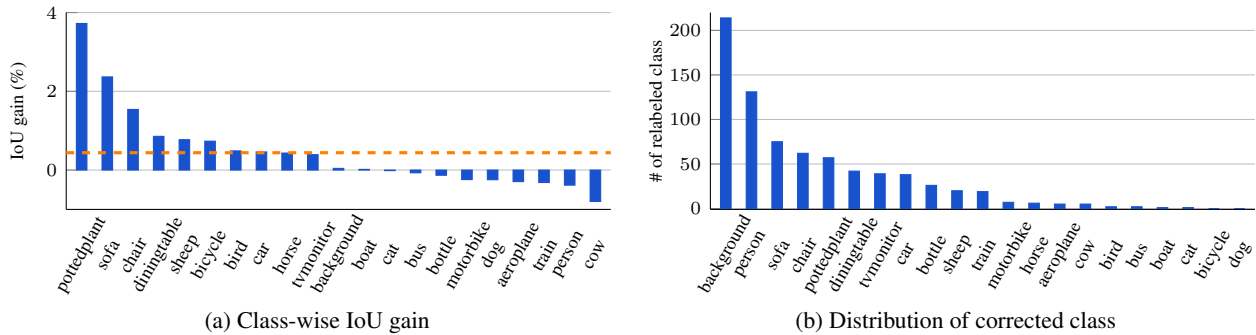

\begin{table}
\caption{{\em Fair comparison between Spx and ALC.} For a fair comparison, we integrate two advantages of foundation models into Spx. We refine the initial dataset using 3K budgets in PASCAL.}
\vspace{2.5mm}
\label{tab:fair-spx}
\centering
\begin{center}
\begin{small}
\setlength\tabcolsep{6pt}
\begin{tabular}{c|ccc}
\toprule
Methods & Initial stage & Superpixels & Model mIoU (\%) \\ \midrule
Spx & Cold-start & SEEDS & 52.34$_{\pm 0.85}$ \\ 
Spx & Warm-start & SEEDS & 57.77$_{\pm 0.70}$ \\ 
Spx & Warm-start & SAM & 57.79$_{\pm 0.66}$ \\ 
ALC & Warm-start & SAM & \textbf{65.30$_{\pm 0.21}$} \\ 
\bottomrule
\end{tabular}
\end{small}
\end{center}
\vspace{-5mm}
\end{table}

%% file: Sections/5_analyses.tex
\section{PASCAL+ corrected from PASCAL}
% We validate the practicality of the proposed framework, called active label correction framework, by constructing the PASCAL+ dataset (Section~\ref{sec:construct}), which is corrected from the PASCAL dataset~\cite{pascal-voc-2012}.
To demonstrate the practicality of the proposed framework, we apply corrections to the widely-used PASCAL dataset~\cite{pascal-voc-2012}, resulting in an enhanced version named PASCAL+ dataset (Section~\ref{sec:construct}).
%To demonstrate the practical application of our proposed framework, we applied corrections to the widely-used PASCAL dataset~\cite{pascal-voc-2012}, resulting in an enhanced version we call PASCAL+ (detailed in Section~\ref{sec:construct}).
Figures~\ref{fig:org-labels} and \ref{fig:pascal+} illustrate the change in labels between PASCAL and PASCAL+ datasets, respectively.
We demonstrate the enhanced model performance when using PASCAL+ compared to PASCAL and verify the cost-effectiveness of our \sftype{SIM} acquisition function (Section~\ref{sec:exp-pp}).

% To carry out the process of active label correction, two key components are necessary: an initial dataset requiring label correction and a set of predefined superpixels.
% The PASCAL dataset serves as an initial dataset, while we utilize detected objects from Grounded-SAM for the superpixels.

\subsection{Construction Process}
\label{sec:construct}
We apply our active label correction to construct the refined version of the PASCAL dataset.
We first generate 81K superpixels using Grounded-SAM, where we use 0.1 as the box threshold.
% We set the box threshold of Grounded-SAM to 0.1 to capture objects, yielding about 81K superpixels.
Considering that PASCAL has 1,464 images for training and 1,449 for validation, the average number of superpixels per image is around 29.
% Given that PASCAL contains 1,464 images for training and 1,449 for validation sets, the average is approximately 29 superpixels per image.
Then we correct the pseudo label of each superpixel by annotating the true label to the corresponding representative pixel and expanding the label to the superpixel.
% Our next step is to fix incorrectly labeled superpixels by assigning them the correct label from their corresponding representative pixel. 
% Our next step involves rectifying incorrectly labeled superpixels with accurate labels across all images. 
% This is done by first recording the correct label for each representative pixel that requires correction, and then assigning these clean labels to all the pixels within the same superpixel.
The relabeling tasks are conducted by two annotators, each spending around 60 hours over two weeks. 
% In cases where the labels assigned by the two annotators differ, a consensus is reached through discussion.
When labels from two annotators are different, the final annotation is determined by discussion.
The qualitative result of PASCAL+ compared to PASCAL is illustrated in Figures~\ref{fig:pascal-pascal-plus} and~\ref{fig:extra_qual}.
Additionally, in Figure~\ref{fig:error_qual}, we report few failure cases for label correction due to the imperfection of superpixels, budget constraints, and human error.
% Despite our efforts to construct a clean segmentation dataset, the PASCAL+ dataset still contains some noisy labels, as illustrated in Figures~\ref{fig:extra_qual} and~\ref{fig:error_qual}, 
% Despite our efforts to construct a clean segmentation dataset, the PASCAL+ dataset still contains some noisy labels, as illustrated in Figures~\ref{fig:extra_qual} and~\ref{fig:error_qual}, due to the 

% \noindent\textbf{Influence of PASCAL+.}
% PASCAL+ not only enhances the reliability of segmentation model evaluations but also has the potential to reduce both false negatives and false positives in the literature of detecting noisy labels for segmentation tasks, thereby contributing to more reliable and precise outcomes in this field.

\begin{table}[t!]
\caption{{\em Effect of PASCAL+.} P denotes PASCAL, while P+ denotes PASCAL+.
% The impact of the refined train set becomes apparent when evaluated using the same valid set.}
The refined train set increases model performance on both the original and refined validation sets.}
% \vspace{-3mm}
\begin{center}
\begin{small}
\setlength\tabcolsep{4pt}
\centering
\begin{tabular}{cc|cc}
\toprule
Train & Valid & Data mIoU (\%) & Model mIoU (\%) \\ \midrule
P & P & 99.1 & 75.36$_{\pm 0.07}$ \\ 
P+ & P & 100.0 & 75.78$_{\pm 0.12}$\\ 
P & P+ & 99.1 & 76.18$_{\pm 0.08}$ \\ 
P+ & P+ & 100.0 & \textbf{76.42$_{\pm 0.03}$} \\
\bottomrule
\end{tabular}
\end{small}
\end{center}
\label{tab:effect-of-pp}
\vspace{-3mm}
\end{table}

\subsection{Analysis of PASCAL+}
\label{sec:exp-pp}
\noindent\textbf{Effect of PASCAL+.}
In PASCAL+, we make 743 superpixel label corrections in total, with 375 in the training set and 368 in the validation set.
% For PASCAL+, we correct a total of 743 superpixel labels, with 375 corrections in the training set and 368 in the validation set, respectively.
Approximately 0.5\% of the pixel labels, equivalent to 2.6 million pixels, are altered, resulting in a 0.9\% improvement in the mean Intersection over Union (mIoU) for the training set, as shown in Table~\ref{tab:effect-of-pp}.
% A total of 0.5\% pixel labels, i.e., 2.6 million pixels, are changed, which leads to a 0.9\% increase in mIoU for the training set as represented in Table~\ref{tab:effect-of-pp}.
Regardless of whether the valid set is PASCAL or PASCAL+, corrections to the training data enhance the mIoU by around 0.3\%.
In particular, Figure~\ref{fig:stats-(a)} represents that IoU scores for the pottedplant and sofa classes are increased by more than 2\%.
This trend is related to the distribution of the corrected classes in Figure~\ref{fig:stats-(b)}.
Excepting the background and person classes, which already achieve high IoU scores with PASCAL in Figure~\ref{fig:pascal-iou}, the IoU scores tend to improve in line with the number of corrections applied to classes that initially have more errors.
% ncrease in corrections made to classes that initially had a higher number of errors.
% \noindent\textbf{Influence of PASCAL+.}
PASCAL+ not only enhances the reliability of segmentation model evaluations but also has the potential to reduce both false negatives and false positives in the literature of detecting noisy labels for segmentation tasks, thereby contributing to more reliable and precise outcomes in this field.

\iffalse
For instance, we rectify the label of the object within the cyan box to person. However, identifying such a dark object as a person is challenging even for annotators, which could potentially impede the training process. In addition, while our PASCAL+ successfully segments the bottle located close to the bicycle, neglecting these bottles could lead to a better IoU. Regarding the cow class, we speculate that the increase in the similar sheep class, as depicted in Figure 2 of Link 3, causes a decrease in IoU. Since we cannot know the effect of correction beforehand, the best way is to correct all noisy labels, as we have done for PASCAL+.
\fi

\noindent\textbf{Various Acquisitions for PASCAL+.}
Since it is possible to access both the noisy PASCAL and clean PASCAL+ datasets at the same time, we analyze which acquisition function is effective in real-world. Table~\ref{tab:acq-pascal-plus} indicates that our \sftype{SIM} acquisition achieves nearly 100\% Data mIoU, i.e., almost similar to PASCAL+, with selecting 10K pixels for correction.
As the training dataset's quality improves, there is a corresponding slight increase in model performance.

%% file: Sections/6_conclusion.tex
\section{Conclusion}
In this work, we propose a framework for active label correction in semantic segmentation operating with foundation models.
Our framework includes cost-efficient correction queries, which are verified theoretically and empirically, that ask for a pixel label to be corrected if needed.
We fully enjoy the benefits of foundation models, namely initial pseudo-labels and decent superpixels, resulting in significant budget reduction across various datasets in different domains.
In addition, we demonstrate the practicality of our framework by constructing PASCAL+, a corrected version of the PASCAL dataset.

\noindent\textbf{Limitations.}
Our framework depends on foundation models, particularly Grounded-SAM~\cite{liu2023grounding}, and shares the same inherent limitations as these models, like generating incomplete superpixels for minor domains.
However, we demonstrate the effectiveness of our framework in the medical field, and we expect these issues to be resolved as foundation models continue to improve over time.

\begin{table}[t!]
\caption{{\em Performance of corrected dataset.} With 10K budgets, we correct PASCAL to PASCAL+ with different acquisitions.}
% \vspace{-3mm}
\begin{center}
\begin{small}
\setlength\tabcolsep{6pt}
\centering
\begin{tabular}{c|cc}
\toprule
Acquisition function & Data mIoU (\%) & Model mIoU (\%) \\ \midrule
LCIL & 99.16$_{\pm 0.00}$ & 75.68$_{\pm 0.25}$ \\
SoftMin & 99.38$_{\pm 0.01}$ & 75.76$_{\pm 0.23}$ \\ 
AIoU & 99.28$_{\pm 0.02}$ & 75.61$_{\pm 0.22}$ \\
SIM (ours) & \textbf{99.78$_{\pm 0.11}$} & \textbf{75.87$_{\pm 0.22}$} \\ 
\bottomrule
\end{tabular}
\end{small}
\end{center}
\label{tab:acq-pascal-plus}
\vspace{-3mm}
\end{table}

% active learning from the web

%% file: Sections/7_appendix.tex
\clearpage
\appendix

\section{Text Prompts for Warm-start}
\label{app:text-prompts}
In Section~\ref{sec:initial-dataset-preparation}, for the warm-start process, we generate initial pseudo labels through a sequence of text prompts described. For example, we employ “Road. Sidewalk. Building. … Bicycle.” prompts for Cityscapes and “Aeroplane. Bicycle. Bird. … Tvmonitor.” for PASCAL, where each word aligns with the respective target class. However, each prompt, such as “Diningtable”, can be segmented into multiple tokens, such as “Dining” and “table”. Therefore, we assign each token to its corresponding class to derive the initial labels for the warm-start process.

\section{User Study with Different Queries}
\label{app:user-study}
% We theoretically compare the costs of classification and correction queries in Theorem~\ref{the:queries}.
% \subsection{Verification }\
% \subsection{Details of User Study}
% \begin{figure}[h!]
%     \centering
%     \includegraphics[width=0.99\linewidth]{Figures/actcorr_query.pdf}
%     \caption{\textit{Questionnaire of correction query}. Each question includes an instruction, an image displaying an object delineated by a bounding box, a representative pixel, and class options.}
%     \label{fig:actcorr_queries}
% \end{figure}
To verify the efficiency of the proposed correction query $C_{\textnormal{cor}}$ in Active Label Correction (ALC) compared to classification query $C_{\textnormal{cls}}$ in conventional Active Learning (AL),
we conduct a user study focusing on actual labeling costs, specifically annotation time.
% We conduct a user study to compare the proposed correction query $C_{\textnormal{cor}}$ in Active Label Correction (ALC) and classification query $C_{\textnormal{cls}}$ in conventional Active Learning (AL) in terms of actual labeling cost, i.e., the annotation time.
The example of the correction query questionnaire is illustrated in Figure~\ref{fig:cor-query}, and the results are summarized in Table~\ref{tab:user-study-results}.
% As in Figure.~\ref{fig:actcorr_queries}
% for each question, users are given an instruction, an image where an object is marked with a bounding box and a representative pixel, and class selection options.
Each question presents the user with instructions, an image with an object highlighted, and options for classifying the object.
For the correction query scenario, the instructions include the pseudo label of the foundation model, and users only need to correct if the pseudo label is incorrect.
% are asked to correct the answer only when the prediction is incorrect.
The detailed instruction for correction query is given as follows:

\begin{center}
\it Is this pixel a \textbf{TV}? \\
\it Give the correct label only if the pseudo label is incorrect.
\end{center}

On the other hand, the example instruction for classification query is given as follows:
\begin{center}
\it Give the correct label of the pixel.
\end{center}

% \begin{quote}
% \centering
% ``Is this pixel a \textbf{TV}? \\ Give the correct label only if the pseudo label is incorrect.''.
% \end{quote}
% On the other hand, the example instruction for classification query is given as follows:
% \begin{quote}
% \centering
% ``Give the correct label of the pixel.''.
% \end{quote}

% Using a ground-truth segmentation mask,
% we divided regions into two groups based on the correctness of the foundation model, where the number of correct images and incorrect images is identical (\textit{i.e}, p=0.5).
Based on the ground-truth,
we collect 20 images consisting of 10 images with correct pseudo labels,
% from the foundation model
and 10 images with incorrect pseudo labels counterparts, i.e., $p = 0.5$.
We ask for labels of these images in both correction queries and classification queries.
A total of 20 volunteers participates in the survey.
To prevent the user from memorizing images, we only ask one type of query per user, which means we ask the correction queries to 10 users and the classification queries to the others.
The responses from annotators are evaluated by calculating the accuracy of the classification prediction.
As shown in Table~\ref{tab:user-study-results}, the correction query only requires 75\% labeling time of that of the classification query.
In terms of accuracy, both queries show the same 95\%.
% We observe that the 

\section{Absolute Performance of ALC vs. AL}
\label{sec:abs-alc-al}
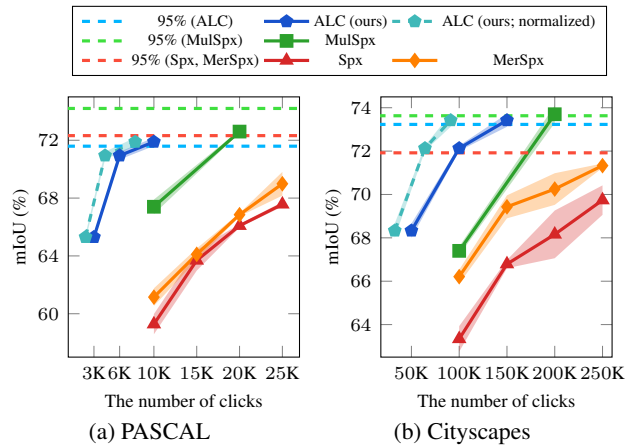
\begin{figure}[h!]
    \captionsetup[subfigure]{font=footnotesize,labelfont=footnotesize,aboveskip=0.05cm,belowskip=-0.15cm}
    \centering
    \hspace{-3mm}
    \begin{subfigure}{.47\linewidth}
        \centering
        \begin{tikzpicture}
            \begin{axis}[
                legend style={nodes={scale=0.6}, at={(2.33, 1.35)}},
                legend columns=3,
                xlabel={The number of clicks},
                ylabel={mIoU (\%)},
                width=1.2\linewidth,
                height=1.3\linewidth,
                ymin=57,
                ymax=75,
                xtick={3, 6, 10, 15, 20, 25},
                ytick={60, 64, 68, 72},
                xlabel style={yshift=0.15cm},
                ylabel style={yshift=-0.6cm},
                xmin=0,
                xmax=27,
                label style={font=\scriptsize},
                tick label style={font=\scriptsize},
                xticklabel={$\pgfmathprintnumber{\tick}$K}
            ]
            % Full - ALC
            \addplot[cBlue4, very thick, mark size=2pt, mark options={solid}, dashed] coordinates {(0,-1)};
            \draw [cBlue4, very thick, dashed] (axis cs:-2,71.59) -- (axis cs:30,71.59);
            % Ours
            \addplot[cBlue, very thick, mark=pentagon*, mark size=2pt, mark options={solid}] table[col sep=comma, x=x, y=ALC]{Data/alc_pascal_abs.csv};
            % Ours
            \addplot[cBlue2, very thick, dashed, mark=pentagon*, mark size=2pt, mark options={solid}] table[col sep=comma, x=x, y=BALC]{Data/alc_pascal_abs.csv};
            % Full - Mul
            \addplot[cGreen2, very thick, mark size=2pt, mark options={solid}, dashed] coordinates {(0,-1)};
            \draw [cGreen2, very thick, dashed] (axis cs:-2,74.2) -- (axis cs:30,74.2);
            % NeurIPS, 23
            \addplot[cGreen, very thick, mark=square*, mark size=2pt, mark options={solid}] table[col sep=comma, x=x, y=MulSpx]{Data/alc_pascal_abs.csv};
            % Empty
            \addplot[cWhite, very thick, mark size=2pt, mark options={solid}, dashed] coordinates {(0,-1)};
            % Full - Mer, Spx
            \addplot[cRed4, very thick, mark size=2pt, mark options={solid}, dashed] coordinates {(0,-1)};
            \draw [cRed4, very thick, dashed] (axis cs:-2,72.32) -- (axis cs:30,72.32);
            % CVPR, 22
            \addplot[cRed, very thick, mark=triangle*, mark size=2pt, mark options={solid}] table[col sep=comma, x=x, y=Spx]{Data/alc_pascal_abs.csv};
            % ICCV, 23
            \addplot[orange, very thick, mark=diamond*, mark size=2pt, mark options={solid}] table[col sep=comma, x=x, y=MerSpx]{Data/alc_pascal_abs.csv};

            % CVPR, 22
            \addplot[name path=Spx-U, draw=none, fill=none] table[col sep=comma, x=x, y=Spx-U]{Data/alc_pascal_abs.csv};
            \addplot[name path=Spx-D, draw=none, fill=none] table[col sep=comma, x=x, y=Spx-D]{Data/alc_pascal_abs.csv};
            \addplot[cRed, fill opacity=0.3] fill between[of=Spx-U and Spx-D];
            % ICCV, 23
            \addplot[name path=MerSpx-U, draw=none, fill=none] table[col sep=comma, x=x, y=MerSpx-U]{Data/alc_pascal_abs.csv};
            \addplot[name path=MerSpx-D, draw=none, fill=none] table[col sep=comma, x=x, y=MerSpx-D]{Data/alc_pascal_abs.csv};
            \addplot[orange, fill opacity=0.3] fill between[of=MerSpx-U and MerSpx-D];
            % NeurIPS, 23
            \addplot[name path=MulSpx-U, draw=none, fill=none] table[col sep=comma, x=x, y=MulSpx-U]{Data/alc_pascal_abs.csv};
            \addplot[name path=MulSpx-D, draw=none, fill=none] table[col sep=comma, x=x, y=MulSpx-D]{Data/alc_pascal_abs.csv};
            \addplot[cGreen, fill opacity=0.3] fill between[of=MulSpx-U and MulSpx-D];
            % Ours
            \addplot[name path=ALC-U, draw=none, fill=none] table[col sep=comma, x=x, y=ALC-U]{Data/alc_pascal_abs.csv};
            \addplot[name path=ALC-D, draw=none, fill=none] table[col sep=comma, x=x, y=ALC-D]{Data/alc_pascal_abs.csv};
            \addplot[cBlue, fill opacity=0.3] fill between[of=ALC-U and ALC-D];
            % Ours
            \addplot[name path=BALC-U, draw=none, fill=none] table[col sep=comma, x=x, y=BALC-U]{Data/alc_pascal_abs.csv};
            \addplot[name path=BALC-D, draw=none, fill=none] table[col sep=comma, x=x, y=BALC-D]{Data/alc_pascal_abs.csv};
            \addplot[cBlue2, fill opacity=0.3] fill between[of=BALC-U and BALC-D];
            
            \legend{95\% (ALC), ALC (ours), ALC (ours; normalized), 95\% (MulSpx), {MulSpx}, { }, {95\% (Spx, MerSpx)}, Spx, MerSpx}
            \end{axis}
        \end{tikzpicture}
        \caption{PASCAL}
    \label{fig:alc-vs-al-pascal-abs}
    \end{subfigure}
    \hspace{1mm}    
    \begin{subfigure}{.47\linewidth}
        \centering
        \begin{tikzpicture}
            \begin{axis}[
                xlabel={The number of clicks},
                ylabel={mIoU (\%)},
                width=1.2\linewidth,
                height=1.3\linewidth,
                ymin=62.5,
                ymax=74.5,
                xtick={50, 100, 150, 200, 250},
                ytick={62, 64, 66, 68, 70, 72, 74},
                xlabel style={yshift=0.15cm},
                ylabel style={yshift=-0.6cm},
                xmin=18,
                xmax=260,
                label style={font=\scriptsize},
                tick label style={font=\scriptsize},
                xticklabel={$\pgfmathprintnumber{\tick}$K}
            ]
            % Full - ALC
            \draw [cBlue4, very thick, dashed] (axis cs:-2, 73.23) -- (axis cs:300,73.23);
            % Ours
            \addplot[cBlue, very thick, mark=pentagon*, mark size=2pt, mark options={solid}] table[col sep=comma, x=x, y=ALC]{Data/alc_city_abs.csv};
            % Ours
            \addplot[cBlue2, very thick, dashed, mark=pentagon*, mark size=2pt, mark options={solid}] table[col sep=comma, x=x, y=BALC]{Data/alc_city_abs.csv};
            % Full - MulSpx
            \draw [cGreen2, very thick, dashed] (axis cs:-2,73.63) -- (axis cs:300,73.63);
            % NeurIPS, 23
            \addplot[cGreen, very thick, mark=square*, mark size=2pt, mark options={solid}] table[col sep=comma, x=x, y=MulSpx]{Data/alc_city_abs.csv};
            % Full - Spx, MerSpx
            \draw [cRed4, very thick, dashed] (axis cs:-2,71.92) -- (axis cs:300,71.92);
            % CVPR, 22
            \addplot[cRed, very thick, mark=triangle*, mark size=2pt, mark options={solid}] table[col sep=comma, x=x, y=Spx]{Data/alc_city_abs.csv};
            % ICCV, 23
            \addplot[orange, very thick, mark=diamond*, mark size=2pt, mark options={solid}] table[col sep=comma, x=x, y=MerSpx]{Data/alc_city_abs.csv};

            % CVPR, 22
            \addplot[name path=Spx-U, draw=none, fill=none] table[col sep=comma, x=x, y=Spx-U]{Data/alc_city_abs.csv};
            \addplot[name path=Spx-D, draw=none, fill=none] table[col sep=comma, x=x, y=Spx-D]{Data/alc_city_abs.csv};
            \addplot[cRed, fill opacity=0.3] fill between[of=Spx-U and Spx-D];
            % ICCV, 23
            \addplot[name path=MerSpx-U, draw=none, fill=none] table[col sep=comma, x=x, y=MerSpx-U]{Data/alc_city_abs.csv};
            \addplot[name path=MerSpx-D, draw=none, fill=none] table[col sep=comma, x=x, y=MerSpx-D]{Data/alc_city_abs.csv};
            \addplot[orange, fill opacity=0.3] fill between[of=MerSpx-U and MerSpx-D];
            % NeurIPS, 23
            \addplot[name path=MulSpx-U, draw=none, fill=none] table[col sep=comma, x=x, y=MulSpx-U]{Data/alc_city_abs.csv};
            \addplot[name path=MulSpx-D, draw=none, fill=none] table[col sep=comma, x=x, y=MulSpx-D]{Data/alc_city_abs.csv};
            \addplot[cGreen, fill opacity=0.3] fill between[of=MulSpx-U and MulSpx-D];
            % Ours
            \addplot[name path=ALC-U, draw=none, fill=none] table[col sep=comma, x=x, y=ALC-U]{Data/alc_city_abs.csv};
            \addplot[name path=ALC-D, draw=none, fill=none] table[col sep=comma, x=x, y=ALC-D]{Data/alc_city_abs.csv};
            \addplot[cBlue, fill opacity=0.3] fill between[of=ALC-U and ALC-D];
            % Ours
            \addplot[name path=BALC-U, draw=none, fill=none] table[col sep=comma, x=x, y=BALC-U]{Data/alc_city_abs.csv};
            \addplot[name path=BALC-D, draw=none, fill=none] table[col sep=comma, x=x, y=BALC-D]{Data/alc_city_abs.csv};
            \addplot[cBlue2, fill opacity=0.3] fill between[of=BALC-U and BALC-D];

            \end{axis}
        \end{tikzpicture}
        \caption{Cityscapes}
    \end{subfigure}
    \caption{{\em Effect of active label correction.} Our \sftype{ALC} shows comparable results on both datasets with much fewer clicks. Our \sftype{ALC (normalized)} reflects the reduced budget of correction queries in Theorem~\ref{the:queries}.} 
    % \js{change the legend name for cyan ours, perhaps ACL(ours; normalized)}}
    \label{fig:alc-vs-al-abs}
\end{figure}

While conventional AL for semantic segmentation methods use the same DeepLab-v3+~\cite{chen2018encoder} segmentation decoder combined with backbone pre-trained with the ImageNet~\cite{deng2009imagenet} dataset, the architecture of their backbones are slightly different.
\sftype{ALC} (ours) utilize plain ResNet101, \sftype{MulSpx}~\cite{hwang2023active} use ResNet101 combined with deepstem tricks~\cite{he2019bag}, and \sftype{MerSpx}~\cite{Kim_2023_ICCV} and \sftype{Spx}~\cite{cai2021revisiting} employ Xception-65~\cite{chollet2017xception}.
% Thus, the fully supervised model of the AL baselines varies, and the AL methods combined with better backbone architecture tend to show better performance since they learn the knowledge during the pre-training on ImageNet.
Figure~\ref{fig:alc-vs-al} presents the performance in terms of recovery rate relative to a fully supervised model, calculated as the ratio of our model's performance to that of the fully supervised model.
% Thus, Figure~\ref{fig:alc-vs-al} reports the performance as the recovery rate compared to the fully supervised model, which is more of the model divided by that of the fully supervised model.

Here, we additionally report the comparison with absolute mIoU in Figure~\ref{fig:alc-vs-al-abs} over various budget levels, represented by the number of clicks, for both PASCAL and Cityscapes datasets.
The 95\% performance of each baseline's fully supervised model is illustrated with a dashed line labeled as \sftype{95\% ($\cdot$)}.
Our proposed \sftype{ALC} method consistently demonstrate the most efficient performance.
% continues to show the most efficient performance relative to budget.

\begin{figure*}[!t]
    \centering
    % width: 131.24
    \begin{subfigure}[h!]{.245\linewidth}
        \centering
        \includegraphics[scale=0.1706]{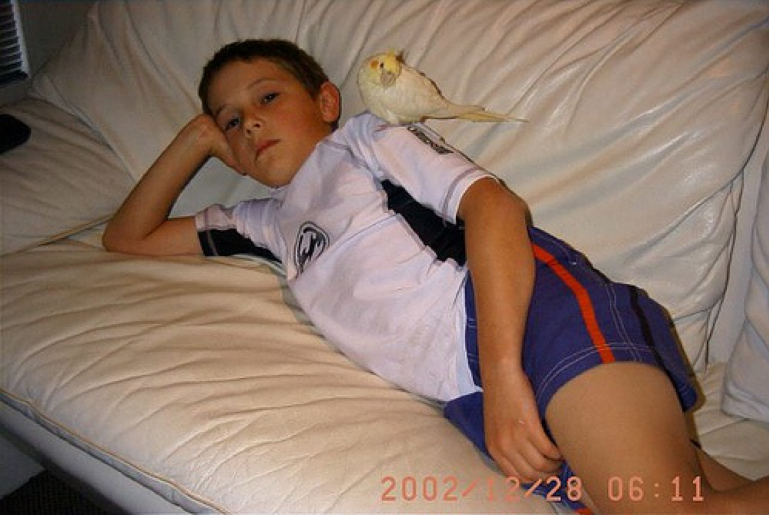}
    \end{subfigure}
    \begin{subfigure}[h!]{.245\linewidth}
        \centering
        \includegraphics[scale=0.1706]{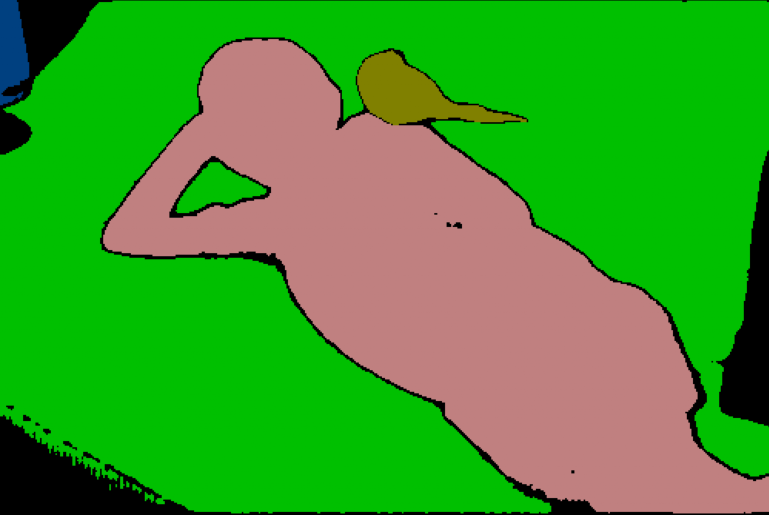}
    \end{subfigure}
    \begin{subfigure}[h!]{.245\linewidth}
        \centering
        \includegraphics[scale=0.1706]{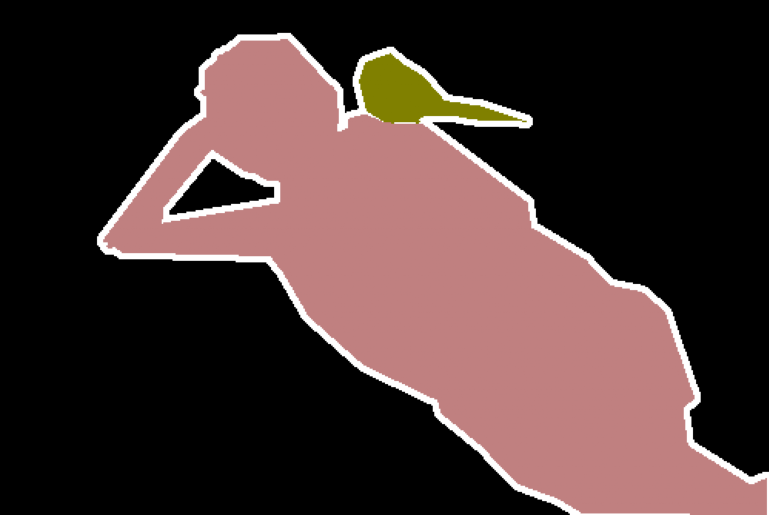}
    \end{subfigure}
    \begin{subfigure}[h!]{.245\linewidth}
        \centering
        \includegraphics[scale=0.1706]{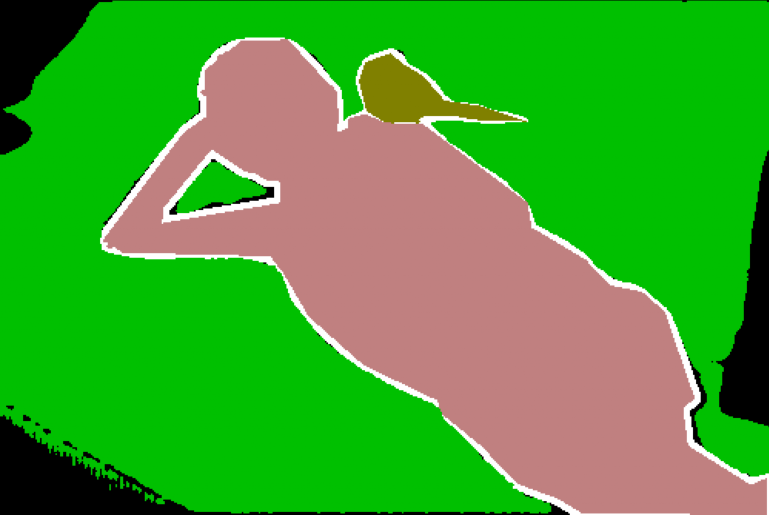}
    \end{subfigure}

    \begin{subfigure}[h!]{.245\linewidth}
        \centering
        \includegraphics[scale=0.191]{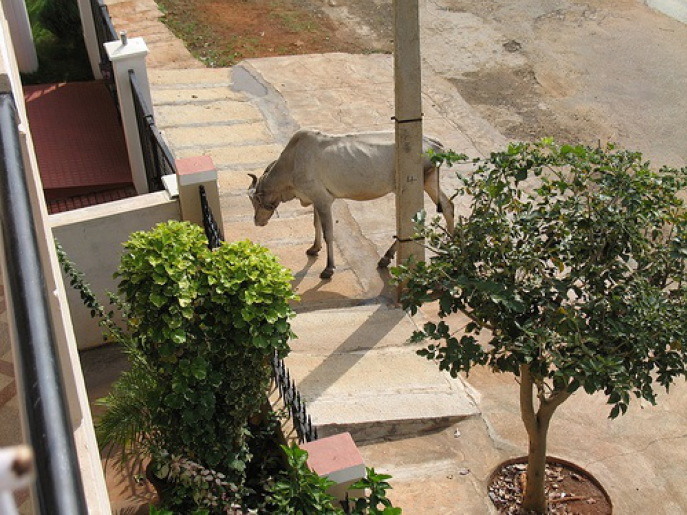}
    \end{subfigure}
    \begin{subfigure}[h!]{.245\linewidth}
        \centering
        \includegraphics[scale=0.191]{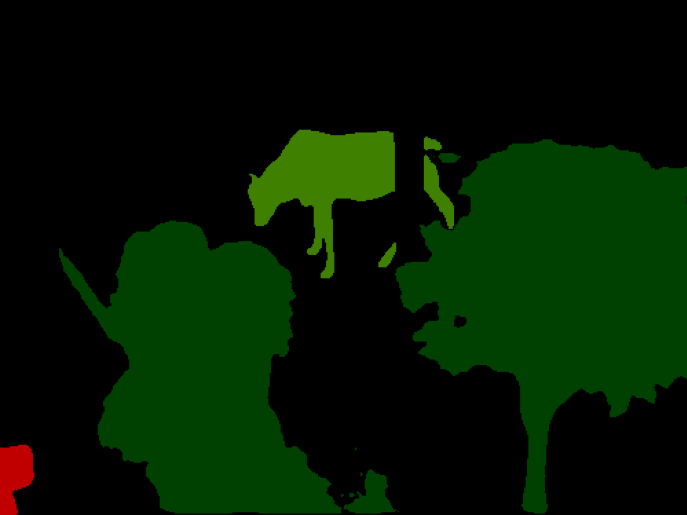}
    \end{subfigure}
    \begin{subfigure}[h!]{.245\linewidth}
        \centering
        \includegraphics[scale=0.191]{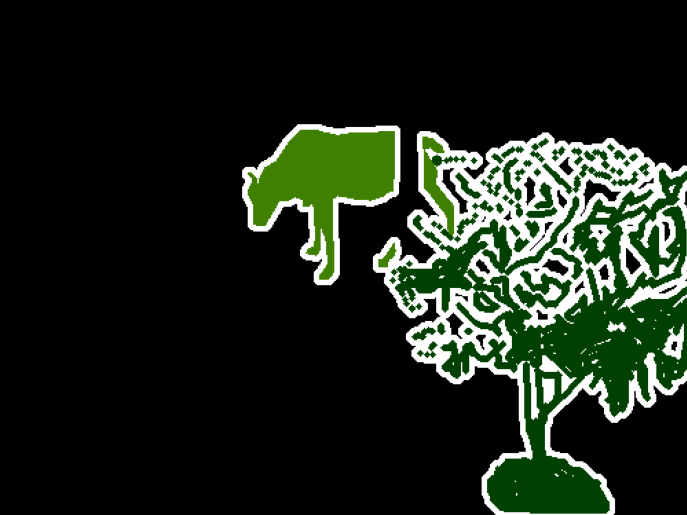}
    \end{subfigure}
    \begin{subfigure}[h!]{.245\linewidth}
        \centering
        \includegraphics[scale=0.191]{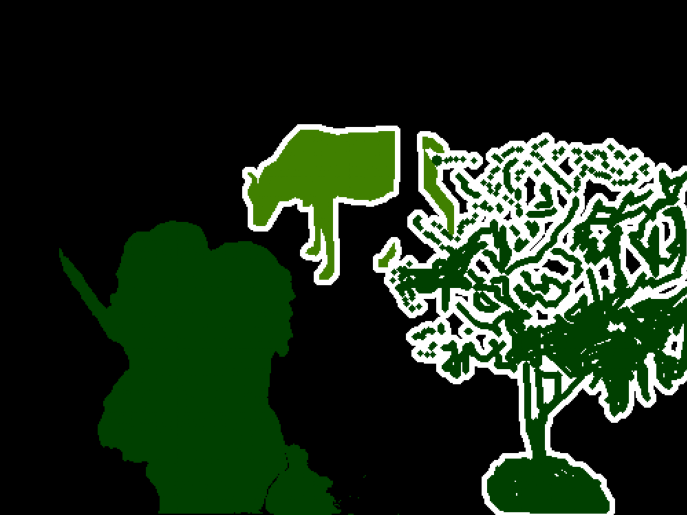}
    \end{subfigure}
    
    \begin{subfigure}[h!]{.245\linewidth}
        \centering
        \includegraphics[scale=0.191]{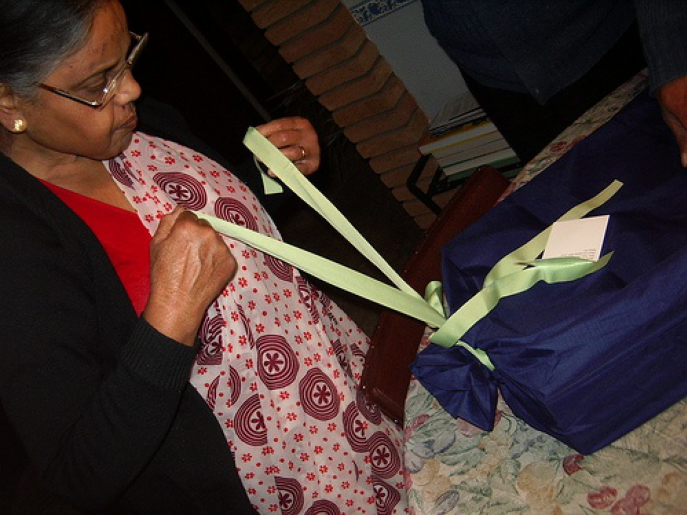}
        \caption{Unlabeled image}
        % \label{fig:org-images}
    \end{subfigure}
    \begin{subfigure}[h!]{.245\linewidth}
        \centering
        \includegraphics[scale=0.191]{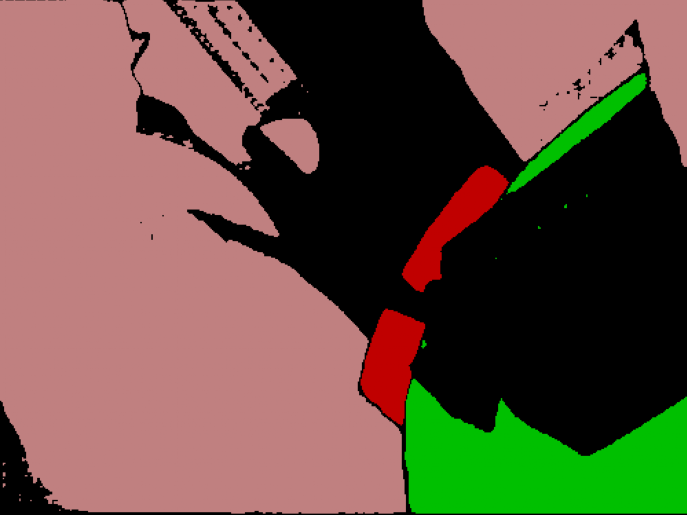}
        \caption{Grounded-SAM}
        % \label{fig:grounded-sam-images}
    \end{subfigure}
    \begin{subfigure}[h!]{.245\linewidth}
        \centering
        \includegraphics[scale=0.191]{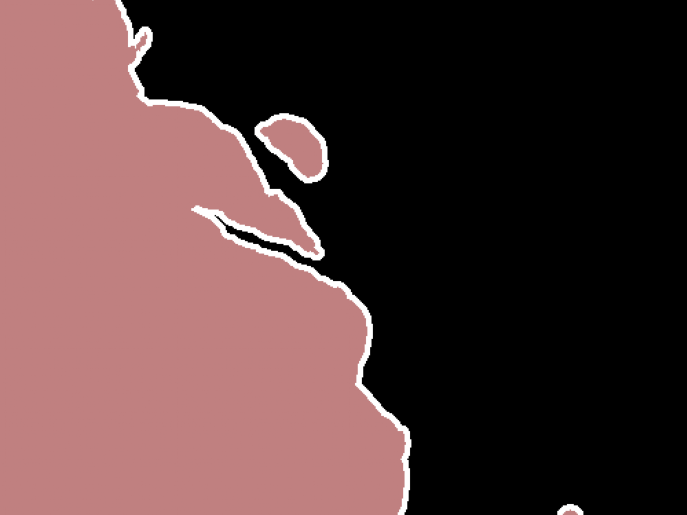}
        \caption{PASCAL}
        % \label{fig:org-labels}
    \end{subfigure}
    \begin{subfigure}[h!]{.245\linewidth}
        \centering
        \includegraphics[scale=0.191]{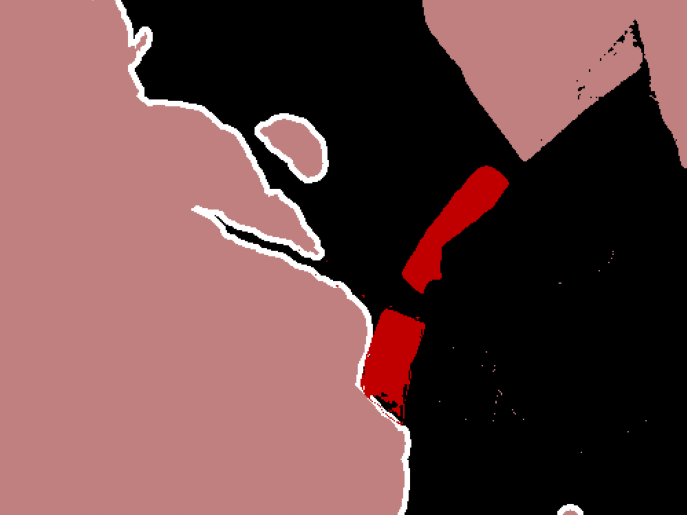}
        \caption{PASCAL+ (ours)}
        % \label{fig:pascal+}
    \end{subfigure}
    \caption{{\em Additional examples of noisy and corrected labels in PASCAL.} We correct PASCAL into PASCAL+ utilizing the superpixels of Grounded-SAM.}
    \label{fig:extra_qual}
\end{figure*}

\section{Ablation Studies}

\subsection{Initial Pseudo Labels}
\begin{table}[h!]
\caption{{\em Performance of initial pseudo labels from Grounded-SAM.} Noisy pseudo labels cause the data and model mIoU to worsen.}
\label{tab:grounded-threshold}
\centering
\begin{center}
\begin{small}
\setlength\tabcolsep{3pt}
\begin{tabular}{c|ccc}
\toprule
Box-thresohld & \# of objects & Data mIoU (\%) & Model mIoU (\%) \\ \midrule
% 0.1 & 31,676 & 26.19 & 33.17 \\
0.2 & 11,257 & 55.32 & 59.04 \\ 
0.3 & 5,995 & 65.14 & 66.15 \\
0.4 & 3,890 & 66.71 & 65.30 \\
0.5 & 2,798 & 60.87 & 59.50 \\
\bottomrule
\end{tabular}
\end{small}
\end{center}
\end{table}

% We generate the initial pseudo labels using Grounded-SAM~\cite{liu2023grounding}, which needs to set two hyperparameters: box-threshold and text-threshold.

To evaluate the quality of pseudo labels generated by Grounded-SAM~\cite{liu2023grounding} on the PASCAL dataset, we measure the data and model mIoU while adjusting a hyperparameter.
Grounded-SAM operates with two hyperparameters: box-threshold and text-threshold.
The text-threshold aims to identify all potential classes with a potential value exceeding the threshold.
As we only focus on a specific class per a object, we employ the argmax function on the potential classes.
% aims to identify all possible classes for objects, but since we are focused on a specific object class, we just use the argmax function instead. 
% The box-threshold determines the confidence level in the identified object's bounding box. 
% A lower box-threshold means the models recognize more objects, as Table~\ref{tab:grounded-threshold} shows. 
% However, this often results in many incorrectly labeled objects, leading to lower mIoU for both data and model mIoU.
% The benefit of identifying many objects is that correcting the pseudo labels for all objects can lead to high performance, i.e., 72.59\%, 70.90\%, and 66.97\% for box-thresholds ranging from 0.2 to 0.4.
The box-threshold determines the confidence level in the bounding box of the identified object.
With a lower box-threshold, the foundation model can detect more objects, as demonstrated in Table~\ref{tab:grounded-threshold}.
However, this often leads to numerous incorrectly labeled objects, resulting in decreased mIoU for both data and model.
Yet, the benefit of detecting lots of objects lies in the potential for enhanced performance when correcting the pseudo labels of all detected objects, resulting in model mIoU of 72.59\%, 70.90\%, and 66.97\% for box-thresholds of 0.2, 0.3, and 0.4, respectively.

% The advantage of identifying a greater number of objects is that correcting the pseudo-labels for all objects can lead to improved performance, achieving 72.59\%, 70.90\%, and 66.97\% for box-thresholds ranging from 0.2 to 0.4.
% The purpose of text-threshold is to list all potential classes of identified objects, however, since we only need a specific class of objects, we ignore the text-threshold and use the argmax function instead.
% The box-threshold is about how sure foundation models are of the detected box surrounding objects.
% Setting a lower box-threshold implies foundation models find more objects, as shown in Table~\ref{tab:grounded-threshold}.
% However, since many objects with the wrong label are detected, it can be seen that the data and model mIoU beomce low.
% However, when we proceed with dominant labeling on each superpixel, i.e., only one majority class is assigned to each superpixel, we can achieve 72.59\%, 70.90\%, and 66.97\% Data mIoU with varying threshold 0.2 to 0.4, respectively.

\iffalse
Here, achievable segmentation accuracy (ASA) is defined as:
the segmentation accuracy when each superpixel $s \in S$ is associated with the oracle superpixel with the largest overlap.
The ASA is calculated as follows:
\begin{equation}
\text{ASA}(S; G) := \frac{\sum_{s \in S} \max_{g \in G} |s \cap g|}{\sum_{s \in S} |s|} \;,
\end{equation}
\fi
% 74.10, 72.59, 70.90, 66.97

\subsection{Similarity Threshold for Label Expansion}
\begin{table}[h!]
\caption{{\em Similarity threshold.} For correction, we select 5K pixels from the initial labels and adjust the extent of label expansion.}
\label{tab:epsilon}
\centering
\begin{center}
\begin{small}
\setlength\tabcolsep{6pt}
\begin{tabular}{c|cc}
\toprule
$\epsilon$ & Data mIoU (\%) & Model mIoU (\%) \\ \midrule
$0.0$ & $83.34$ & $68.71$ \\
$0.2$ & $82.85$ & $68.48$ \\
$0.4$ & $82.11$ & $68.58$ \\
$0.6$ & $81.17$ & $68.48$ \\
$0.8$ & $80.18$ & $68.32$ \\
$1.0$ & $55.61$ & $56.05$ \\
\bottomrule
\end{tabular}
\end{small}
\end{center}
\end{table}

During the label expansion phase detailed in Section~\ref{sec:look-ahead}, a challenge can emerge when superpixels contain pixels belonging to various classes, potentially diminishing the dataset's overall quality. 
To this end, we propose expanding the clean label of a pixel $x_i$ only to similar pixels within its corresponding superpixel $s_i$ as follows:
% utilizing cosine similarity:
% In the label expansion stage in Section~\ref{sec:look-ahead}, the issue can arise when superpixels consist of pixels from different classes, which can lower the overall quality of the dataset.
% To this end, we can expand the clean label of pixel $x_i$ only to similar pixels within the same superpixel $s_i$, i.e., $s'_i \subset s_i$ by utilizing the cosine similarity defined as:
\begin{equation}
s_i(x_i,\epsilon) := \{ x \in s_i : \text{cos} \big( f_\theta(x_i), f_\theta(x) \big) \ge \epsilon)\}\;,
\label{eq:epsilon}
\end{equation}
where the degree of expansion is determined by hyperparameter $\epsilon$.
The more incomplete the superpixel, the larger $\epsilon$ is required.
For our main experiments in Section~\ref{sec:exp}, we set $\epsilon$ as 0, indicating complete expansion, where $s_i(x_i,\epsilon) = s_i$.
Here, we investigate how the value of $\epsilon$ in~\eqref{eq:epsilon} affects results.
Since foundation models accurately generate superpixel boundaries in PASCAL, we observe that setting $\epsilon$ to 0, thereby allowing the corrected pixel label to cover the entire superpixel, yields the best performance, as demonstrated in Table~\ref{tab:epsilon}.

\begin{figure*}[!t]
    \centering
    % width: 131.24
    \begin{subfigure}[h!]{.245\linewidth}
        \centering
        \includegraphics[scale=0.17]{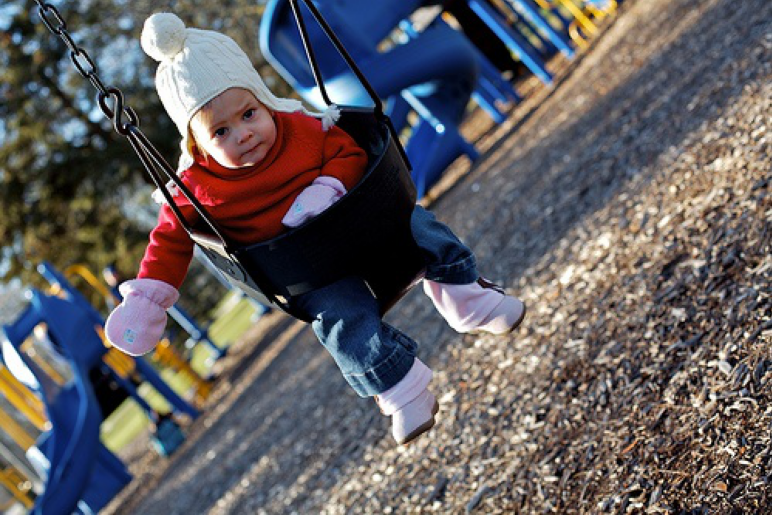}
    \end{subfigure}
    \begin{subfigure}[h!]{.245\linewidth}
        \centering
        \includegraphics[scale=0.17]{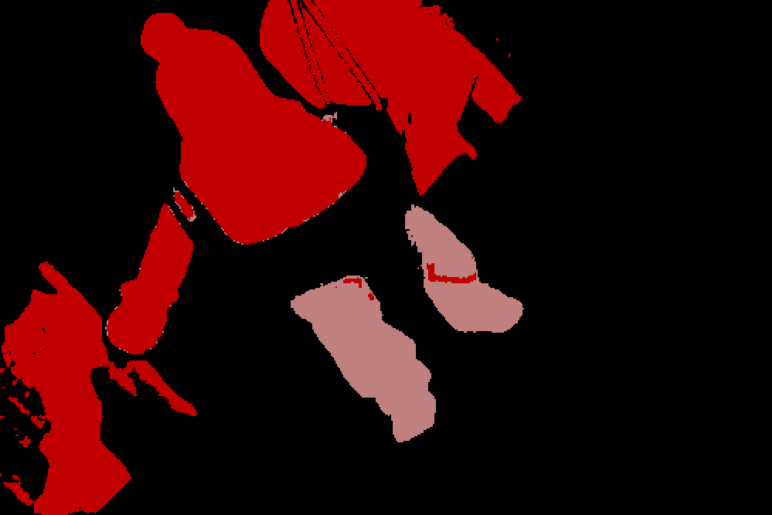}
    \end{subfigure}
    \begin{subfigure}[h!]{.245\linewidth}
        \centering
        \includegraphics[scale=0.17]{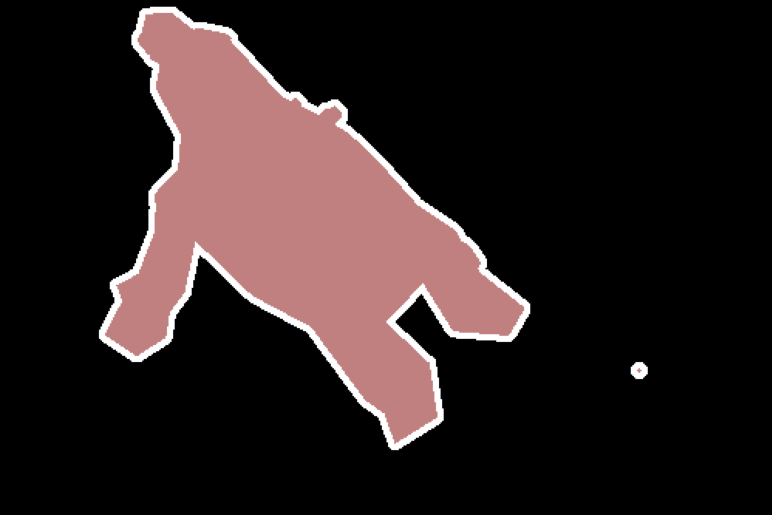}
    \end{subfigure}
    \begin{subfigure}[h!]{.245\linewidth}
        \centering
        \includegraphics[scale=0.17]{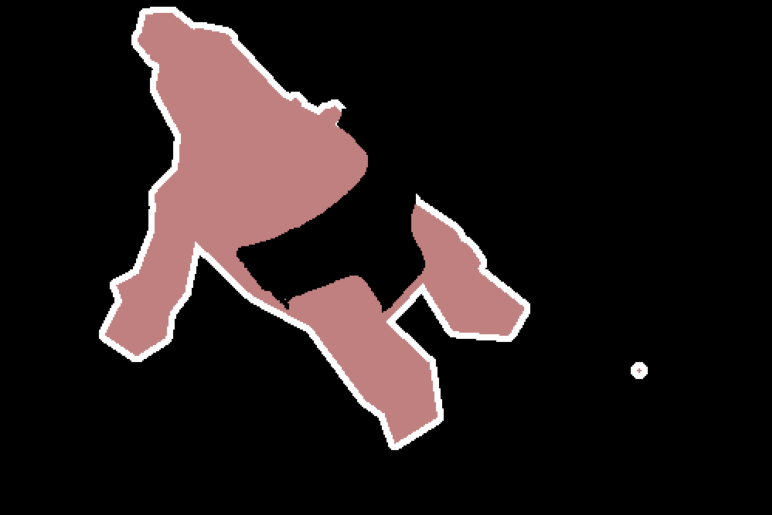}
    \end{subfigure}

    \begin{subfigure}[h!]{.245\linewidth}
        \centering
        \includegraphics[scale=0.168]{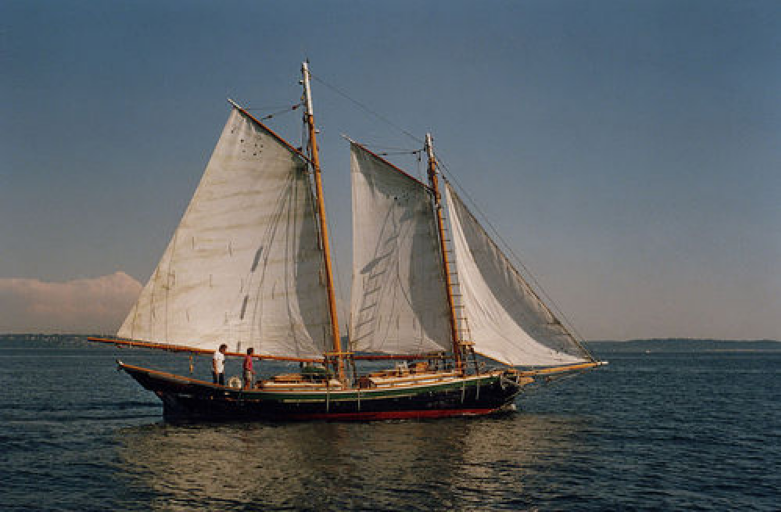}
    \end{subfigure}
    \begin{subfigure}[h!]{.245\linewidth}
        \centering
        \includegraphics[scale=0.168]{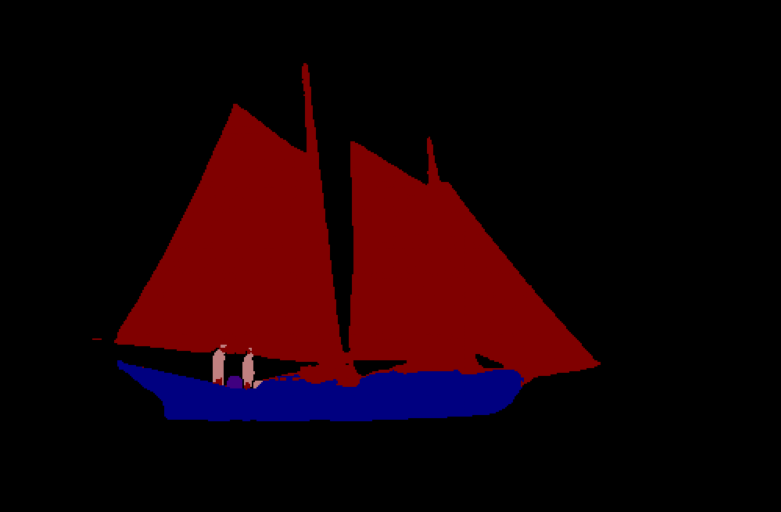}
    \end{subfigure}
    \begin{subfigure}[h!]{.245\linewidth}
        \centering
        \includegraphics[scale=0.168]{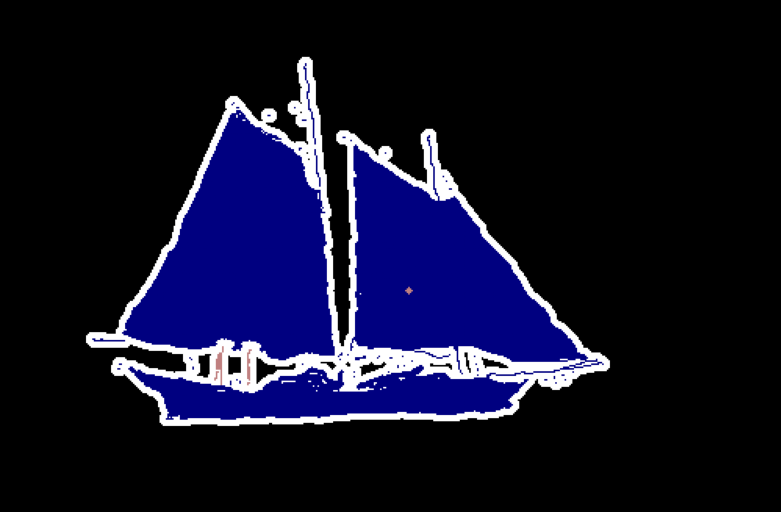}
    \end{subfigure}
    \begin{subfigure}[h!]{.245\linewidth}
        \centering
        \includegraphics[scale=0.168]{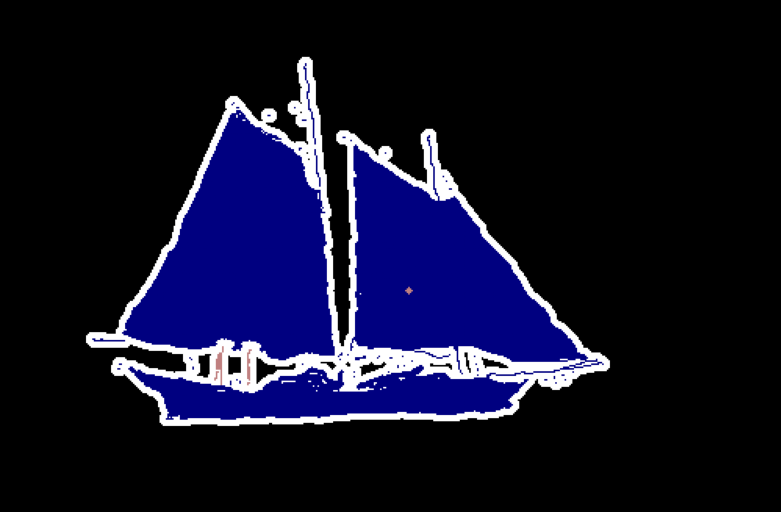}
    \end{subfigure}
    
    \begin{subfigure}[h!]{.245\linewidth}
        \centering
        \includegraphics[scale=0.191]{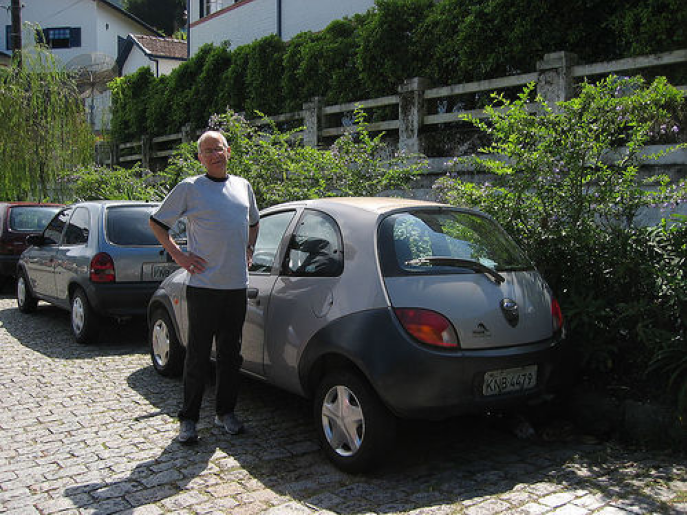}
        \caption{Unlabeled image}
        % \label{fig:org-images}
    \end{subfigure}
    \begin{subfigure}[h!]{.245\linewidth}
        \centering
        \includegraphics[scale=0.191]{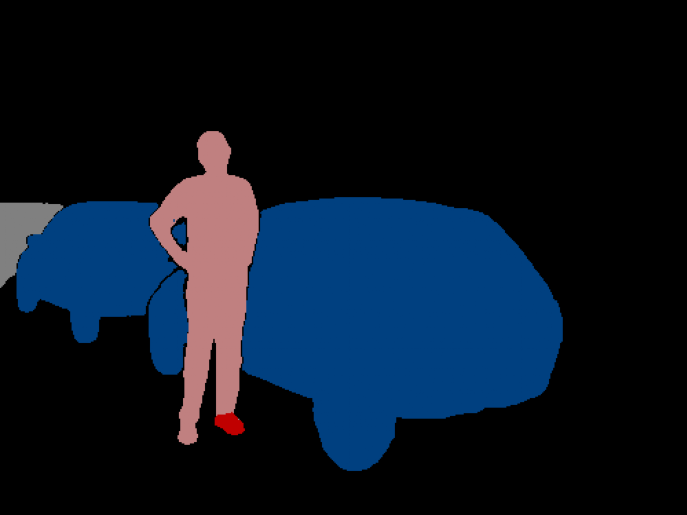}
        \caption{Grounded-SAM}
        % \label{fig:grounded-sam-images}
    \end{subfigure}
    \begin{subfigure}[h!]{.245\linewidth}
        \centering
        \includegraphics[scale=0.191]{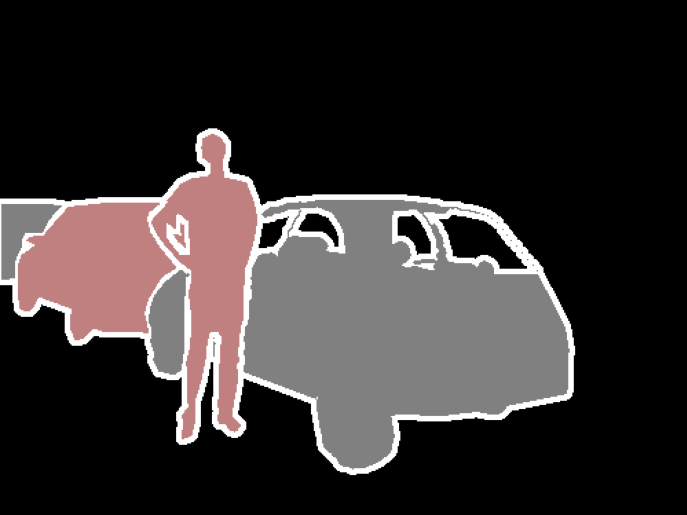}
        \caption{PASCAL}
        % \label{fig:org-labels}
    \end{subfigure}
    \begin{subfigure}[h!]{.245\linewidth}
        \centering
        \includegraphics[scale=0.191]{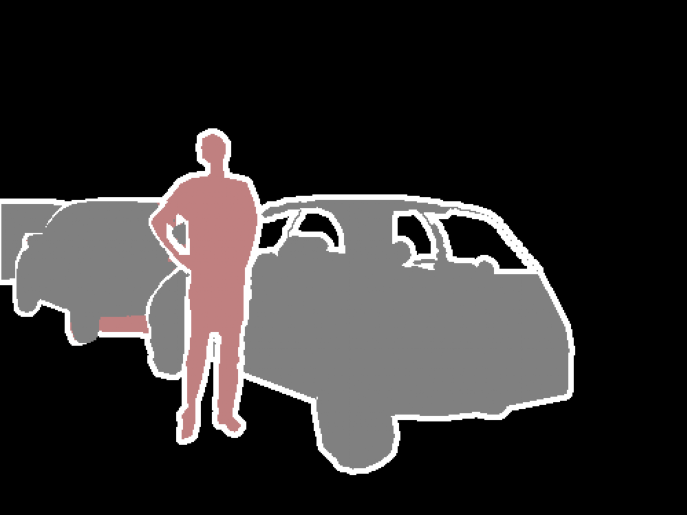}
        \caption{PASCAL+ (ours)}
        % \label{fig:pascal+}
    \end{subfigure}
    \caption{{\em Uncorrectable examples of noisy and corrected labels in PASCAL.} We correct PASCAL into PASCAL+ utilizing the superpixels of Grounded-SAM, however, due to the inherent limitations of superpixels, some failure cases can be observed.}
    \label{fig:error_qual}
\end{figure*}

% we get the best performance, as shown in Table~\ref{tab:epsilon}.
% In our experiments on the PASCAL dataset, we investigate how the value of $\epsilon$ in Equation~\eqref{eq:epsilon} influences the outcomes. Since foundation models accurately delineate superpixel boundaries in PASCAL, we find that setting $\epsilon$ to 0, thereby allowing the corrected pixel label to encompass the entire superpixel, yields the optimal performance, as demonstrated in Table~\ref{tab:epsilon}.

\subsection{Comparison with Other Diversified Pixel Pool}
% To address the problem of redundancy in the selected pixels, PixelPick~\cite{shin2021all} first rank all pixels using the acquisition function, then uniformly sampling pixels from the top 5\% ranked locations in each image.
% We compare the method in PixelPick with our proposed diversified pixel pool.
% Here, we apply all the other techniques including look-ahead acqusition and label expansion and only change pixel pool.
% In Table~\ref{tab:diversity}, our \textit{ALC} outperform PixelPick in both data and model mIoU. 
\begin{table}[h!]
\caption{{\em Experiments for diversified pixel pools.} With 5K budgets, we select pixels from different pixel pools and correct the initial labels to PASCAL.}
\label{tab:diversity}
\centering
\begin{center}
\begin{small}
\setlength\tabcolsep{6pt}
\begin{tabular}{c|cc}
\toprule
Methods & Data mIoU ($\%$) & Model mIoU ($\%$) \\ \midrule
PixelPick & 66.88 & 62.59 \\
ALC & 83.60 & 68.71 \\ 
\bottomrule
\end{tabular}
\end{small}
\end{center}
\end{table}

To solve the issue of picking similar pixels, as described in Section~\ref{sec:diversified-pixel-pool}, \sftype{PixelPick} employs an acquisition function to rank all pixels, subsequently uniformly selecting them from the top 5\% ranked pixels in each image~\cite{shin2021all}.
Thus, we contrast our diversified pixel pool based on superpixels with the PixelPick method.
For a fair comparison, we incorporate all other techniques, including \sftype{SIM} acqusition equipped with the concept of look-ahead and label expansion. 
As shown in Table~\ref{tab:diversity}, our \sftype{ALC} performs better than \sftype{PixelPick} in terms of both data and model mIoU.

\begin{figure*}[!t]
    \centering
    \begin{subfigure}[h!]{.33\linewidth}
        \centering
        \includegraphics[scale=0.175]{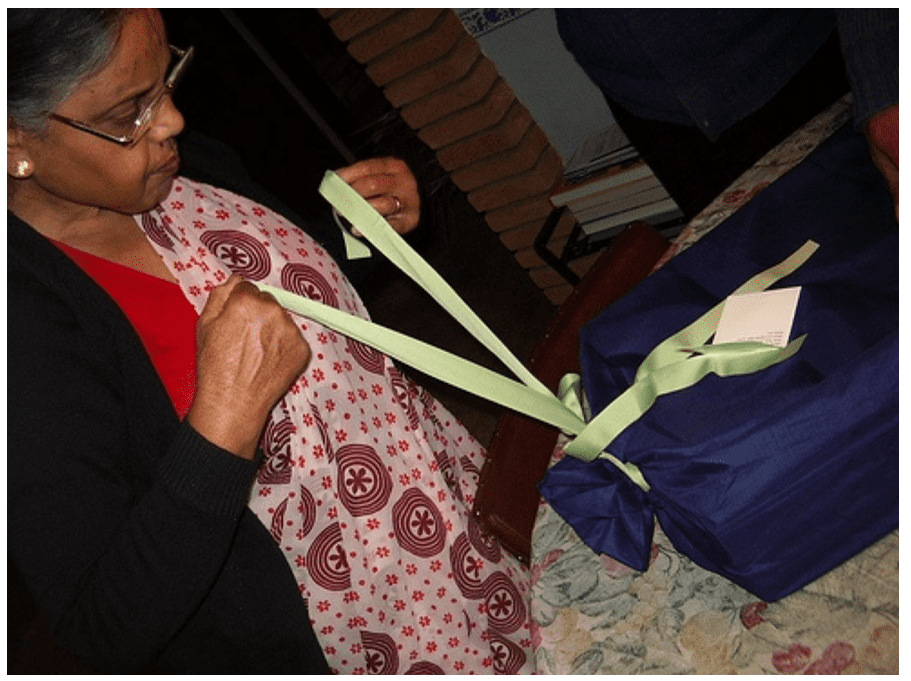}
    \end{subfigure}
    \begin{subfigure}[h!]{.33\linewidth}
        \centering
        \includegraphics[scale=0.175]{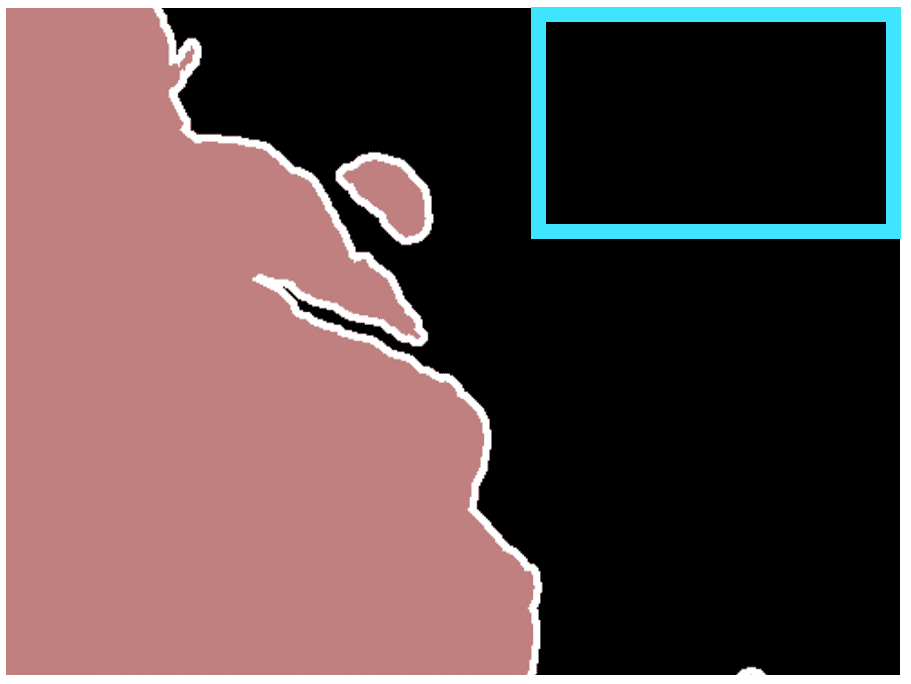}
    \end{subfigure}
    \begin{subfigure}[h!]{.33\linewidth}
        \centering
        \includegraphics[scale=0.175]{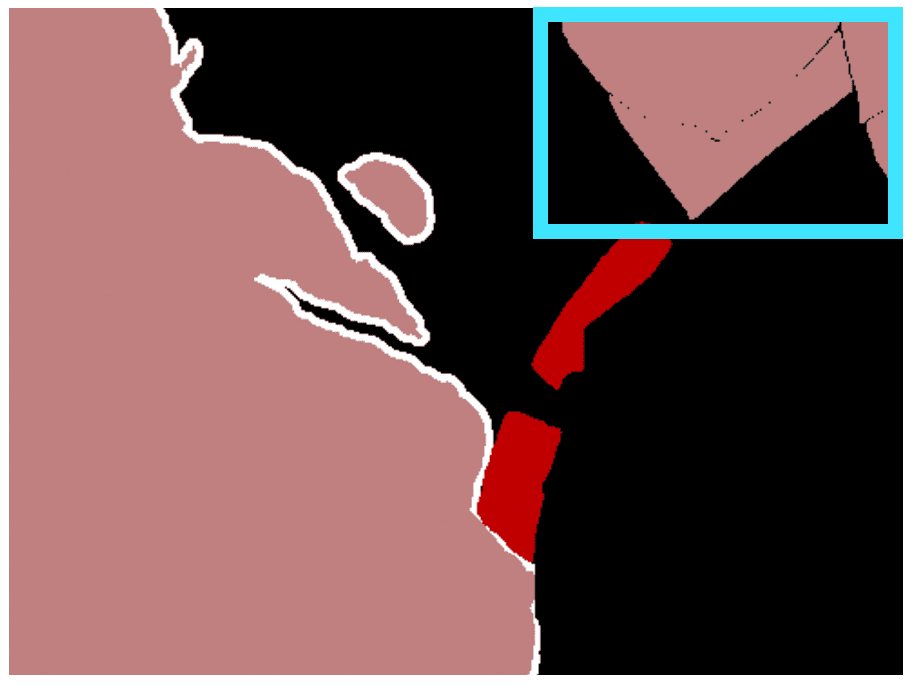}
    \end{subfigure}

    \begin{subfigure}[h!]{.33\linewidth}
        \centering
        \includegraphics[scale=0.175]{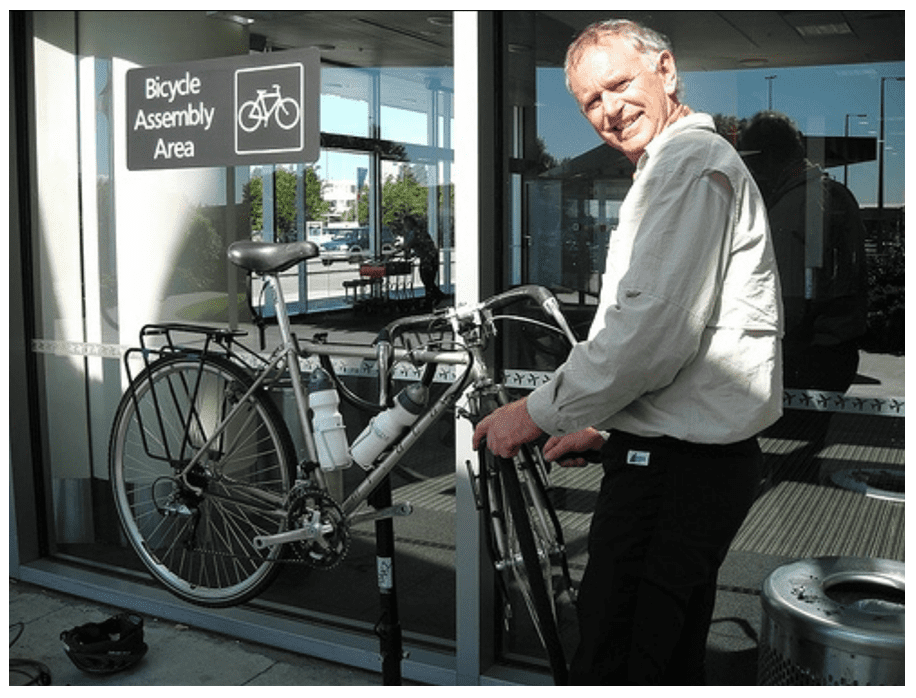}
        \caption{Image}
    \end{subfigure}
    \begin{subfigure}[h!]{.33\linewidth}
        \centering
        \includegraphics[scale=0.175]{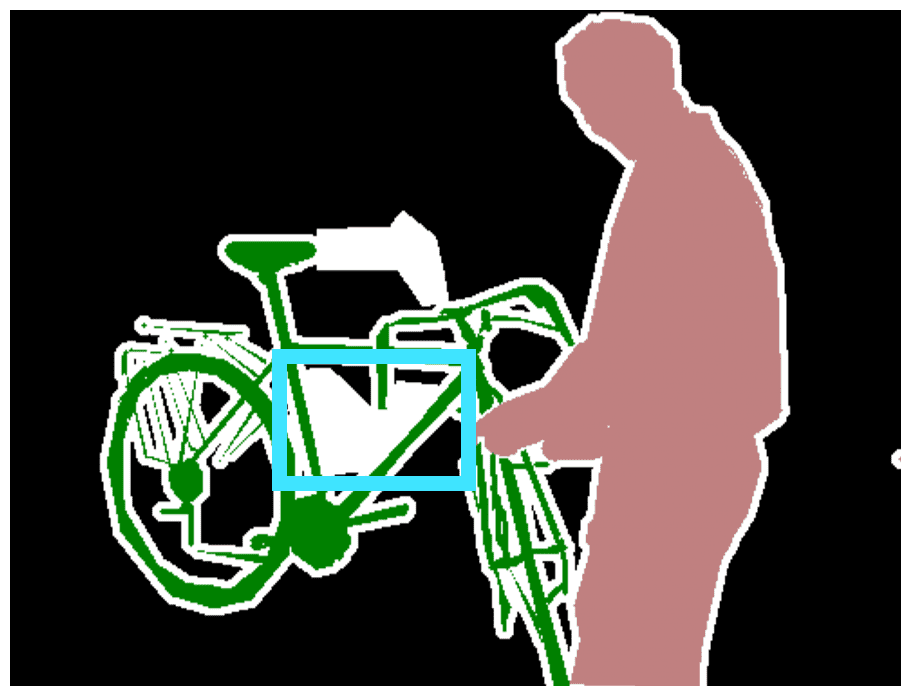}
        \caption{PASCAL}
    \end{subfigure}
    \begin{subfigure}[h!]{.33\linewidth}
        \centering
        \includegraphics[scale=0.175]{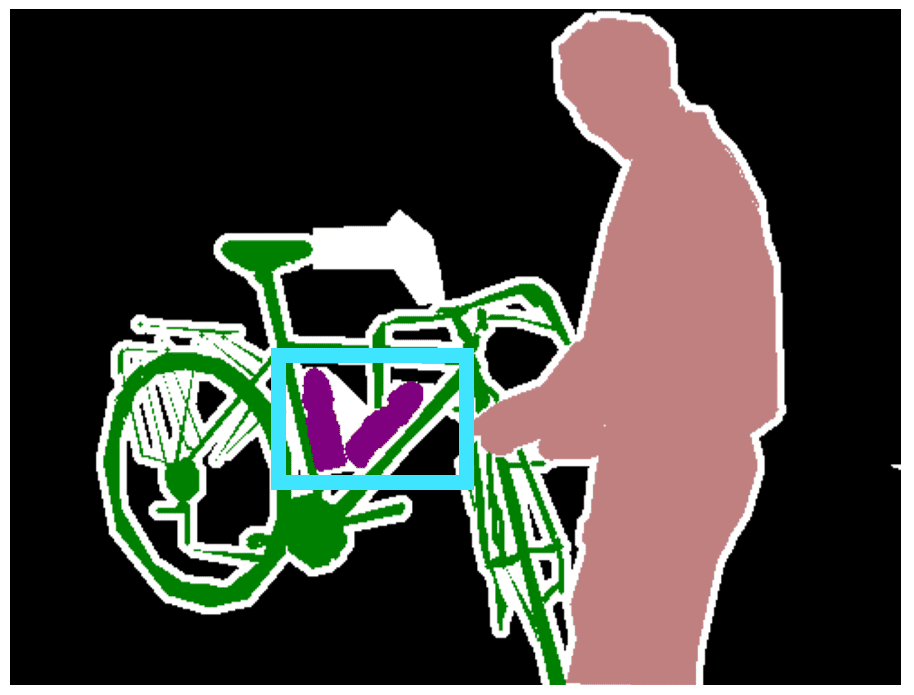}
        \caption{PASCAL+}
    \end{subfigure}

    \caption{{\em Correction that appears to cause negative IoU gains.} Here, the colors black, red, purple, green, and pink represent the background, chair, bottle, bicycle, and person classes, respectively.}
    \label{fig:negative-iou}
\end{figure*}

\subsection{Comparison with Other Acquisitions}
\begin{table}[h!]
\caption{{\em Experiments with other acquisitions.} With 5K budgets, we select pixels from different pixel pools and correct the initial labels to PASCAL.}
\label{tab:acquisition}
\centering
\begin{center}
\begin{small}
\setlength\tabcolsep{6pt}
\begin{tabular}{c|c}
\toprule
Methods & Model mIoU ($\%$) \\ \midrule
Entropy & $57.09 \pm 0.40$ \\
BvSB & $57.58 \pm 0.41$ \\ 
ClassBal & $57.51 \pm 0.67$ \\
SIM & $65.30 \pm 0.21$ \\ 
\bottomrule
\end{tabular}
\end{small}
\end{center}
\end{table}

In Table~\ref{tab:acquisition}, our SIM acquisition outperforms other various acquisitions including Entropy, Best-versus-Second-Best (BvSB), and Class-Balanced (ClassBal), employed in active learning, due to the incorporation of the look-ahead concept.
We concentrate on adjusting the acquisition function, while simultaneously applying other techniques such as diversified pixel pool and expansion techniques. We correct the labels of 3K pixels selected using various acquisition functions, and expand the labels to their corresponding superpixels.

\subsection{Class IoU on PASCSAL}
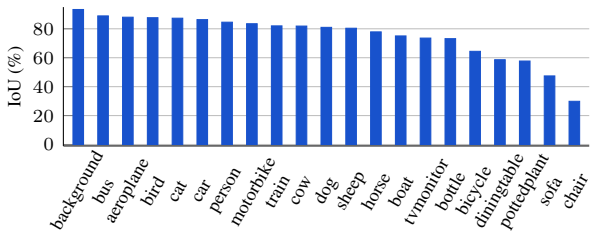
\begin{figure}[h!]
% average line
\captionsetup[subfigure]{font=footnotesize,labelfont=footnotesize,aboveskip=0.05cm,belowskip=-0.15cm}
\centering
\begin{tikzpicture}
    \begin{axis}[
        label style={font=\scriptsize},
        tick label style={font=\scriptsize},
        xticklabel style={anchor=center, yshift=-4.5mm, font=\scriptsize},
        symbolic x coords={
            \rotatebox{60}{background},
            \rotatebox{60}{bus},
            \rotatebox{60}{aeroplane},
            \rotatebox{60}{bird},
            \rotatebox{60}{cat},
            \rotatebox{60}{car},
            \rotatebox{60}{person},
            \rotatebox{60}{motorbike},
            \rotatebox{60}{train},
            \rotatebox{60}{cow},
            \rotatebox{60}{dog},
            \rotatebox{60}{sheep},
            \rotatebox{60}{horse},
            \rotatebox{60}{boat},
            \rotatebox{60}{tvmonitor},
            \rotatebox{60}{bottle},
            \rotatebox{60}{bicycle},
            \rotatebox{60}{diningtable},
            \rotatebox{60}{pottedplant},
            \rotatebox{60}{sofa},
            \rotatebox{60}{chair},
        },
        axis y line*=left,
        axis x line=bottom,
        width=0.5\textwidth,
        height=0.2\textwidth,
        major x tick style = transparent,
        ybar=3*\pgflinewidth,
        bar width=4pt,
        ymajorgrids=true,
        ylabel={IoU (\%)},
        xtick=data,
        scaled y ticks=false,
        enlarge x limits=0.03,
        axis line style={-},
        ymin=-1,
        ymax=95,
        xlabel style={yshift=0.4cm},
        ylabel style={yshift=-0.6cm},
    ]
        \addplot[style={cBlue,fill=cBlue,mark=none}] coordinates {
            (\rotatebox{60}{background}, 93.34)
            (\rotatebox{60}{bus}, 88.91)
            (\rotatebox{60}{aeroplane}, 88.01)
            (\rotatebox{60}{bird}, 87.66)
            (\rotatebox{60}{cat}, 87.32)
            (\rotatebox{60}{car}, 86.33)
            (\rotatebox{60}{person}, 84.56)
            (\rotatebox{60}{motorbike}, 83.50)
            (\rotatebox{60}{train}, 82.01)
            (\rotatebox{60}{cow}, 81.87)
            (\rotatebox{60}{dog}, 80.99)
            (\rotatebox{60}{sheep}, 80.38)
            (\rotatebox{60}{horse}, 77.86)
            (\rotatebox{60}{boat}, 75.06)
            (\rotatebox{60}{tvmonitor}, 73.57)
            (\rotatebox{60}{bottle}, 73.17)
            (\rotatebox{60}{bicycle}, 64.41)
            (\rotatebox{60}{diningtable}, 58.61)
            (\rotatebox{60}{pottedplant}, 57.70)
            (\rotatebox{60}{sofa}, 47.41)
            (\rotatebox{60}{chair}, 29.86)
        };
    \end{axis}
\end{tikzpicture}
\caption{{\em Class IoU on PASCAL.} The IoU values of diningtable, pottedplant, sofa, and, chair classes are relatively low when trained with PASCAL.}
\label{fig:pascal-iou}
% \vspace{-2mm}
\end{figure}

We provide the rationale of IoU gain in Figure~\ref{fig:stats-(a)}.
For a detail, thanks to the corrected PASCAL+, we observe that the IoU values of pottedplant, sofa, chair, and diningtable classes increase.
This is related to the number of corrections in Figure~\ref{fig:stats-(b)}, as those class are corrected lots than other classes.
However, in case of background and person classes, we cannot obtain IoU gain as those classes already attain high IoU with PASCAL as depicted in Figure~\ref{fig:pascal-iou}.

\iffalse
\subsection{Practical Noises in Semantic Segmentation Datasets}
In the literature of noisy labels in semantic segmentation, synthetic noise are, Drop and Flip are unrealistic.
\fi

\iffalse
\subsection{Superpixels vs. Foundation model}
\begin{table}[h!]
\caption{{\em Superpixel Comparison}}
\label{tab:superpixel-comparison}
\begin{center}
\begin{small}
\setlength\tabcolsep{8pt}
\centering
\begin{tabular}{cc|cc}
\toprule
Cityscapes & ASA & PASCAL & ASA \\ \midrule
$\text{SLIC}_{4096}$ & 0.887 & $\text{SLIC}$ &  \\ 
$\text{SEEDS}_{4096}$ & 0.909 & $\text{SEEDS}$ &  \\ 
$\text{G-SAM}_{0.1}$ & 0.932 & $\text{G-SAM}_{0.1}$ & 0.983 \\ 
$\text{G-SAM}_{0.2}$ &  & $\text{G-SAM}_{0.2}$ & 0.982 \\ 
\bottomrule
\end{tabular}
\end{small}
\end{center}
\end{table}
\fi

\section{Additional Results of PASCAL+}
\subsection{Qualitative Results}
\label{sec:extra_results}
Additional qualitative results of corrected labels using our proposed method are depicted in Figure~\ref{fig:extra_qual}.
These results demonstrate that our proposed correction method effectively identifies objects overlooked in the original labels.

\subsection{Uncorrectable Cases}
\label{sec:failure}
Figure~\ref{fig:error_qual} presents examples where corrections made by our proposed method are not entirely successful.
Specifically, the examples in the first and second rows of Figure~\ref{fig:error_qual} illustrate situations where annotators mistakenly assign pixel clicks to the wrong classes.
Such errors can occur under limited budgets. 
% using a rule-based approach.
In the last row of Figure~\ref{fig:error_qual}, an area mislabeled as person class is effectively corrected to car class.
However, due to the insufficient granularity of the superpixels, small areas remain uncorrected.
This limitation can be mitigated by employing more refined superpixels or utilizing improved foundational models.

\subsection{Negative IoU Gains of PASCAL+}
\iffalse
To compare PASCAL with PASCAL+, we first examine the class distribution of pixels in each dataset, as described in Figure 1 of Link 3. For a closer look, we represent the difference in the pixel ratio between PASCAL and PASCAL+, as depicted in Figure 2 of Link 3. Notably, we observe significant corrections in the classes diningtable, person, pottedplant, sheep, and sofa, which correlates with the IoU gains illustrated in Figure 6a.
\fi

Figure 6a represents negative IoU gains for certain classes such as person, bottle, and cow. Here, we provide the rationale for these negative gains. The final IoU gain is determined by the positive and negative impacts of corrections. Although corrections generally aim to reduce noisy labels, yielding positive effects, they can also have negative effects, especially on challenging objects, as shown in Figure~\ref{fig:negative-iou}.

\begin{figure*}[!t]
    \centering
    \begin{subfigure}[h!]{.45\linewidth}
        \centering
        \includegraphics[scale=0.4]{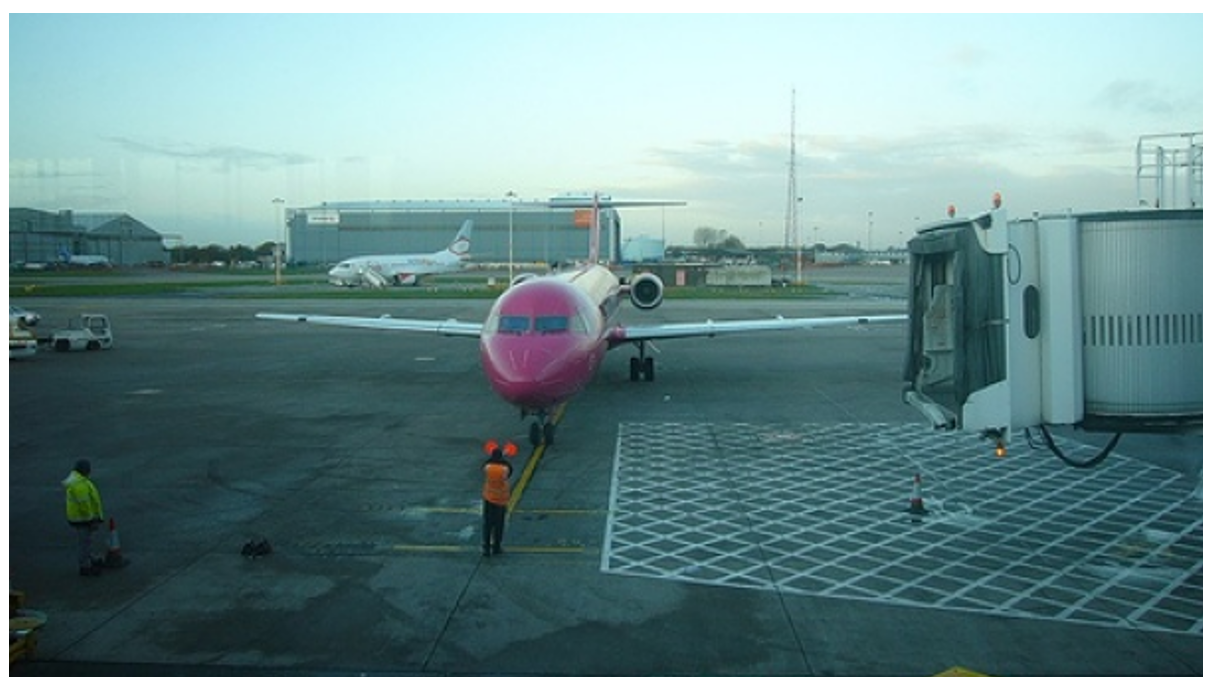}
        \caption{Unlabeled image}
    \end{subfigure}
    \begin{subfigure}[h!]{.45\linewidth}
        \centering
        \includegraphics[scale=0.4]{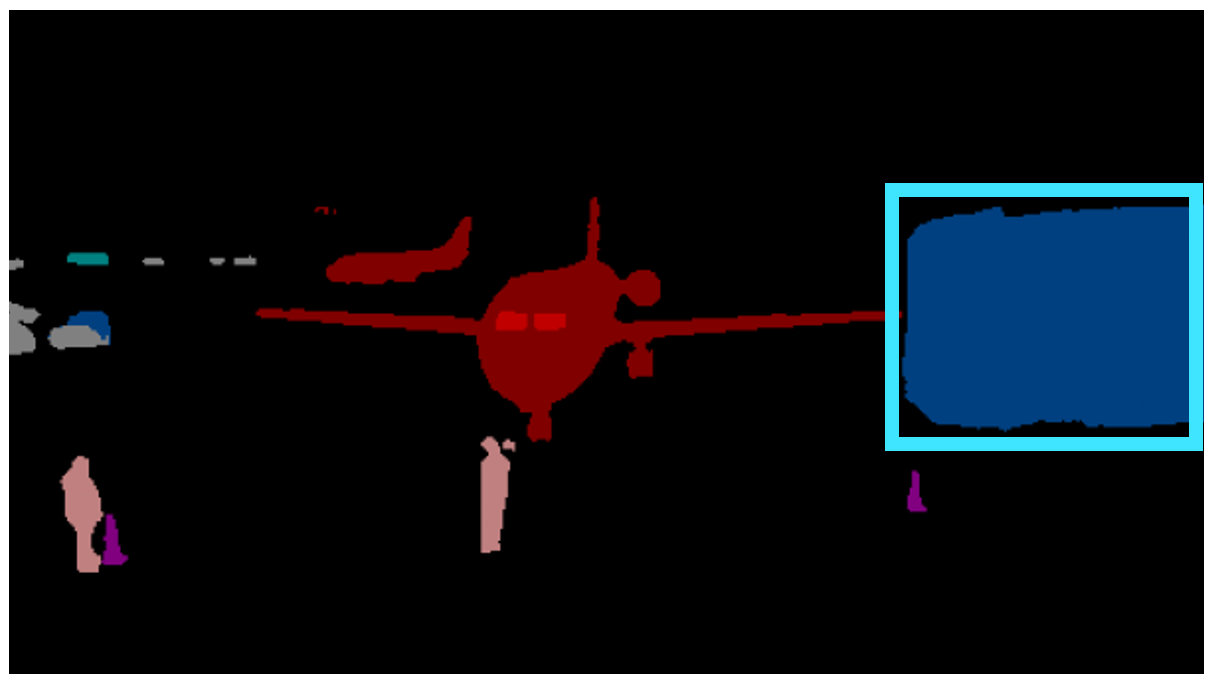}
        \caption{Round 0 (Grounded-SAM)}
    \end{subfigure}

    % \vspace{3mm}
    \begin{subfigure}[h!]{.45\linewidth}
        \centering
        \includegraphics[scale=0.4]{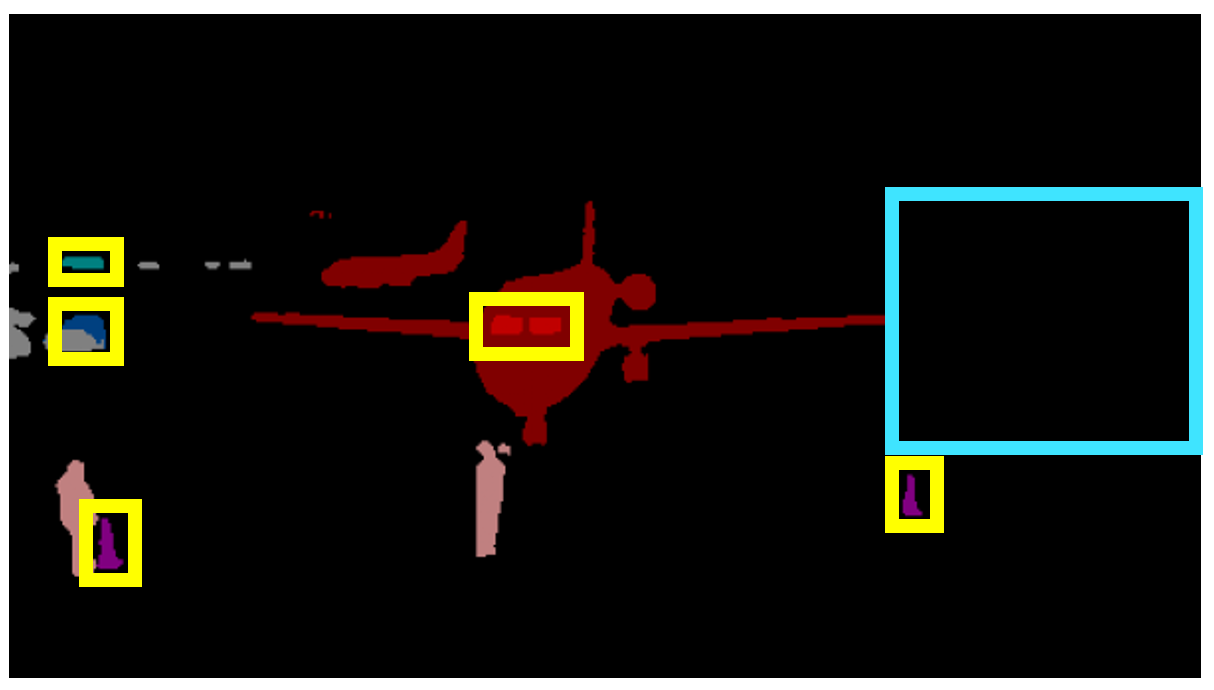}
        \caption{Round 1}
    \end{subfigure}
    \begin{subfigure}[h!]{.45\linewidth}
        \centering
        \includegraphics[scale=0.4]{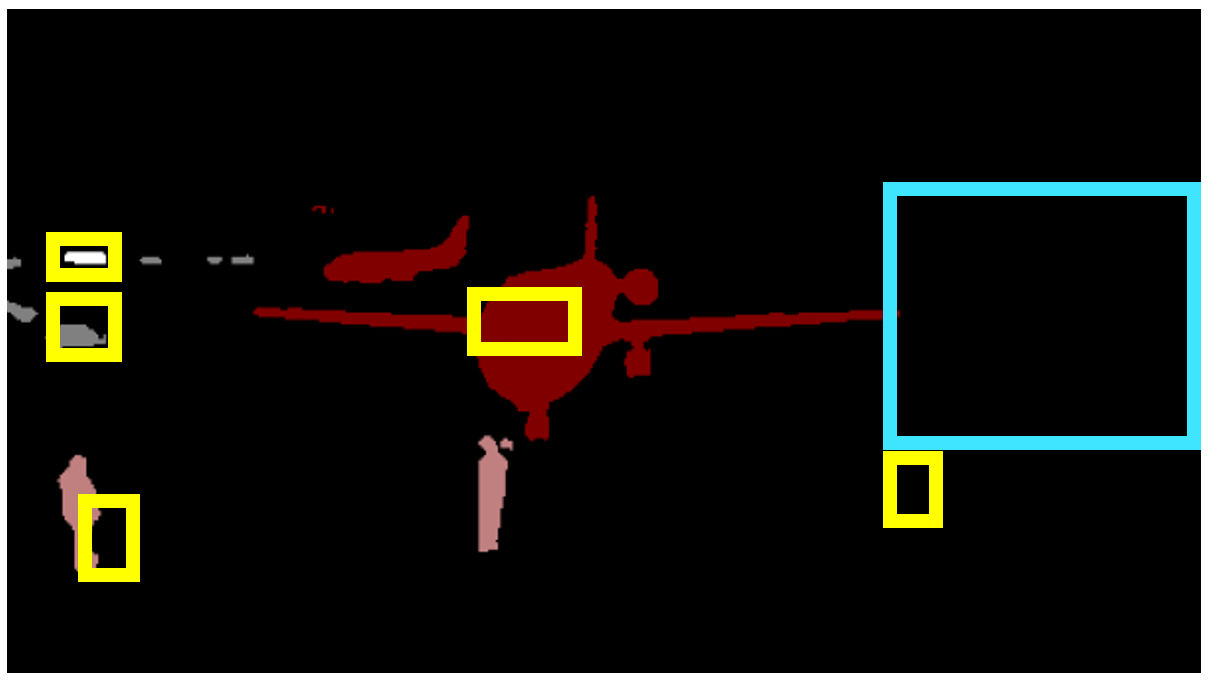}
        \caption{Round 2}
    \end{subfigure}
    \caption{{\em Segmentation changes through active label correction.} (b) The initial pseudo labels obtained from Grounded-SAM contain numerous noisy labels, exemplified by instances like tvmonior inside the cyan box. (c) In the first round, the object labeled as tvmonitor is corrected to background. Nonetheless, many noisy labels exist within the yellow boxes. (c) In the second round, we rectify all remaining noisy labels. With the help of the proposed look-ahead acquisition function, we prioritize correcting large objects before addressing small ones. Here, the colors black, blue, red, dark red, purple, and pink represent the background, tvmonitor, chair, airplane, bottle and person classes, respectively.}
    % \vspace{-3mm}
\end{figure*}